\documentclass{article}

\usepackage[preprint]{neurips_2026}

\usepackage[utf8]{inputenc} 
\usepackage[T1]{fontenc}    
\usepackage{hyperref}       
\usepackage{url}            
\usepackage{booktabs}       
\usepackage{amsfonts}       
\usepackage{nicefrac}       
\usepackage{microtype}      
\usepackage{xcolor}         

\usepackage{microtype}
\usepackage{graphicx}
\usepackage{caption}
\usepackage{booktabs}
\usepackage{algorithm}
\usepackage{algorithmic}
\usepackage{natbib}
\setcitestyle{numbers,square,comma,sort&compress}
\usepackage{hyperref}

\usepackage{amsmath}
\usepackage{amssymb}
\usepackage{mathtools}
\usepackage{amsthm}
\usepackage{mathrsfs}
\usepackage{subcaption}

\newcommand{\R}{\mathbb{R}}

\newcommand{\N}{\mathbb{N}}

\newcommand{\E}{\mathbb{E}}

\newcommand{\eps}{\varepsilon}
\newcommand{\Prob}{\mathbb{P}}

\usepackage[capitalize,noabbrev]{cleveref}

\theoremstyle{plain}
\newtheorem{theorem}{Theorem}[section]
\newtheorem{proposition}[theorem]{Proposition}
\newtheorem{lemma}[theorem]{Lemma}
\newtheorem{corollary}[theorem]{Corollary}
\theoremstyle{definition}
\newtheorem{definition}[theorem]{Definition}
\newtheorem{assumption}[theorem]{Assumption}
\theoremstyle{remark}

\title{Optimal Initialization in Depth:\\ Lyapunov Initialization and Limit Theorems for Deep Leaky ReLU Networks}

\author{%
  Constantin Kogler\textsuperscript{1}\thanks{Correspondence: \texttt{kogler@ias.edu}}
  \quad Tassilo Schwarz\textsuperscript{2,3}
  \quad Samuel Kittle\textsuperscript{4} \\ \\
  \textsuperscript{1} School of Mathematics, Institute for Advanced Study
  \\
  \textsuperscript{2} Mathematical Institute, University of Oxford \\
  \textsuperscript{3} Max Planck Institute for Multidisciplinary Sciences \\
  \textsuperscript{4} Deparment of Mathematics, University College London
}

\begin{document}

\maketitle

\begin{abstract}
  Effective initialization in deep networks requires an understanding of random neural networks. In this work, a rigorous probabilistic analysis of deep bias-free random Leaky ReLU networks is provided. We prove a Law of Large Numbers and a Central Limit Theorem for the logarithm of the norm of network activations, establishing that, as the number of layers increases, their growth is governed by a parameter called the Lyapunov exponent. This parameter characterizes a sharp phase transition between vanishing and exploding activations, and we calculate the Lyapunov exponent explicitly for Gaussian or orthogonal weight matrices. Our results reveal that standard methods, such as He initialization or orthogonal initialization, do not guarantee activation stability for deep networks of low width. Based on these theoretical insights, we propose a novel initialization method, referred to as Lyapunov initialization, which sets the Lyapunov exponent to zero and thereby ensures that the neural network is as stable as possible, leading empirically to improved learning.
\end{abstract}

\section{Introduction}

Training a neural network consists of iteratively optimizing its weights until a local minimum 
is reached. To ensure convergence, beyond choosing a suitable optimization algorithm, one must therefore also determine where the optimization should begin. In other words, prior to training, an initial set of weights must be selected, which is referred to as the \textit{initialization} of the neural network. 

Setting all weights to the same constant value at initialization is undesirable because, in a fully connected MLP, this choice results in identical gradients for all parameters, which prevents the network from learning.
Thus, random initialization of the weights was already suggested almost 30 years ago \cite{lecunEfficientBackProp1998} and is now common in practice.

When only the biases of the neural network are initialized as $\equiv 0$, individual neurons still behave differently, allowing effective learning. For this reason, it is standard practice to set all the biases $\equiv 0$ at initialization. The most common initialization methods  \cite{glorotUnderstandingDifficultyTraining2010,heDelvingDeepRectifiers2015} therefore only randomly initialize the weight matrices while keeping the biases $\equiv 0$.
In this paper, we establish novel theoretical results for random bias-free Leaky ReLU networks,  leading to better behaved initializations, especially in low dimensions and when the depth of the neural network is larger than its width.

To introduce notation, suppose we have a bias-free neural network of fixed width $d \in \N_{\geq 1}$, where a depth-$\ell$ layer is defined iteratively for $\ell \geq 1$ as 
\begin{equation}\label{Def:Xell}
	X_{\ell} = \phi(W_\ell X_{\ell-1})   \quad\quad \text{ and } \quad\quad X_{0} = x_0,
\end{equation}
where $x_0 \in \R^d$ is the input, $W_{\ell} \in M_d(\R)$ is a $d\times d$ weight matrix and $\phi: \R \to \R$ is an activation 
function applied component-wise.  

This paper establishes unique and remarkable properties of Leaky ReLU networks that do not extend to other activation functions, making a strong case for their use in specific applications. We proceed with presenting our theoretical results and then outline the resulting novel initialization methods.

\begin{figure}
\begin{minipage}[c]{0.34\textwidth}
    \vspace{0pt}
    \centering
    \includegraphics[width=\linewidth]{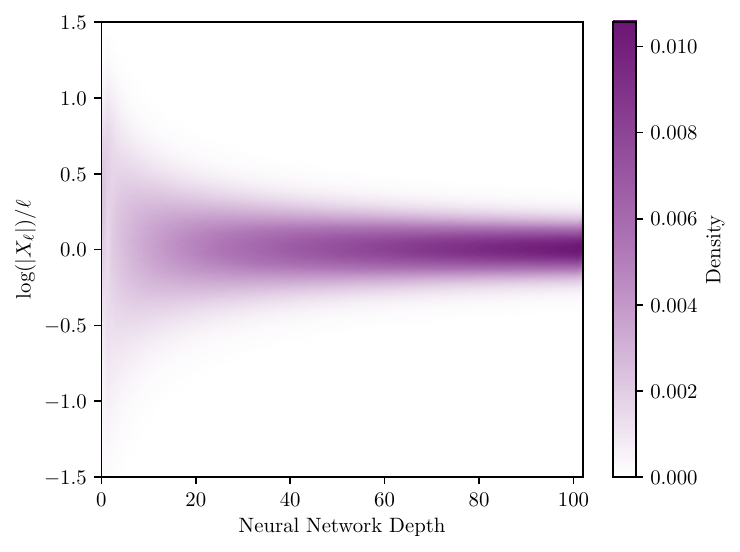}
\end{minipage}
\hfill
\begin{minipage}[c]{0.62\textwidth}
     \vspace{0pt}
    \centering
    
     \begin{subfigure}[t]{0.49\linewidth}
    \centering
    \includegraphics[width=\linewidth]{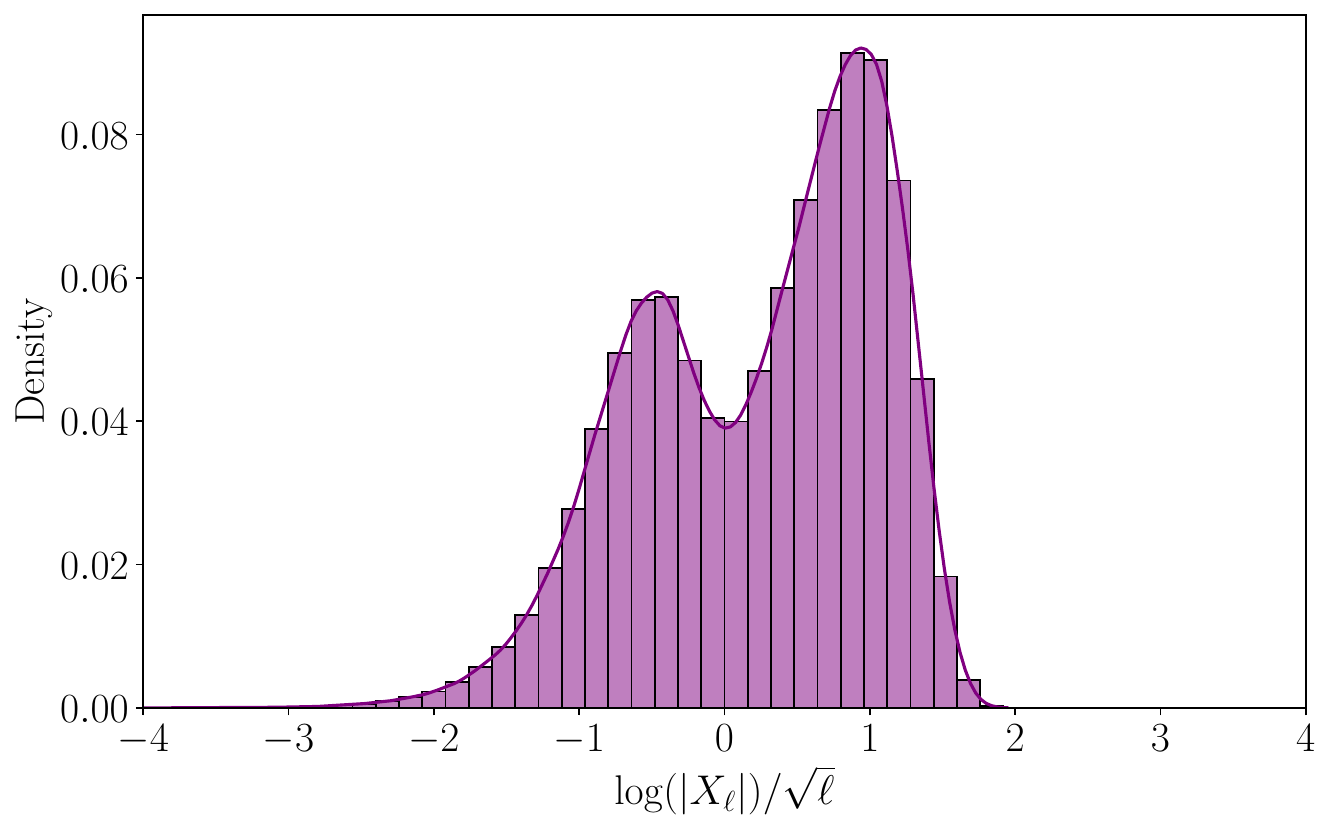}
    \caption{1 Layer}
  \end{subfigure}
  \begin{subfigure}[t]{0.49\linewidth}
    \centering
    \includegraphics[width=\linewidth]{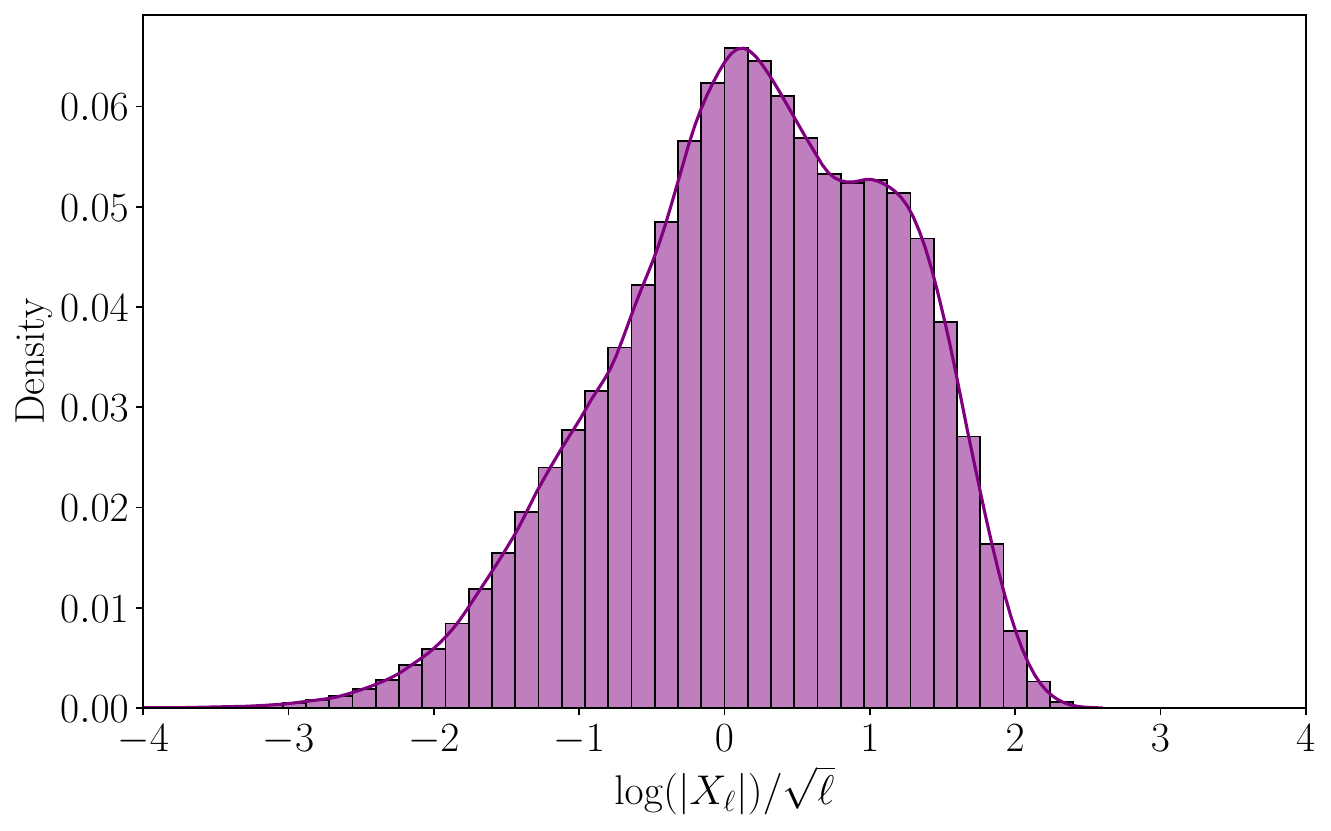}
    \caption{2 Layers}
  \end{subfigure}
  
  \begin{subfigure}[t]{0.49\linewidth}
    \centering
    \includegraphics[width=\linewidth]{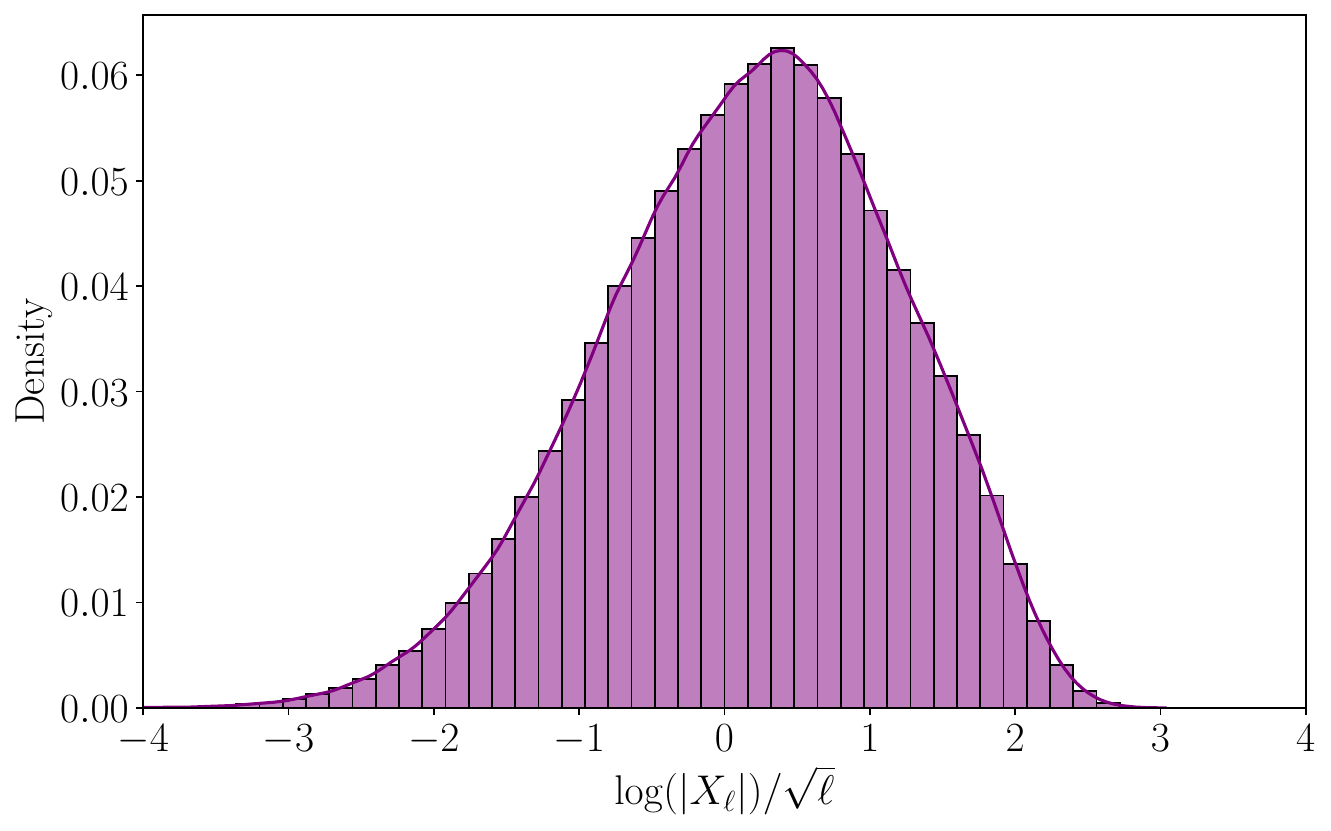}
    \caption{3 Layers}
\end{subfigure}
  \begin{subfigure}[t]{0.49\linewidth}
    \centering
    \includegraphics[width=\linewidth]{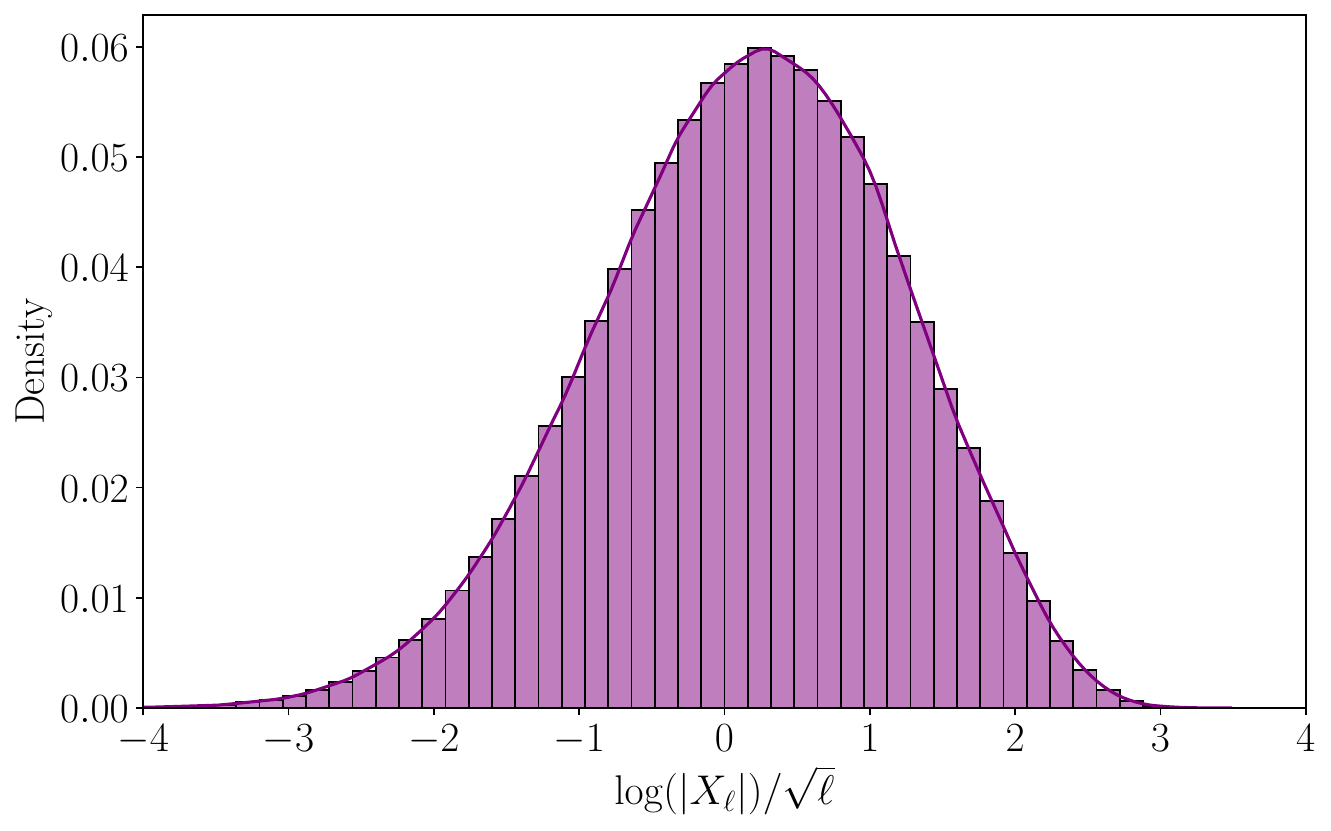}
    \caption{4 Layers}
  \end{subfigure}
\end{minipage}
\caption{Visualization of the Law of Large Numbers (Theorem~\ref{ExplicitLeakyReLULLN}) and Central Limit Theorem (Theorem~\ref{ExplicitLeakyReLUCLT}) for $\log |X_{\ell}|$ in dimension $d = 2$ with a Leaky ReLU activation function $\phi(x) = \max(x,\alpha x)$ and $\alpha = 0.1$. We initialized the matrix coefficients with independent Gaussian coefficients chosen such that the Lyapunov exponent is zero and therefore $\frac{1}{\ell}\log |X_{\ell}|$ converges to zero almost surely (left picture) and $\frac{1}{\sqrt{\ell}}\log |X_{\ell}|$ to a Gaussian with mean zero (4 right pictures), which is already visible after 4 layers.}
\label{fig:LLNFigure}
\end{figure}

\paragraph{Law of Large Numbers.} Throughout the rest of this paper, we consider a Leaky ReLU activation function defined for a fixed $\alpha \in \R_{\neq 0}$ as
\begin{equation}\label{LeakyReLUDef}
    \phi(x) = \max(x,\alpha x)
\end{equation}
for $x\in \R$. A main insight of this article is that while it is difficult to study the distribution of $|X_{\ell}|$ directly, we can instead study $\log |X_{\ell}|$, where $|\circ|$ is the Euclidean norm.

Indeed, we will show in Theorem~\ref{ExplicitLeakyReLULLN} the following analogue of the Law of Large Numbers: If $\mu$ is a distribution on the space of matrices satisfying suitable assumptions and the weight matrices $W_{\ell}$ are sampled independently from $\mu$, then there exists a number $\lambda_{\mu,\phi} \in \R$ such that almost surely 
\begin{equation}\label{EquationLLN}
    \frac{1}{\ell}\log |X_{\ell}|   \overset{\ell \to \infty}{\longrightarrow} \lambda_{\mu,\phi}.
\end{equation}
In the case where $\phi$ is the identity, $X_{\ell}$ is simply a product of independent random matrices applied to a vector. In the latter case, \eqref{EquationLLN} is well-known and was first established under certain assumptions by Furstenberg-Kesten \cite{FurstenbergKesten} in the 1960s. The novelty in our work lies in the extension to random Leaky ReLU neural networks, which is, to the authors' knowledge, a new non-linear Law of Large Numbers and is therefore also of mathematical significance. When $\phi$ is the identity, the limit in \eqref{EquationLLN} is called the Lyapunov exponent and therefore we call $\lambda_{\mu,\phi}$ the \textbf{Lyapunov exponent} of the neural network.

The Law of Large Numbers \eqref{EquationLLN} states that as a first approximation $|X_{\ell}| \approx  e^{\ell \cdot \lambda_{\mu,\phi}}.$ This implies that when $\lambda_{\mu,\phi} > 0$, then $|X_{\ell}|$ explodes as $\ell$ increases, while $|X_{\ell}|$ vanishes when $\lambda_{\mu,\phi} < 0$. Therefore, the Lyapunov exponent governs the exponential growth rate of the activations.

We will show that in low dimensions all the previous initialization methods have a negative Lyapunov exponent and therefore lead to vanishing activations and inefficient learning. For example, when $d = 2$ and $\alpha = 0.1$, He initialization (see Section~\ref{section:relatedwork}) has a Lyapunov exponent of $\lambda_{\mu,\phi} \approx -0.82$.  As the dimension increases, however, He initialization leads to a zero Lyapunov exponent. Thus, our work can also be viewed as a theoretical justification of He initialization in high dimensions.

\paragraph{Central Limit Theorem.} The Law of Large Numbers gives a rough estimate on the size of the values of $\log |X_{\ell}|$. We also establish a version of the Central Limit Theorem (Theorem~\ref{ExplicitLeakyReLUCLT}): If $\mu$ is again a matrix distribution satisfying suitable assumptions, the matrix weights are sampled independently from $\mu$ and $\lambda_{\mu,\phi}$ is the Lyapunov exponent, then there exists $\gamma_{\mu,\phi} \geq 0$ such that 
\begin{equation}\label{EquationCLT}
\frac{\log |X_{\ell}| - \ell \lambda_{\mu, \phi}}{\sqrt{\ell}} \overset{\ell \to \infty}{\longrightarrow} \mathcal{N}(0,\gamma_{\mu,\phi}^2),
\end{equation}
where the convergence holds in distribution.

The Central Limit Theorem \eqref{EquationCLT} provides a finer description of $\log |X_{\ell}|$ than the Law of Large Numbers, stating that $\log |X_{\ell}|$ roughly behaves like a Gaussian with mean $\ell \lambda_{\mu,\phi}$ and standard deviation $\sqrt{\ell} \gamma_{\mu,\phi}$.

\paragraph{Lyapunov Initialization.} Building on these theoretical insights, we propose a novel initialization strategy, referred to as \textit{Lyapunov initialization}. By choosing the weight distribution such that the Lyapunov exponent is exactly zero, we ensure that activations are as stable as possible for deep networks. We provide explicit formulas for calculating the optimal scaling factors for both Gaussian and uniformly sampled orthogonal weight matrices. 

According to the Central Limit Theorem, setting the Lyapunov exponent to zero prevents typical exponential explosion or vanishing, but it nonetheless establishes stochastic variation. In particular, the logarithm of the activation norm, $\log |X_{\ell}|$, behaves as a Gaussian variable with a standard deviation of $O(\sqrt{\ell})$, causing the activation magnitude $|X_{\ell}|$ to fluctuate between $e^{-O(\sqrt{\ell})}$ and $e^{O(\sqrt{\ell})}$. 

To counter these depth-dependent fluctuations, we propose a refinement of our method termed \textit{sampled Lyapunov initialization}: We suggest to generate $O(\sqrt{\ell})$ candidate initializations prior to the training phase. By doing so, we ensure with positive probability that at least one candidate yields an output norm $|X_{\ell}|$ close to one, allowing us to select the most stable configuration for training.

\paragraph{Experiments.} We provide experimental evidence that our new Lyapunov methods outperform previous standard initialization strategies. From our viewpoint, the reason why the previous methods are not optimal, is that their Lyapunov exponent is negative and far from zero. Therefore, not even small changes in their coefficients lead to a regime where the activations have roughly the right norm, preventing efficient learning.

\paragraph{Outline and Contributions.}

Our main contributions are:

\begin{itemize}
    \item \textbf{Limit Theorems for Deep Leaky ReLU Networks:} We prove a Law of Large Numbers and a Central Limit Theorem for the logarithm of the activation norms, establishing that as depth increases, the activation growth is governed by a parameter called the Lyapunov exponent.   (Section~\ref{section:LimitTheorems})
    
    \item \textbf{Analytical Formulas:} We derive explicit, closed-form formulas for the Lyapunov exponent for both Gaussian and orthogonal weight matrices. (Section~\ref{section:Initalization})
    
    \item \textbf{Lyapunov Initialization:} We propose a novel strategy that explicitly sets the Lyapunov exponent to zero to maximize stability. We further introduce Sampled Lyapunov Initialization to mitigate depth-dependent stochastic fluctuations predicted by our Central Limit Theorem. (Section~\ref{section:Initalization})

    \item \textbf{Theoretical Analysis of Standard Initialization:} We demonstrate that widely used methods, such as He and standard orthogonal initialization, yield negative Lyapunov exponents in low-width regimes, precisely quantifying the vanishing of activations in deep networks. On the other hand, in infinite width, He initialization and standard orthogonal initialization lead to a zero Lyapunov exponent. (Section~\ref{section:Initalization})
\end{itemize}

We discuss our experimental evidence in Section~\ref{section:Experiments} and conclude the paper with Section~\ref{section:Conclusion}.

\section{Related Work}\label{section:relatedwork}

\paragraph{Glorot and He initialization.}

In their seminal work, Glorot and Bengio \cite{glorotUnderstandingDifficultyTraining2010} proposed choosing all weights i.i.d.\ from distributions with mean zero and a variance of $\frac{1}{d}$, such as $\mathcal{N}(0,\frac{1}{d})$ or $\text{Unif}[-\sqrt{3/d},\sqrt{3/d}]$. Known as \textit{Glorot (or Xavier) initialization}, this method aims to keep the mean and variance of activations constant across layers. However, this approach relies on the simplifying assumption that the activation function is linear and therefore only works well in practice for activation functions that are linear near zero (i.e., $\phi(x) \approx x$), which holds for functions like $\tanh$ but fails for ReLU and Leaky ReLU.

To address this, \textit{He initialization} \cite{heDelvingDeepRectifiers2015} was developed for ReLU-based networks. For Leaky ReLU with a slope parameter $\alpha$, it samples weights with a variance of $2/(d(1+\alpha^{2}))$. While this choice successfully preserves the second moment of the activations, that is ensuring $\mathbb{E}[|X_{\ell}|^{2}]\!=\!\mathbb{E}[|X_{\ell'}|^{2}]$ for all $\ell,\! \ell' \!\geq\! 0$, it does not preserve the mean or variance of $|X_{\ell}|$. In this paper, we demonstrate that in low dimensions, He initialization results in a negative Lyapunov exponent that is far from zero, establishing quantitatively how quickly the activations vanish rather than stabilize. On the other hand, our results prove that He initialization has a Lyapunov exponent close to zero in high dimensions.

\paragraph{Orthogonal Initialization.}

\cite{SaxeOrthogonal} proposed initializing weights with random orthogonal matrices rather than independent Gaussian entries. While effective for deep linear networks, applying them to non-linear networks requires a scaling factor (see also \cite{Mishkinetal}). For Leaky ReLU activations, the standard scaling factor is $\sqrt{2/(1 + \alpha^2)}$ as it satisfies $\mathbb{E}[|X_{\ell}|^{2}]\!=\!\mathbb{E}[|X_{\ell'}|^{2}]$ for all $\ell,\! \ell' \!\geq\! 0$, yet it still leads to a negative Lyapunov exponent. In this paper, we establish the optimal scaling factor for orthogonal initializations of infinite-depth Leaky ReLU networks.

In a complementary direction, using mean-field theory for the infinite-width case, \cite{Penningtonetal} established the optimal scaling rate for certain networks with orthogonal initialization. \cite{Xiaoetal} further extended these mean-field limit based ideas to CNNs, allowing the training of a 10,000-layer CNN.

\paragraph{Lyapunov exponents and random matrix products.}

From a theoretical perspective, our work is inspired by classical results on random matrix products \cite{FurstenbergKesten,LePage}. Lyapunov exponents arising from such products play a central role in probability theory and dynamical systems. Our work differs from these classical results by analyzing compositions of random matrices and nonlinear activation functions. This difference fundamentally changes the dynamics, meaning that existing linear results do not directly apply and the proofs of our results require new ideas.

\paragraph{Lyapunov Exponents and Edge of Chaos in Deep Learning.}

While the application of Lyapunov exponents in our specific setting appears to be novel, the concept has been explored in other areas of deep learning. Notably, Recurrent Neural Networks (RNNs) have been viewed as a dynamical system and their Lyapunov exponents were empirically studied in \cite{Vogtetal} and \cite{Engelkenetal}. 

The \textit{edge of chaos} refers to a critical phase transition of neural networks, where the parameters are calibrated to a regime that maximizes stability and learning performance. The first papers exploring this phenomenon were \cite{Bertschineretal} and \cite{Boedeckeretal}. More recently, edge of chaos has been analyzed in the context of infinite-width networks using mean-field theory \cite{Pooleetal,Schoenholetal}. In contrast, the Lyapunov exponent in this paper quantifies the edge of chaos for the regime of infinite-depth neural networks.

We also mention \cite{Hanin2018}, which studies exploding and vanishing gradients in expectation and \cite{HaninNica2018} addressing neural tangent kernels. Moreover, \cite{Hayou19} analyzes the impact of the activation function on stability and \cite{Hayou21} discusses ResNets and also neural tangent kernels.
In \cite{Hayou2022}, the infinite-depth limit distribution of certain special neural networks is studied.

\section{Limit Theorems for Leaky ReLU networks}\label{section:LimitTheorems}

\paragraph{Probabilistic Setup and Notation.} For our results to be valid, we require that the weight matrices $W_{\ell}$ are independent and identically distributed. With this assumption, the stochastic process $(X_{\ell})_{\ell \geq 1}$ is a Markov chain 
in the sense that the transition probabilities do not depend on the layer. Let $\mu$ be a probability measure on $M_d(\R)$, the space of real $d\times d$ matrices. We work with the probability space
\begin{equation}\label{ProbSpace}
    (\Omega, \mathscr{F}, \mathbb{P}_{\mu}) = (M_d(\R)^{\mathbb{N}}, \mathscr{B}(M_d(\R))^{\mathbb{N}}, \mu^{\mathbb{N}}),
\end{equation}
that is, the space of $M_d(\R)$-sequences endowed with the natural product $\sigma$-algebra and the product measure coming from $\mu$. So every $\omega \in \Omega$ is a sequence of matrices $\omega = (W_1, W_2, W_3, \ldots)$ and for a given $\omega \in \Omega$ we denote by $W_n(\omega)$ the projection of $\omega$ to the n-th coordinate. Therefore, \eqref{Def:Xell} formally means that for $\omega \in \Omega$ we have for $\ell \geq 1$ and a fixed $x_0 \in \R^d$,
\begin{equation}\label{ProbNNDef}
    X_{\ell}(\omega) = \phi(W_\ell (\omega) X_{\ell -1}(\omega))   \quad \text{ and } \quad X_{0} = x_0.
\end{equation} 
Although our methods require $W_{\ell}$ to be independent and identically distributed, the choice of the underlying distribution is flexible. In particular, our theorems hold provided $\mu$ satisfies either of the following assumptions.

\begin{assumption}\label{MeasureAssumptions}
Let $\mu$ be a probability measure on the space of $d \times d$ real matrices $M_d(\mathbb{R})$ with independent entries. For each $1 \leq i, j \leq d$, assume the $(i, j)$-coordinate distribution has a density $p_{ij}$ satisfying:  

\begin{enumerate}
    \item[(1)] \textbf{Boundedness:} Each density is bounded, i.e., $\sup_{u \in \mathbb{R}} |p_{ij}(u)| < \infty.$
    \item[(2)] \textbf{Finite Second Moment:} Each density has a finite second moment:
    $\int_{-\infty}^{\infty} u^2 p_{ij}(u) \, du < \infty. $
    \item[(3)] \textbf{Local Positivity:} Each density is uniformly positive in a neighborhood of $0$. That is, there exists $\epsilon > 0$ such that for each $1\leq i,j \leq d$ it holds that
    $ \inf_{|u| \leq \epsilon} p_{ij}(u) > 0$. 
\end{enumerate}
\end{assumption}
 
For example, when $p_{ij}$ is $\mathcal{N}(m_{ij},\sigma_{ij})$ with $m_{ij} \in \R$ and $\sigma_{ij} > 0$ or $\mathrm{Unif}([-a_{ij}, b_{ij}])$ for $a_{ij}, b_{ij} > 0$, then Assumption 
\ref{MeasureAssumptions} is satisfied. However, when $p_{ij} \sim \mathrm{Unif}[0,a]$ for some $a > 0$, then Assumption 
\ref{MeasureAssumptions} (3) does not hold, and we will actually show that our main theorems fail in this case. 

Second, we want to deal with scaled orthogonal matrices, that is, for a parameter $\eta > 0$, we denote
\begin{equation}\label{DefetaOrthogonalGroup}
    \eta\cdot\mathrm{O}(d) = \{ \eta \cdot Q \,:\,  Q \in \mathrm{O}(d)  \} \quad\text{ for } \quad \mathrm{O}(d) = \{ Q \in M_d(\R) \,:\,  Q^TQ = QQ^T = \mathrm{Id}_d  \}.
\end{equation} 
Orthogonal groups have a natural volume probability measure denoted $m_{\mathrm{O}(d)}$ and called the Haar probability measure as defined and discussed in Section~\ref{section:Haar}. We denote for $\eta > 0$ by 
\begin{equation}\label{HaaretaOdef}
    m_{\eta\cdot \mathrm{O}(d)} = \eta_{*}m_{\mathrm{O}(d)}
\end{equation}
the pushforward of $m_{\mathrm{O}(d)}$ under the map $O \mapsto \eta O$ for $O \in \mathrm{O}(d)$. We also remark that an absolutely continuous measure with respect to $m_{\eta \cdot \mathrm{O}(d)}$ (see Definition~\ref{TopGpACDefinition}) is by Theorem~\ref{RadonNikodym} one with a density with respect to the measure $m_{\eta\cdot \mathrm{O}(d)}$.

\begin{assumption}\label{SecondMeasureAssumptions}
Let $\eta > 0$ and let $\mu$ be a probability measure that is absolutely continuous with respect to $m_{\eta \cdot \mathrm{O}(d)}$ with a uniformly positive density $p: \eta \cdot \mathrm{O}(d) \to \R$, i.e. $\inf_{\substack{ Q \in \eta\cdot \mathrm{O}(d)}}  p(Q) > 0.$
\end{assumption}

\paragraph{Law of Large Numbers and Central Limit Theorem.} We are now ready to state our version of the law of large numbers.

\begin{theorem}(Law of large numbers for $|X_{\ell}|$)\label{ExplicitLeakyReLULLN}
     Let $\mu$ be a probability measure on $M_d(\R)$ satisfying either Assumption~\ref{MeasureAssumptions} or Assumption~\ref{SecondMeasureAssumptions}. Let $(\Omega, \mathscr{F}, \mathbb{P}_{\mu})$ be as in \eqref{ProbSpace}, $(X_{\ell})_{\ell \geq 0}$ be as in \eqref{ProbNNDef} and let $\phi(x) = \max(x,\alpha x)$ be a Leaky ReLU activation function with $\alpha \in \mathbb{R}_{\neq 0}$. 
     
     Then there exists a real number $\lambda_{\mu,\phi} \in \R$ depending on $\mu$ and $\phi$ such that for each fixed $x_0 \in \R^d \backslash \{ 0\}$ it holds that
     \begin{equation}\label{EqNormLLN}
         \lim_{\ell \to \infty}\frac{1}{\ell}\log |X_{\ell}(\omega)|  =\lambda_{\mu,\phi}
     \end{equation} for $\mathbb{P}_{\mu}$-almost all $\omega \in \Omega$.  The convergence is moreover in $L^1$, uniformly for $x_0$ of constant modulus, i.e.,  $$\lim_{\ell \to \infty}\sup_{x_0 \in \R^d \backslash \{ 0\}} \mathbb{E}_{\omega \sim \mathbb{P}_{\mu}} \left[ \bigg|\frac{1}{\ell}\log\frac{|X_{\ell}(\omega)|}{|x_0|} - \lambda_{\mu,\phi} \bigg| \right] = 0.$$
\end{theorem}

We now formally define the \textbf{Lyapunov exponent} of the neural network \eqref{Def:Xell} as the number $\lambda_{\mu,\phi}$ from \eqref{EqNormLLN}. Next, we state a version of the central limit theorem in this setting.

\begin{theorem}
    (Central Limit Theorem for $|X_{\ell}|$)\label{ExplicitLeakyReLUCLT} 
     Let $\mu$ be a probability measure on $M_d(\R)$ satisfying either Assumption~\ref{MeasureAssumptions} or Assumption~\ref{SecondMeasureAssumptions}. Let $(\Omega, \mathscr{F}, \mathbb{P}_{\mu})$ be as in \eqref{ProbSpace}, $(X_{\ell})_{\ell\geq 0}$ be as in \eqref{ProbNNDef} and let $\phi(x) = \max(x,\alpha x)$ be a Leaky ReLU activation function with $\alpha \in \mathbb{R}_{\neq 0}$.

     Then there exists $\gamma_{\mu,\phi} \geq 0$ such that for any $x_0 \in \R^d\backslash \{ 0 \}$ it holds that as $\ell \to \infty$, $$\frac{\log |X_{\ell}| - \ell \lambda_{\mu, \phi}}{\sqrt{\ell}} \longrightarrow \mathcal{N}(0,\gamma_{\mu,\phi}^2),$$ where the convergence holds in distribution and $\lambda_{\mu, \phi} \in \R$ is from Theorem~\ref{ExplicitLeakyReLULLN}.
\end{theorem}

Theorem~\ref{ExplicitLeakyReLULLN} and Theorem~\ref{ExplicitLeakyReLUCLT} are visualized in Figure~\ref{fig:LLNFigure}. We remark that with the methods employed, as discussed in Appendix~\ref{Section:ProofofLLN}, we can establish a Law of Large Numbers and a Central Limit Theorem for more general measures $\mu$ and for any activation function of the form $\phi(x) = \max(\alpha_1x, \alpha_2 x)$ with $\alpha_1, \alpha_2 \in \R_{\neq 0}$.  However, the results do not carry over to other standard activation functions. The proof idea is explained in Appendix~\ref{section:outlineofproofs}, also exposing what part of the argument fails for other activation functions. 

The Law of Large Numbers and especially the Central Limit Theorem as stated above are subtle results. In particular, we show in Lemma~\ref{Lemma:LLNCounterexample} that the Law of Large Numbers (Theorem~\ref{ExplicitLeakyReLULLN}) fails for ReLU activations and also for Leaky ReLU activations, for example, when the coefficients of $\mu$ are distributed independently as $\mathrm{Unif}([0,a])$ for $a > 0$. In addition, we empirically show in Figure~\ref{fig:CLT_failure_for_tanh} that the Central Limit Theorem fails for $\tanh$ networks as the distribution of $\frac{\log |X_{\ell}|}{\sqrt{\ell}}$ is not approximately Gaussian.

\section{Lyapunov Initialization}\label{section:Initalization}

\paragraph{Formulas for Lyapunov Exponent.}

In dynamical systems, Lyapunov exponents are usually difficult to calculate. It is remarkable that for the most common initializations and for Leaky ReLU activation functions, we can actually give effective formulas for the Lyapunov exponent. We first state our formulas in the Gaussian case.

\begin{theorem}\label{LyapunovGaussianFormula}
    Let $\mu_{\sigma}$ be the probability measure on $M_d(\R)$  with all coefficients independent and distributed as $\mathcal{N}(0,\sigma^2)$ for some $\sigma > 0$. Let $\phi = \max(x,\alpha x)$ be a Leaky ReLU activation function with $\alpha \in \R_{\neq 0}$. Then the Lyapunov exponent $\lambda_{\mu_{\sigma},\phi}$ from Theorem~\ref{ExplicitLeakyReLULLN} satisfies 
    \begin{equation}\label{GaussianFormula}
        \lambda_{\mu_{\sigma},\phi} = \log(\sigma) + I(d,\alpha), \quad\quad \text{ for }  \quad\quad I(d,\alpha) = \int_0^\infty \frac{e^{-t} - \frac{1}{2^d} \left(\frac{1}{\sqrt{1+2 t}} + \frac{1}{\sqrt{1+2\alpha^2 t}}  \right)^d }{2t}  dt.
    \end{equation}
\end{theorem}

Next, we address the scaled orthogonal group $\eta \cdot \mathrm{O}(d)$ as defined in \eqref{DefetaOrthogonalGroup}.

\begin{theorem}\label{LyapunovOrthogonalFormula}
    For $\eta > 0$ let $m_{\eta\cdot \mathrm{O}(d)}$ be the volume measure on $\eta\cdot \mathrm{O}(d)$ as defined in \eqref{HaaretaOdef}. Let $\phi = \max(x,\alpha x)$ be a Leaky ReLU activation function with $\alpha \in \R_{\neq 0}$. Then the Lyapunov exponent $\lambda_{m_{\eta\cdot\mathrm{O}(d)},\phi}$ from Theorem~\ref{ExplicitLeakyReLULLN} satisfies for $I(d,\alpha)$ from \eqref{GaussianFormula} that
    \begin{equation}\label{OrthogonalFormula}
        \lambda_{m_{\eta\cdot\mathrm{O}(d)},\phi} = \log(\eta) + I(d,\alpha)  - I(d,1). 
    \end{equation}
\end{theorem}

Theorem~\ref{LyapunovGaussianFormula} and Theorem~\ref{LyapunovOrthogonalFormula} allow us to calculate Lyapunov exponents efficiently. Indeed, the integral $I(d,\alpha)$ from \eqref{GaussianFormula} can be easily approximated using a quadrature method. 

\begin{table}[b]
\caption{Lyapunov exponents for He initialization and standard scaled orthogonal initialization from \eqref{DefLambdaHe} as well as critical standard deviation and critical scaling factor \eqref{sigmacrit}. All values depend on $d$ and are given for $\alpha = 0.1$ and $\alpha = 0.01$. See Section~\ref{section:lookuptables} for more extensive lookup tables.}
\label{tab:lyapunov-table}
\centering
\begin{tabular}{ccccc}
\multicolumn{5}{c}{$\alpha = 0.1$} \\
\midrule
 $d$ & $\lambda_{\mathrm{He}}$ & $\lambda_{\mathrm{ortho}}$ & $\sigma_{\mathrm{crit}}$ & $\eta_{\mathrm{crit}}$ \\
\midrule
1          &  -1.445 & -0.810 &  5.968 & 3.162 \\
2          &  -0.822 & -0.533 &  2.263 & 2.398 \\
3          &  -0.561 & -0.376  & 1.423 & 2.050 \\
4          &  -0.415 & -0.280 &  1.066 & 1.862\\
8          &  -0.188 & -0.123 &  0.600 & 1.591  \\
16         &  -0.085 & -0.053 &  0.383 & 1.483 \\
32         &  -0.040 & -0.024 &  0.259 & 1.442 \\
64         &  -0.020 & -0.012 &  0.179 & 1.424 \\
128        &  -0.010 & -0.006 &  0.126 & 1.415 \\
1024       &  -0.001 & -0.001 &  0.044 & 1.408\\
\bottomrule
\end{tabular}
\begin{tabular}{ccccc}
\multicolumn{5}{c}{$\alpha = 0.01$} \\
\midrule
$d$ & $\lambda_{\mathrm{He}}$ & $\lambda_{\mathrm{ortho}}$ & $\sigma_{\mathrm{crit}}$ & $\eta_{\mathrm{crit}}$ \\
\midrule
1          &  -2.591 & -1.956  & 18.874 & 10.000 \\
2          &  -1.435 & -1.146 &  4.199 & 4.450 \\
3          &  -0.898 & -0.713 &  2.004 & 2.886 \\
4          &  -0.605 & -0.470 &  1.295 & 2.263 \\
8          &  -0.214 & -0.149 &  0.619 & 1.642 \\
16         &  -0.088 & -0.056 &  0.386 & 1.496 \\
32         &  -0.041 & -0.025 &  0.261 & 1.451 \\
64         &  -0.020 & -0.012 &  0.180 & 1.431 \\
128        &  -0.010 & -0.006 &  0.126 & 1.423 \\
1024       &  -0.001 & -0.001 &  0.044 & 1.415 \\
\bottomrule
\end{tabular}
\end{table}

Recall that for a Leaky ReLU function of slope $\alpha$, He initialization samples the matrix entries independently from $\mathcal{N}(0,\sigma_{\mathrm{He}}^2)$ with 
\begin{equation}\label{DefSigmaHe}
    \sigma_{\mathrm{He}} = \sqrt{\frac{2}{d(1 + \alpha^2)}}.
\end{equation}
We then denote by $\lambda_{\mathrm{He}}$ the Lyapunov exponent resulting from $\mu_{\sigma_{\mathrm{He}}}$ and $\phi$ and by $\lambda_{\mathrm{orth}}$ the Lyapunov exponent of $m_{\eta_{0}\cdot\mathrm{O}(d)}$ and $\phi$, for $\eta_{0} =\sqrt{\frac{2}{1 + \alpha^2}}$ the standard Leaky ReLU orthogonal scaling factor analogous to He initialization, that is, \begin{equation}\label{DefLambdaHe}
    \lambda_{\mathrm{He}} = \lambda_{\mu_{\sigma_{\mathrm{He}}},\phi} = \log(\sigma_{\mathrm{He}}) + I(d,\alpha) \quad\text{ and }\quad \lambda_{\mathrm{orth}} = \lambda_{\eta_{0}\cdot\mathrm{O}(d),\phi} = \log(\eta_{0})  + I(d,\alpha) - I(d,1).
\end{equation}

We show in Table~\ref{tab:lyapunov-table} and Section~\ref{section:lookuptables} some numerical values for $\lambda_{\mathrm{He}}$ and $\lambda_{\mathrm{orth}}$. We next state the following asymptotic expansion result for $d\to \infty$ for $I(d,\alpha)$.

\begin{theorem}\label{AsymptoticExpansion}(Asymptotic Expansion)
    Let $\alpha \in \R_{\neq 0}$ and let $I(d,\alpha)$ be from \eqref{GaussianFormula}. Then for $d\to \infty$ it holds that 
    \begin{equation}\label{AsymtoticExpansionIdalpha}
        I(d,\alpha) = \frac{1}{2}\log\left(d\frac{1+\alpha^2}{2}\right) - \frac{C_\alpha}{4d} + O_{\alpha}(d^{-2}),
    \end{equation}
where $C_\alpha = \frac{5 - 2\alpha^2 + 5\alpha^4}{(1+\alpha^2)^2}$ and the implied constant depends on $\alpha$.

In particular, for $\lambda_{\mu_{\sigma},\phi}$ from \eqref{GaussianFormula} it holds as $d\to \infty$ that
\begin{equation}\label{AsymtoticExpansionGaussian}
    \lambda_{\mu_{\sigma},\phi} = \frac{1}{2}\log\left(\sigma^2d\frac{1+\alpha^2}{2}\right) - \frac{C_\alpha}{4d} + O_{\alpha}(d^{-2})
\end{equation} and for $\lambda_{m_{\eta\cdot\mathrm{O}(d)},\phi}$ from \eqref{OrthogonalFormula} it holds as $d\to \infty$ that
\begin{equation}\label{AsymtoticExpansionOrthogonal}
    \lambda_{m_{\eta\cdot\mathrm{O}(d)},\phi} = \frac{1}{2}\log\left(\eta^2\frac{1+\alpha^2}{2}\right)+ \frac{2-C_{\alpha}}{4d} + O_{\alpha}(d^{-2}).
\end{equation}
\end{theorem}

We remark that it follows from \eqref{AsymtoticExpansionGaussian} that $\lambda_{\mu_{\sigma_{\mathrm{He}}},\phi} = O_{\alpha}(d^{-1})$ and thus as $d\to \infty$, He initialization leads to a zero Lyapunov exponent. Moreover, by \eqref{AsymtoticExpansionOrthogonal}, the standard orthogonal scaling factor $\eta_{\mathrm{0}} = \sqrt{\frac{2}{1 
+ \alpha^2}}$ also leads to a zero Lyapunov exponent in high dimensions. To prove the latter two insights, we believe that all the theory developed in this paper is necessary. However, in low dimensions, He initialization and standard orthogonal initialization do not give a zero Lyapunov exponent, as Table~\ref{tab:lyapunov-table} shows. 

\paragraph{Lyapunov Initialization.} Based on the Law of Large Numbers established in Theorem~\ref{ExplicitLeakyReLULLN}, to prevent vanishing ($\lambda_{\mu,\phi} < 0$) or exploding ($\lambda_{\mu,\phi} > 0$) activations, we suggest to choose our random weights such that the Lyapunov exponent is exactly zero, that is
$\lambda_{\mu,\phi} = 0.$

For the Gaussian case, where $\mu_{\sigma}$ is the measure with independent entries distributed as $\mathcal{N}(0, \sigma^{2})$, we denote by $\sigma_{\mathrm{crit}} > 0$ the specific standard deviation that satisfies $\lambda_{\mu_{\sigma_{\mathrm{crit}}},\phi} = 0$. Analogously, we denote by $\eta_{\mathrm{crit}}$ the scaling factor such that $\lambda_{m_{\eta_{\mathrm{crit}} \cdot \mathrm{O}(d)}} = 0$. Using the formula derived in Theorem~\ref{LyapunovGaussianFormula} and Theorem~\ref{LyapunovOrthogonalFormula}, these critical values can be calculated as
\begin{equation}\label{sigmacrit}
\sigma_{\mathrm{crit}} = \exp(-I(d, \alpha)) \quad\quad\text{ and } \quad\quad \eta_{\mathrm{crit}} = \exp(I(d, 1) - I(d, \alpha)).
\end{equation}

Numerical values for $\sigma_{\mathrm{crit}}$ and $\eta_{\mathrm{crit}}$ across various widths and slopes are provided in Table~\ref{tab:lyapunov-table} and in Section~\ref{section:lookuptables}.

\begin{figure}[t]

\centering
\caption{Algorithm for Lyapunov Gaussian (Orthogonal) and Sampled Lyapunov Gaussian (Orthogonal) initialization.}

\begin{minipage}{0.46\textwidth}

\begin{algorithm}[H]

   \caption{Lyapunov Gaussian (Orthogonal) Initialization}

   \label{alg:lyapunov-init}

\begin{algorithmic}

   \STATE {\bfseries Setting:} Network width $d$, Leaky ReLU slope $\alpha$.

   \STATE {\bfseries Compute numerically:} $\sigma_{\mathrm{crit}}$ or $\eta_{\mathrm{crit}}$.

   \STATE {\bfseries Initialize:} Sample weights $W_{\ell} \sim \mu_{\sigma_{\mathrm{crit}}}$ (or

   $W_\ell \sim m_{\eta_{\mathrm{crit}}\cdot \mathrm{O}(d)}$), and set biases $b_{\ell}=0$.

\end{algorithmic}

\end{algorithm}

\end{minipage}
\hspace{.02\linewidth}
\begin{minipage}{0.46\textwidth}

\begin{algorithm}[H]

   \caption{Sampled Lyapunov Gaussian (Orthogonal) Initialization}

   \label{alg:lyapunov-init-sampled}

\begin{algorithmic}

   \STATE {\bfseries Setting:} Network width $d$, Leaky ReLU slope $\alpha$, input distribution $\mathcal{D}_{\mathrm{input}}$.

   \STATE {\bfseries Compute numerically:} $\sigma_{\mathrm{crit}}$ or $\eta_{\mathrm{crit}}$.

   \STATE {\bfseries Initialize:} Sample $O(\sqrt{\ell})$ weights $W_{\ell} \sim \mu_{\sigma_{\mathrm{crit}}}$ (or $W_\ell \sim m_{\eta_{\mathrm{crit}}\cdot \mathrm{O}(d)}$), set $b_{\ell}=0$, and choose the one with $\mathbb{E}_{x_0\sim\mathcal{D}_{\mathrm{input}}}[|X_\ell|]$ closest to $1$.

\end{algorithmic}

\end{algorithm}

\end{minipage}
\end{figure}


\paragraph{Sampled Lyapunov Initialization.}

Sampled Lyapunov Initialization is a refined strategy to counter depth-dependent stochastic fluctuations. Indeed, the Central Limit Theorem predicts that the logarithm of the activation norm, $\log|X_{\ell}|$, behaves as a Gaussian variable with a standard deviation of $O(\sqrt{\ell})$. This causes the magnitude $|X_{\ell}|$ to fluctuate between $e^{-O(\sqrt{\ell})}$ and $e^{O(\sqrt{\ell})}$ as depth $\ell$ increases. To counter this, we generate $O(\sqrt{l})$ candidate initializations. Denote by $\mathcal{D}_{\mathrm{input}}$ the input distribution on $\R^d$, for example the uniform distribution on the input values. We then select the initialization where the expected output norm $\mathbb{E}_{x_0 \sim \mathcal{D}_{\mathrm{input}}}\!\left[|X_{\ell}|\right]$ is closest to $1$. We stress that the calculation of $\mathbb{E}_{x_0 \sim \mathcal{D}_{\mathrm{input}}}\!\left[|X_{\ell}|\right]$ should be performed before training and hyperparameter tuning, and is therefore a computationally efficient way to improve Lyapunov initialization.

\section{Experiments}\label{section:Experiments}

\begin{figure}[b]
\begin{minipage}{0.42\textwidth}
    \centering    
    \includegraphics[width=\linewidth]{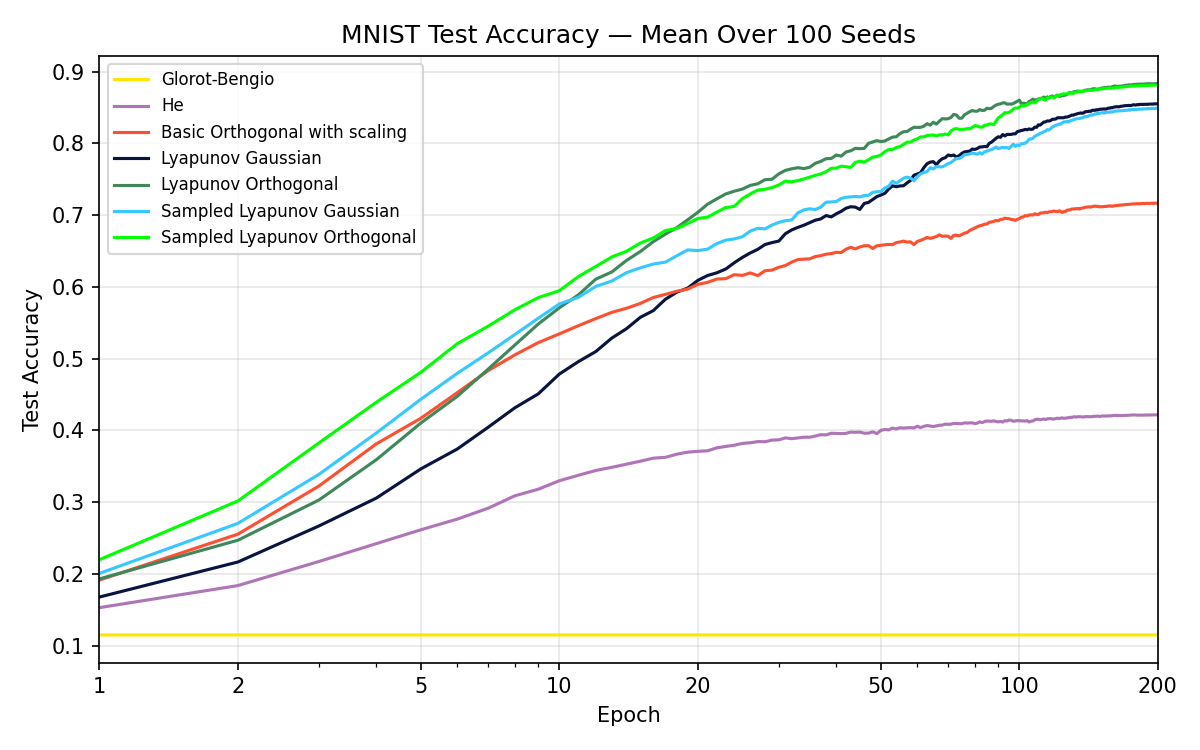}
    \end{minipage}
    \begin{minipage}{0.57\textwidth}
    \scriptsize
    \begin{tabular}{lcccccc}
\toprule
Method \quad\quad\quad\quad\quad  Epoch: & 1 & 5 & 15 & 50 & 200 \\
\midrule
Glorot-Bengio & 12\% & 12\% & 12\% & 12\% & 12\% \\
He &  15\% &	23\%	& 31\% &	34\% &	36\% \\
Scaled Orthogonal & 18\% &	36\% &	49\% &	55\% &	60\% \\
\textbf{Lyapunov Gaussian} & 16\% &	31\% &	49\% &	63\% &	76\% \\
\textbf{Lyapunov Orthogonal} & 19\% &	40\% &	62\% &	76\%	 & 84\% \\
\textbf{Sampled Lyapunov Gaussian} & 18\% &	38\% &	52\% &	61\% &	70\% \\
\textbf{Sampled Lyapunov Orthogonal} & 21\% &	44\% &	60\% &	70\% &	80\% \\
\bottomrule
\end{tabular}
\end{minipage}
\caption{MNIST results for a hidden dimension 10 and 100-layer network averaged over 100 seeds.}
\label{fig:exp:mnist100:losses}
\end{figure}

We perform three experiments to test for empirical advantages gained from our theoretical insights: 1) MNIST, 2) Learning a degree 5 polynomial and 3) Learning a score (discussed in the appendix). 

\paragraph{MNIST.} MNIST is a standard dataset for the task of recognizing handwritten digits. We score each initialization method on test accuracy, that is, the percentage of correctly recognized digits on a held-out test set. As all methods are inherently random, the results were averaged over 100 seeds. We trained a hidden dimension 10 model with 100 layers and a Leaky ReLU slope $\alpha = 0.1$, performed hyperparameter tuning and chose the best model with respect to a validation set. As we have 100 layers, the network is difficult to train and our new methods make training feasible. 

We compare our methods with Glorot-Bengio, He and standard scaled orthogonal initialization, that is scaled by $\eta_0 = \sqrt{2/(1 + \alpha^2)}$. All our four new methods strongly outperform the previous methods. We further remark that the sampled methods are better at epochs 1 and 5 than the unsampled ones, while for higher epochs the unsampled versions perform better. Thus, for MNIST, sampling gives an early advantage that surprisingly does not carry over to later epochs. 

In Figure~\ref{MNISTHists}, we show histograms for all 100 seeds. They reveal that either the model does not train at all or scores above $85\%$ on average. Therefore, the right initialization makes the training success much more likely. With this model setup, it appears to be difficult to reach more than $90\%$ accuracy, as for all our methods this happens rather rarely. The experiments are described in detail in Appendix~\ref{sec:app:exp-details:MNIST}.

\paragraph{Polynomial.} We learn the degree $5$ polynomial $p(x) = x^5+x^2-x$ on the interval $[-1.5, 1.5]$ with a $60$-layer network and width $d=2$. The same initialization methods as for MNIST are used. 

Since the Lyapunov exponent is negative for the previous methods, the initial polynomial is the zero polynomial, leading to a mean squared error of $\approx 3.19$, explaining why Glorot-Bengio and He initialization have this roughly constant error. On the other hand, the unsampled Lyapunov methods might have enormous coefficients by our Central Limit Theorem, leading to a large loss initially. Usually, this loss will become rapidly small during training, yet there are some outliers. We therefore, to avoid overweighting the large outliers, present in \Cref{fig:polynomiallosses} the median training loss from the best $80$ of initialization samples over 100 seeds. In Figure~\ref{PolynomialHists} we show all 100 seeds, showing precise statistics that our novel methods strongly improve on the previous ones.

All Lyapunov initializations outperform the previous initialization methods in the first few thousand training steps, while scaled orthogonal initialization eventually outperforms Lyapunov Gaussian initialization. These results are somewhat consistent with the proven benefits of orthogonal over Gaussian initialization as established in \cite{SaxeOrthogonal,Huorthogonal}. In this example, our sampled methods are useful, strongly outperforming the unsampled methods as the initializations with very large coefficients are discarded. The experiments are described in detail in Appendix~\ref{sec:app:exp-details:poly}.

\begin{figure}[t]
\begin{minipage}{0.42\textwidth}
    \centering    
    \includegraphics[width=\linewidth]{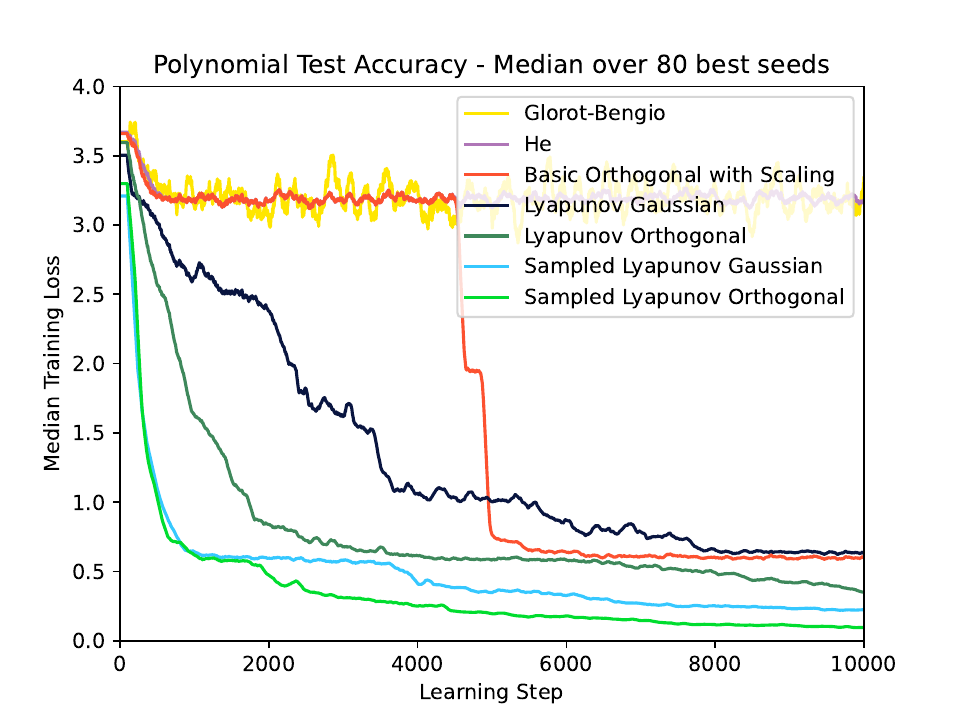}
    \end{minipage}
    \begin{minipage}{0.57\textwidth}
    \scriptsize
    \begin{tabular}{lccccc}
\toprule
Method \quad\quad\quad\quad  Training Step: & 500 & 2000  & 7000  & 10000 \\
\midrule
Glorot-Bengio & 3.27 &	3.02 &	3.20 &		3.34 \\
He &  3.22 &	3.17 &		3.18 &		3.18 \\
Scaled Orthogonal & 3.20 &	3.17 &		0.61 &	0.60 \\
\textbf{Lyapunov Gaussian} & 3.01 &	2.38 &	0.79 &		0.63 \\
\textbf{Lyapunov Orthogonal} & 2.57 &	0.84 &		0.54 &		0.35 \\
\textbf{Sampled Lyapunov Gaussian} & 1.10 &	0.60 &		0.27 &		0.22 \\
\textbf{Sampled Lyapunov Orthogonal} & 1.03 &	0.47  &		0.15 &		0.10 \\
\bottomrule
\end{tabular}
\end{minipage}
\caption{Results for polynomial experiment averaged over the best 80 seeds.}
\label{fig:polynomiallosses}
\end{figure}


\section{Limitations and Extensions} \label{sec:limitations-and-extensions}

\noindent \textbf{Beyond Leaky ReLU activation functions.} The main theoretical results of this paper only hold for Leaky ReLU activation functions and are likely false for all other non-linear activation functions. However, it would be interesting to investigate to what extent our results could at least partially be extended to other activation functions.

\noindent \textbf{Different Types of Architectures.} It is a natural question whether the results and methods of our paper could be adapted to $c$-homogeneous activation functions or to residual connections. However, this would not be straightforward and would require distinct ergodic results in comparison to the ones used in this paper and, therefore, would lead to novel directions of research.

\section{Conclusion}\label{section:Conclusion}

In this work, we established a rigorous probabilistic analysis of deep bias-free Leaky ReLU networks by proving a Law of Large Numbers and a Central Limit Theorem for the logarithm of network activations. Our results demonstrate that activation growth is governed by the Lyapunov exponent, a parameter that characterizes the sharp phase transition between vanishing and exploding activations. We revealed that standard initialization methods, such as He or orthogonal initialization, fail to guarantee stability in deep networks of low width because they result in negative Lyapunov exponents. To address this, we proposed Lyapunov initialization, which sets the Lyapunov exponent to zero to ensure maximum stability, alongside a sampled refinement to counter depth-dependent stochastic fluctuations resulting from the Central Limit Theorem. Empirical experiments confirm that these novel methods outperform previous initialization strategies in low-dimensional regimes, resulting in more efficient learning in deep architectures.

\newpage





\newpage 

\section*{Acknowledgments}

C.K.~ holds a Postdoc Mobility Fellowship from the
Swiss National Science Foundation (grant number 235409) and thanks the Institute for Advanced Study. 
T.S.~gratefully acknowledges financial support from the Rhodes Trust and the EPSRC Centre for Doctoral Training in Mathematics of Random Systems: Analysis, Modelling and Simulation (EP/S023925/1). S.K.~gratefully acknowledges support from
the Heilbronn Institute for Mathematical Research.
This work used the Scientific Compute Cluster at GWDG, the joint data center of the Max Planck Society (MPG) and the University of Göttingen. In part funded by the Deutsche Forschungsgemeinschaft (DFG, German Research Foundation), project number 405797229.

\bibliographystyle{plainnat}   
\bibliography{bibliography}

\newpage
\appendix
\onecolumn

\section{Organization of the Appendix}

For convenience of the reader, we comment on the organization of the appendix. In Section~\ref{section:Haar}, we briefly discuss Haar measures on topological groups such as Lie groups (including the cases $\mathrm{GL}_d(\R)$ and $\mathrm{O}(d)$). In Section~\ref{Section:ProofofLLN}, we prove the main results of this paper, that is, Theorem~\ref{ExplicitLeakyReLULLN} and Theorem~\ref{ExplicitLeakyReLUCLT} as well as Lemma~\ref{Lemma:LLNCounterexample}. In Section~\ref{section:AppendixFormulas}, we establish the formulas for the Lyapunov exponent, that is, we prove Theorem~\ref{LyapunovGaussianFormula}, Theorem~\ref{LyapunovOrthogonalFormula} and Theorem~\ref{AsymptoticExpansion} as well as provide lookup tables for the various parameters discussed in this paper. In Section~\ref{section:FailureCLT}, experimental evidence that our Central Limit Theorem (Theorem~\ref{ExplicitLeakyReLUCLT}) fails for $\tanh$ networks is shown. In Section~\ref{sec:Heorthcalc}, we show why He initialization and standard scaled orthogonal initialization preserve the second moment of the network activations across layers. We describe the experiments in detail in Section~\ref{sec:app:exp-details:MNIST} (MNIST), Section~\ref{sec:app:exp-details:poly} (polynomial) and Section~\ref{sec:app:score-long} (Score). 

\section{Haar Measures}\label{section:Haar}

In this section, we briefly discuss for convenience of the reader Haar measures on topological groups such as Lie groups.

A \textbf{topological group} is a group $G$ endowed with a topology such that the maps $(g,h) \mapsto gh$ and $g\mapsto g^{-1}$ are continuous. Recall that the Borel $\sigma$-algebra is the smallest $\sigma$-algebra containing all the open sets. A left Haar measure is a Borel measure that is invariant under left-multiplication and is defined as follows.

\begin{definition}\label{HaarDefinition}
    Let $G$ be a topological group. A Borel measure $m_G$ on $G$ is called a (left) \textbf{Haar measure} if $m_G(U) > 0$ for every open set $U\subset G$ and $$m_G(g \cdot C) = m_G(C)$$ for every Borel set $C\subset G$ and $g \in G$, where $g\cdot C = \{gc \,:\, c\in C \}$.
\end{definition}

To ensure that Haar measures exist, we need to introduce a few further definitions: A topological space $X$ is called 
\begin{enumerate}
    \item[(1)] \textbf{locally compact} if every point has a compact neighborhood, that is for every $x \in X$ there is an open set $U_x \subset X$ and a compact set $C \subset X$ with $x \in U_x$ and $U_x \subset C$. 
    \item[(2)] \textbf{Hausdorff} if for every $x,y \in X$ there exist open sets $U_x$ and $U_y$ satisfying that $x \in U_x$, $y \in U_y$ and that $U_x \cap U_y = \emptyset$.
\end{enumerate}

\begin{theorem}\citep[Chapter XI]{HalmosMT}\label{HaarTheorem}
    Let $G$ be a topological group that is locally compact and Hausdorff. Then there exists a left Haar measure. The Haar measure is unique up to scaling, that is if $m_G$ and $m_G'$ are left Haar measures then there is $\lambda > 0$ such that $m_G = \lambda \cdot m_G'$. 
    
    Moreover, $m_G(G) < \infty$ if and only if $G$ is compact.
\end{theorem}

In this paper, we work with the topological groups $$\mathrm{GL}_d(\R) = \{ W \in M_d(\R) \,:\, \det(M) \neq 0 \}$$ and  $$\mathrm{O}(d) = \{ W \in M_d(\R) \,:\, WW^T = W^TW = \mathrm{Id} \}.$$ These are Lie groups and therefore they are locally compact and Hausdorff. Thus they both have Haar measures that we denote as $m_{\mathrm{GL}_d(\R)}$ and $m_{\mathrm{O}(d)}$. Since $\mathrm{O}(d)$ is compact, we normalize $m_{\mathrm{O}(d)}$ to satisfy $m_{\mathrm{O}(d)}(\mathrm{O}(d)) = 1$ so that $m_{\mathrm{O}(d)}$ is a probability measure.

In the main part of the paper we have also discussed measures absolutely continuous to the Haar measure. These are defined as follows.

\begin{definition}\label{TopGpACDefinition}(Absolutely Continuous with respect to the Haar measure)
    Let $G$ be a topological group with Haar measure $m_G$ and let $\mu$ be a Borel measure on $G$. Then we say that $\mu$ is absolutely continuous to $m_G$ if for all Borel sets $C \subset G$ with $m_G(C) = 0$ it holds that $\mu(C) = 0$.
\end{definition}

The fundamental property of measures absolutely continuous to the Haar measure is that they have a density with respect to the Haar measure. This is an instance of the Radon-Nikodym Theorem, which we state here for the reader's convenience.

\begin{theorem}(\citep[See Chapter VI]{HalmosMT} Radon-Nikodym Theorem for Haar measures)\label{RadonNikodym}
    Let $G$ be a topological group with Haar measure $m_G$ and let $\mu$ be a Borel measure on $G$. Then $\mu$ is absolutely continuous to $m_G$ if and only if there is a density $p\in L^1(G,m_G)$ such that $d\mu = p\cdot dm_G$, i.e. for all Borel measurable sets $B \subset G$ it holds that $$\mu(B) = \int_B p(g) \, dm_G(g).$$
\end{theorem}

\renewcommand{\E}{\mathbb{E}}

\section{Proof of Limit Theorems}
\label{Section:ProofofLLN}

In this section we prove the main limit theorems of this paper, that is Theorem~\ref{ExplicitLeakyReLULLN} and Theorem~\ref{ExplicitLeakyReLUCLT} as well as Lemma~\ref{Lemma:LLNCounterexample}. We first give the idea of the proofs in Section~\ref{section:outlineofproofs}. In Section~\ref{section:AdditiveCocycles}, we expose additive cocycles, which is an important general concept used to establish an abstract Law of Large Numbers and Central Limit Theorem. We then state a Law of Large Numbers for cocycles (Theorem~\ref{CocycleGeneralResult}) in Section~\ref{section:CocycleLLN} followed by a Central Limit Theorem for cocycles (Theorem~\ref{CLTCocycleGeneral}) in Section~\ref{section:CocyleCLT}. In Section~\ref{Section:LLNPosHom} we establish a Law of Large Numbers on $\R^d$ in a more general setting than addressed in the main part of this paper (Theorem~\ref{PositivelyHomLLN}), namely working with positively homogeneous functions. We finally deduce Theorem~\ref{ExplicitLeakyReLULLN} in Section~\ref{section:ProofExplicitLLN} conditional on two assumptions, namely that there is a unique spherical $\mu$-stationary probability measure and that we have a finite second logarithmic moment. We establish general conditions for the uniqueness of spherical measures in Section~\ref{section:UniqueSphericalMeasure} and finally prove suitable logarithmic moment bounds in Section~\ref{section:FinitenessofMoments}. The Central Limit Theorem (Theorem~\ref{ExplicitLeakyReLUCLT}) is proved in Section~\ref{section:ProofExplicitCLT} by establishing the necessary properties from Theorem~\ref{CLTCocycleGeneral}. Finally, Lemma~\ref{Lemma:LLNCounterexample} is shown in Section~\ref{section:CounterexamplesLLN}. 

\subsection{Idea of proofs}\label{section:outlineofproofs}

The proof of the main results is based on similar ideas as for random matrix products, yet with a new viewpoint. Indeed, we decompose $X_{\ell}$ into a directional component on the unit sphere and a magnitude component. More precisely, for the unit sphere $\mathbb{S}^{d-1} = \{x \in \R^d \,:\, |x| = 1 \}$, one studies the spherical distribution $$S_{\ell} = \frac{X_{\ell}}{|X_{\ell}|} \in \mathbb{S}^{d-1}.$$ In the linear case, one works with projective space instead of the unit sphere and the observation that the theory can be developed with activation functions using the unit sphere is the novelty of our proof. The core of the proof of our main results then relies on showing that the Markov chain $S_{\ell}$ is uniformly ergodic (see Theorem~\ref{DoeblinTheorem}) and mixes exponentially fast towards a unique stationary distribution. 

One then studies the magnitude component $$Z_{\ell} = \log \frac{|X_{\ell}|}{|X_{\ell - 1}|}$$ and observes $$\log |X_{\ell}| = Z_{\ell} + Z_{\ell - 1} + \ldots + Z_1.$$ In contrast to the standard Law of Large Numbers, the $Z_i$ are neither independent nor are they identically distributed. However, the ergodicity of $S_{\ell}$, in other words, the rapid forgetting of the initial input direction, allows to deduce our results from a certain abstract Law of Large Numbers (Theorem~\ref{CocycleGeneralResult}) and Central Limit Theorem (Theorem~\ref{CLTCocycleGeneral}). 

We now explain how exactly the ergodicity of $S_{\ell}$ leads to the main results, which also reveals the reason why our strategy only works for Leaky ReLU activation functions. Indeed, we observe that for $\phi = \max(x,\alpha x)$ with $\alpha \in \R_{\neq 0}$, $W\in M_d(\R)$ with $\det(W) \neq 0$ and assuming $X_{\ell} = \phi(WX_{\ell - 1})$ the following fundamental equality holds:
\begin{equation}\label{LeakyReLUReason}
    Z_{\ell} = \log \frac{|X_{\ell}|}{|X_{\ell - 1}|} = \log \frac{|\phi(WX_{\ell - 1})|}{|X_{\ell - 1}|} = \log |\phi(W\tfrac{X_{\ell - 1}}{|X_{\ell - 1}|})| = \log |\phi(WS_{\ell-1})|.
\end{equation}

The latter equation holds for the Leaky ReLU activation function 
as it is positively homogeneous, that is $\phi(\lambda x) = \lambda\phi(x)$ for every $\lambda > 0$ and $x \in \R$. Thus, $Z_{\ell}$ only depends on $S_{\ell-1}$ and therefore convergence results for $S_{\ell-1}$, using in our case uniform ergodicity, can be translated to limit theorems for $Z_{\ell}$. 

In contrast, other standard activation functions such as $\tanh$ or $\text{sigmoid}$ do not satisfy this positive homogeneity, which prevents the proof strategy from working. Indeed, as shown in Section~\ref{section:FailureCLT} the Central Limit Theorem fails for example for $\tanh$ networks.

\subsection{Additive Cocycles}\label{section:AdditiveCocycles}

We deduce Theorem~\ref{ExplicitLeakyReLULLN} and Theorem~\ref{ExplicitLeakyReLUCLT} with similar methods to well-known results in the theory of random walks on Lie groups. Two excellent references for this topic are \cite{BougerolLacroix1985} and \cite{BenoistQuintRandomBook}. We recall some abstract definitions and results from these two books that we use for our proofs.

\begin{definition}
    A \textbf{topological semigroup} $G$ is a semigroup $G$ that is endowed with a topology such that the multiplication map $G \times G \to G, (g_1,g_2) \mapsto g_1g_2$ is continuous. 
\end{definition}

\begin{definition}
    Let $G$ be a topological semigroup. We say that $G$ \textbf{acts continuously} on the topological space $X$ if there  is a continuous map $G \times X \to X, (g,x) \mapsto g. x$ such that $$(g_1g_2).x = g_1.(g_2.x)$$ for all $g_1, g_2 \in G$ and $x \in X$.
\end{definition}

\begin{definition}
    Let $G$ be a topological semigroup acting continuously on a topological space $X$. A continuous map $c: G \times X \to \R$ is called an \textbf{additive cocycle} if $$c(g_1g_2,x) = c(g_1, g_2.x) + c(g_2,x)$$ for all $g_1,g_2\in G$ and $x \in X$.
\end{definition}

To motivate this definition, we explain how these appear in the study of random matrix products. Indeed, we denote 
\begin{equation}\label{YlDef}
    Y_{\ell} = W_{\ell}W_{\ell - 1}\cdots W_1 x_0
\end{equation}
for some fixed $x_0$ and $W_i$, as before, an i.i.d sequence of matrices in $$\mathrm{GL}_d(\R) = \{ W \in M_d(\R) \,:\, \det(W) \neq 0 \}$$ distributed as the measure $\mu$. A law of large numbers for $|Y_{\ell}|$ analogously to Theorem~\ref{ExplicitLeakyReLULLN} is well-known and shown in \citep[Part A III.7] {BougerolLacroix1985} and \citep[Chapter 4.6]{BenoistQuintRandomBook}. 

The central idea in proving a law of large numbers for $Y_{\ell}$ is to consider the $\mathrm{GL}_d(\R)$ action on the space of directions of $\R^d$. Mathematically, this is captured by projective space $$\mathbb{P}\R^d = (\mathbb{R}^d \backslash \{0 \})/\R_{\neq 0},$$ i.e. the set of equivalence classes of $\R^d \backslash \{0 \}$ under the equivalence relation given for $x,y \in \R^d \backslash \{0 \}$ as $x \sim y$ if and only there is $\lambda \in \R_{\neq 0}$ such that $x = \lambda y$. For $x \in \R^d \backslash \{0\}$ we denote by $[x] \in \mathbb{P}\R^d$ the associated element in projective space. We then consider for $W \in \mathrm{GL}_d(\R)$ and $[x] \in \mathbb{P}\R^d$ the action $$W.[x] = [Wx],$$ which is well-defined since $W$ is a matrix and therefore a linear map.

We then want to study the additive cocycle 
\begin{equation}\label{MatrixCocycle}
    c_{\mathrm{mat}}(W,[x]) = \log \frac{|Wx|}{|x|}
\end{equation}
for $W \in \mathrm{GL}_d(\R)$ and $[x] \in \mathbb{P}\R^d$, which again is easily seen to be well-defined. The map $c_{\mathrm{mat}}$ is indeed an additive cocycle since for $W_1, W_2 \in G$ and $[x] \in \mathbb{P}\R^d$ we have 
\begin{align*}
    c_{\mathrm{mat}}(W_1W_2,[x]) &= \log \frac{|W_1W_2x|}{|x|} \\ &= \log \frac{|W_1W_2x|}{|W_2x|} + \log \frac{|W_2x|}{|x|} \\  &= c_{\mathrm{mat}}(W_1, W_2.[x]) + c_{\mathrm{mat}}(W_2,[x]).
\end{align*}

We now want to study the cocycle $c_{\mathrm{mat}}$ for the random walk $Y_{\ell}$. 

\subsection{Law of Large Numbers for Additive Cocycles}\label{section:CocycleLLN}

\begin{definition}
    Let $G$ be a semigroup acting continuously on a compact space $X$. Given a measure $\mu$ on $G$ and $\nu$ on $X$ we denote by $\mu*\nu$ the \textbf{convolution} of $\mu$ and $\nu$, which is the measure on $X$ determined uniquely by $$(\mu * \nu)(f) = \int f(g.x) \, d\mu(g) d\nu(x)$$ for every continuous function $f:X \to \R$.
\end{definition}

\begin{definition}
    Let $G$ be a semigroup acting continuously on a compact space $X$. Let $\mu$ be a measure on $G$ and $\nu$ be one on $X$. Then we say that $\nu$ is a \textbf{$\mu$-stationary measure} if $$\mu * \nu = \nu.$$
\end{definition}

Under the assumption that $\mu$ has a unique stationary measure on $X$, the following result is known and implies the Law of Large Numbers for $Y_{\ell}$, as we explain below. 

\begin{theorem}\citep[Follows from Theorem 3.9]{BenoistQuintRandomBook}\label{CocycleGeneralResult}
    Let $G$ be a topological semigroup acting continuously on a compact metric space $X$. Let $\mu$ be a probability measure on $G$ and assume that there is a unique $\mu$-stationary probability measure on $X$. Let $c:G \times X \to \R$ be an additive cocycle with 
    \begin{equation}\label{CocycleL1Condition}
        \int_G c_{\mathrm{sup}}(g) \, d\mu(g) < \infty\quad \quad \text{ for } \quad\quad c_{\mathrm{sup}}(g) = \sup_{x \in X}|c(g,x)|.
    \end{equation} 
    
    Then there is a constant $\lambda_{\mu, c} \in \R$ such that for all $x\in X$, and $\mathbb{P}_{\mu}$-almost all $\omega \in \Omega$ it holds that $$\frac{1}{\ell} c(g_{\ell}(\omega)\cdots g_1(\omega),x) \to \lambda_{\mu,c}$$ as $\ell \to \infty$. Moreover, the convergence also holds in $L^1$ uniformly for all $x \in X$, i.e. 
    \begin{equation}\label{UnifL1Conv}
       \lim_{\ell \to \infty} \sup_{x \in X} \mathbb{E}\left[ \bigg| \frac{1}{\ell} c(g_\ell\cdots g_1,x) - \lambda_{\mu,c} \bigg| \right] = 0. 
    \end{equation}

    In addition, if $\nu$ is the unique stationary probability measure on $X$, then we can express $\lambda_{\mu,c}$ as $$\lambda_{\mu,c} = \int_{G \times X} c(g,x) \, d\mu(g)d\nu(x).$$
\end{theorem}

We remark that \citep[Theorem 3.9]{BenoistQuintRandomBook} is stated with the additional assumption that $G$ is second countable and locally compact. However, this assumption is not necessary. Indeed, the proof in \citep[Theorem 3.9]{BenoistQuintRandomBook} proceeds via Birkhoff's ergodic theorem on the shift space $G^{\mathbb{N}}$ together with uniqueness of the $\mu$-stationary measure; neither ingredient requires second countability or local compactness of $G$. Moreover, in the setting in which we will apply Theorem~\ref{CocycleGeneralResult}, we can also conclude the required results by applying a martingale Law of Large Numbers together with the Doeblin condition from  section~\ref{section:UniqueSphericalMeasure}.

One can directly deduce a law of large numbers for $Y_{\ell}$ from Theorem~\ref{CocycleGeneralResult} under the assumption that a given measure $
\mu$ on $\mathrm{GL}_d(\R)$ has a unique stationary measure on $\mathbb{PR}^d$. For convenience of the reader, we state it here. 

\begin{corollary}\label{StandardMatrixLLN}
    Let $\mu$ be a probability measure on $\mathrm{GL}_d(\R)$ satisfying $\int \log N(W) \, d\mu(W) < \infty$. Let $x_0 \in \R^d\backslash \{0\}$ and let $Y_{\ell}$ be as in \eqref{YlDef}. Moreover, assume that there is a unique $\mu$-stationary measure on $\mathbb{PR}^d$. 
    
    Then there exists a real number $\lambda_{\mu} \in \R$ depending only on $\mu$ such that for each fixed $x_0 \in \R^d \backslash \{ 0\}$ as $\ell \to \infty$,
     \begin{equation}\label{MatNormLLN}
         \frac{1}{\ell} \log |Y_{\ell}(\omega)| \longrightarrow \lambda_{\mu}
     \end{equation} for $\mathbb{P}_{\mu}$-almost all $\omega \in \Omega$. Moreover, the convergence is in $L^1$ uniformly for $x_0$ of constant modulus, i.e. $$\lim_{\ell \to \infty}\sup_{x_0 \in \R^d \backslash \{ 0\}} \mathbb{E} \left[ \bigg|\frac{1}{\ell}\left( \log |Y_{\ell}(\omega)| - \log |x_0| \right) - \lambda_{\mu} \bigg| \right] =0.$$
\end{corollary}

\begin{proof}
    This follows directly from Theorem~\ref{CocycleGeneralResult}. Indeed, in order to apply Theorem~\ref{CocycleGeneralResult} for the $\mathrm{GL}_d(\R)$ action on $\mathbb{PR}^d$ and to $c_{\mathrm{mat}}$ from \eqref{MatrixCocycle}, we need to check the condition \eqref{CocycleL1Condition}. The latter follows by proving for all $x \in \R^d\backslash\{0\}$ and $W \in \mathrm{GL}_d(\R)$ that 
    \begin{equation}\label{NormIneq}
        |c_{\mathrm{mat}}(W,x)| = \bigg| \log \frac{|Wx|}{|x|}\bigg| \leq \log N(W),
    \end{equation} which implies \eqref{CocycleL1Condition} by our assumption $\int \log N(W) \, d\mu(W) < \infty$.
    
    To prove \eqref{NormIneq}, note that by $\frac{|Wx|}{|x|} \leq ||W||$ as well as $|x| = |W^{-1}Wx| \leq ||W^{-1}|| \, |Wx|$. So it follows that $$-\log ||W^{-1}|| \leq \log \frac{|Wx|}{|x|} \leq \log ||W||.$$ Moreover, since $||W|| \leq N(W)$ we have $\log ||W|| \leq  \log N(W)$ and as $||W^{-1}|| \leq N(W)$ we have $-\log N(W) \leq - \log ||W^{-1}||$. Thus it follows that $$-\log N(W) \leq \log \frac{|Wx|}{|x|} \leq \log N(W),$$ which implies \eqref{NormIneq} as $1 = ||WW^{-1}|| \leq ||W|| \cdot ||W^{-1}||$ which shows that either $||W|| \geq 1$ or $||W^{-1}|| \geq 1$ and so $\log \max \{ ||W||, ||W^{-1}|| \} \geq 0$.
\end{proof}

In \cite{BougerolLacroix1985} and \cite{BenoistQuintRandomBook} results stronger than Corollary~\ref{StandardMatrixLLN} are proven. In fact, instead of assuming that $\mu$ has a unique stationary measure, it suffices to assume that $\mu$ is irreducible, that is, if a subspace $V \subset \R^d$ satisfies $gV \subset V$ for all $g \in \mathrm{supp}(\mu)$, then $V = \{0\}$ or $V = \R^d$. As in this paper, we care about absolutely continuous measures, we do not discuss the irreducible case further. It is well-known that if $\mu$ is an absolutely continuous probability measure on $\mathrm{GL}_d(\R)$, then there is a unique $\mu$-stationary measure on $\mathbb{PR}^d$.

\begin{proposition}\citep[Follows from Proposition 3.7]{BenoistQuintRandomBook} \label{StandardLLNUniqueStat} If $\mu$ is an absolutely continuous probability measure on $\mathrm{GL}_d(\R)$, then there is a unique $\mu$-stationary probability measure on $\mathbb{PR}^d$. In particular, if in addition $\int \log N(W) \, d\mu(W) < \infty$, then the conclusion of Corollary~\ref{StandardMatrixLLN} holds for $\mu$.
\end{proposition}

\subsection{Central Limit Theorem for Additive Cocycles}\label{section:CocyleCLT}

We next state a Central Limit Theorem for abstract cocycles from \cite{BenoistQuintCLT}. To do so, we introduce some further definitions from \cite{BenoistQuintCLT}.

\begin{definition}\label{MoreCocycleDefs}
    Let $G$ be a semigroup acting continuously on a compact space $X$ and let $c: G \times X \to \R$ be an additive cocycle. Moreover, let $\mu$ be a probability measure on $G$. We then say that $c$  
    \begin{enumerate}
        \item has \textbf{constant drift with average} $\lambda_{\mu,c} \in \R$ such that for all $x\in X$ we have that $$ \int c(g,x) \, d\mu(g) = \lambda_{\mu,c}.$$
        \item is \textbf{centerable} if there exists a constant drift additive cocycle $c_0:G \times X \to \R$ and a continuous function $\psi : X \to \R$ such that for all $g \in G$ and $x\in X$, $$c(g,x) = c_0(g,x) + \psi(x) - \psi(g.x).$$
    \end{enumerate}
\end{definition}

We can now state an abstract Central Limit Theorem.

\begin{theorem}\citep[Theorem 3.4]{BenoistQuintCLT}\label{CLTCocycleGeneral}
    Let $G$ be a topological semigroup acting continuously on a compact metric space $X$. Let $\mu$ be a probability measure on $G$ and assume that there is a unique $\mu$-stationary probability measure on $X$. Let $c:G \times X \to \R$ be a centerable additive cocycle with average $\lambda_{\mu,c}$ and satisfying
    \begin{equation}\label{BQCocycleL1Condition}
        \int_G c_{\mathrm{sup}}(g)^2 \, d\mu(g) < \infty\quad \quad \text{ for } \quad\quad c_{\mathrm{sup}}(g) = \sup_{x \in X}|c(g,x)|.
    \end{equation} 
    
    Then there exists a constant $\gamma_{\mu,c} \geq 0$ such that for every $x\in X$ we have for i.i.d. $g_{i} \sim \mu$ that \begin{equation}\label{CocycleCLT}
\frac{c(g_{\ell}\cdots g_1,x) - \ell \lambda_{\mu, c}}{\sqrt{\ell}} \overset{\ell \to \infty}{\longrightarrow} \mathcal{N}(0,\gamma_{\mu,c}^2),
\end{equation} where the convergence is in distribution.
\end{theorem}

As for Theorem~\ref{CocycleGeneralResult}, the original setting of \citep[Theorem 3.4]{BenoistQuintCLT} states that $G$ is locally compact, which again is not necessary. Furthermore, the Leaky ReLU Central Limit Theorem from this paper can be deduced by applying a martingale Central Limit Theorem together with the Doeblin condition from  section~\ref{section:UniqueSphericalMeasure}.

\subsection{A Law of Large Numbers for positively homogeneous functions}\label{Section:LLNPosHom}

We equally want to deduce Theorem~\ref{ExplicitLeakyReLULLN} from Theorem~\ref{CocycleGeneralResult}, yet we will apply Theorem~\ref{CocycleGeneralResult} to a different cocycle. The novel and central observation of this paper is that in order to apply Theorem~\ref{CocycleGeneralResult} in this setting, instead of working with projective space, we use the compact space $$X = (\R^d\backslash \{ 0\})/\R_{>0},$$ i.e. the set of equivalence classes of $\R^d \backslash \{0 \}$ under the equivalence relation given for $x,y \in \R^d \backslash \{0 \}$ as $x \sim y$ if and only there is $\lambda \in \R_{> 0}$ such that $x = \lambda y$. The difference to projective space is that $x$ and $-x$ are not in the same equivalence class. For an element $x \in \R^{d}$ we denote by $\overline{x} \in X$ the associated element in $X$. We observe that $X$ is in bijection to the $(d-1)$-dimensional sphere $\mathbb{S}^{d-1} = \{ x\in \R^d \,:\, |x| = 1 \}$ via the map $$\mathbb{S}^{d-1} \to X \quad\quad x \mapsto \overline{x}.$$
 
Our methods do not only apply to Leaky ReLU activation functions, but also to  $\phi(x) = \max(\alpha_1 x,\alpha_2 x)$ with $\alpha_1,\alpha_2 \in \R_{\neq 0}$, which we call a \textbf{generalized Leaky ReLU activation function}. For each $W \in \mathrm{GL}_d(\R)$ and $\overline{x} \in X$ we then consider the map $$\phi_W(\overline{x}) := \overline{\phi(Wx)} \in X
,$$ which is not only well-defined for Leaky ReLU activation functions, yet also for $\phi$ a generalized Leaky ReLU activation function. The function $\phi_W$ would not be well-defined on projective space, which is the reason we work with the space $X$. 

Instead of just working with the functions $\phi_W$, we establish in this section a Law of Large Numbers for a wider range of functions, namely positively homogeneous functions. We say that a function $f:\R^d \to \R^d$ is \textbf{positively homogeneous} if $$f(\alpha x) = \alpha f(x)$$ for all $\alpha > 0$ and $x \in \R^d$ and denote $$\mathscr{C}_{\mathrm{hom}}(\R^d) = \{ f:\R^d \to \R^d \text{ continuous and positively homogeneous}\}.$$ In order for a function $f \in \mathscr{C}_{\mathrm{hom}}(\R^d)$ to lead to a well-defined function on $X$, we need to assume that $f(x) \neq 0$ for all $x \in \R^d \backslash \{0\}$. So we consider the subset of $\mathscr{C}_{\mathrm{hom}}(\R^d)$ given by $$\mathscr{C}^{\times}_{\mathrm{hom}}(\R^d) = \{ f\in \mathscr{C}_{\mathrm{hom}}(\R^d) \,:\, f(x) \neq 0 \text{ for all } x\in \R^d\backslash \{ 0\} \}.$$

We note that $\mathscr{C}^{\times}_{\mathrm{hom}}(\R^d)$ forms a topological semigroup via the composition of functions and we consider the metric $d(f_1,f_2) = \sup_{x\in \mathbb{S}^{d-1}} |f_1(x) - f_2(x)|$ for $f_1, f_2 \in \mathscr{C}^{\times}_{\mathrm{hom}}(\R^d)$. The space $\mathscr{C}^{\times}_{\mathrm{hom}}(\R^d)$ is thus a topological semigroup that acts continuously on the compact space $X$. 

The cocycle we want to consider is given for $f \in \mathscr{C}^{\times}_{\mathrm{hom}}(\R^d)$ and $x \in \R^d\backslash\{0\}$ as 
\begin{equation}\label{Def:chom}
    c_{\mathrm{hom}}(f,\overline{x})  =  \log \frac{|f(x)|}{|x|}.
\end{equation}
for any representative $x \in \mathbb{R}^d \backslash \{0\}$ of $\overline{x}$. The map $c_{\mathrm{hom}}$ is indeed an additive cocycle since for $f_1,f_2 \in \mathscr{C}^{\times}_{\mathrm{hom}}(\R^d)$ and $x \in \R^d \backslash \{0\}$ we have 
\begin{align*}
    c_{\mathrm{hom}}(f_1\circ f_2,\overline{x}) &= \log \frac{|f_1(f_2(x))|}{|x|} \\ &= \log \frac{|f_1(f_2(x))|}{|f_2(x)|} + \log \frac{|f_2(x)|}{|x|} \\  &= c_{\mathrm{hom}}(f_1, f_2(\overline{x})) + c_{\mathrm{hom}}(f_2,\overline{x}).
\end{align*}

We use the following definition of a spherical stationary measure. This definition is justified as $X$ is diffeomorphic to $\mathbb{S}^{d-1}$.

\begin{definition}\label{Def:SS}
    Let $\eta$ be a Borel probability measure on $\mathscr{C}^{\times}_{\mathrm{hom}}(\R^d)$. We say that a probability measure $\nu$ on $X$ is a  \textbf{spherical $\eta$-stationary probability measure} if $\nu$ is an $\eta$-stationary measure on $X$.
\end{definition}

We can then state the following general law of large numbers for positively homogeneous functions. For a probability measure $\eta$ on $\mathscr{C}^{\times}_{\mathrm{hom}}(\R^d)$ we denote, analogously to \eqref{ProbSpace}, by $$(\Omega, \mathscr{F}, \mathbb{P}_{\eta}) = (\mathscr{C}^{\times}_{\mathrm{hom}}(\R^d)^{\mathbb{N}}, \mathscr{B}(\mathscr{C}^{\times}_{\mathrm{hom}}(\R^d))^{\mathbb{N}}, \eta^{\mathbb{N}})$$ the space of $\mathscr{C}^{\times}_{\mathrm{hom}}(\R^d)$-valued sequences and by $f_n(\omega)$ the $n$-th element of $\omega = (f_1, f_2,\ldots)$. We then write for $\omega = (f_1, f_2,\ldots)$, $$X_{\ell}(\omega) = f_{\ell}(X_{\ell -1}(\omega))   \quad\quad \text{ and } \quad\quad X_{0} = x_0$$

\begin{theorem}(Law of Large Numbers for Positively Homogeneous Functions)\label{PositivelyHomLLN}
     Let $\eta$ be a Borel probability measure on $\mathscr{C}^{\times}_{\mathrm{hom}}(\R^d)$ with a unique spherical $\eta$-stationary probability measure and assume for $c_{\mathrm{hom}}$ from \eqref{Def:chom} that \begin{equation}\label{HomCocycleL1Condition}
        \int_{\mathscr{C}^{\times}_{\mathrm{hom}}(\R^d)} c_{\mathrm{hom, sup}}(f) \, d\eta(f) < \infty\quad \quad \text{ with } \quad\quad c_{\mathrm{hom, sup}}(f) = \sup_{x \in X}|c_{\mathrm{hom}}(f,x)|.
    \end{equation} 
     
     Then there exists a real number $\lambda_{\eta} \in \R$ depending only on $\eta$ such that for each fixed $x_0 \in \R^d \backslash \{ 0\}$ as $\ell \to \infty$,
     \begin{equation}\label{NormLLN}
         \frac{1}{\ell} \log |X_{\ell}(\omega)| \longrightarrow \lambda_{\eta}
     \end{equation} for $\mathbb{P}_{\eta}$-almost all $\omega \in \Omega$. Moreover, the convergence is in $L^1$ uniformly for $x_0$ of constant modulus, i.e. $$\lim_{\ell \to \infty}\sup_{x_0 \in \R^d \backslash \{ 0\}} \mathbb{E} \left[ \bigg|\frac{1}{\ell}\left( \log \frac{|X_{\ell}(\omega)|}{|x_0|}  \right) - \lambda_{\eta} \bigg| \right] =0.$$

     In addition, if $\nu$ is the unique $\eta$-spherical stationary measure on $X$, then 
    \begin{equation}
        \lambda_{\eta} = \int_{\mathscr{C}^{\times}_{\mathrm{hom}}(\R^d)\times X} c_{\mathrm{hom}}(f,\overline{x}) \, d\eta(f)d\nu(\overline{x}).
    \end{equation}
\end{theorem}

\begin{proof}
    This follows directly from Theorem~\ref{CocycleGeneralResult} applied to $c_{\mathrm{hom}}$.
\end{proof}

Theorem~\ref{PositivelyHomLLN} is, to the authors' knowledge, the first instance of a Law of Large Numbers for compositions of functions beyond the linear case. 

In order to deduce Theorem~\ref{ExplicitLeakyReLULLN}, we state the following lemma.

\begin{lemma}\label{lemma:finmom}
    Let $\mu$ be a Borel probability measure on $\mathrm{GL}_d(\R)$ satisfying $\int (\log N(W))^k \, d\mu(W) < \infty$ for $k \geq 1$.  Then for a generalized Leaky ReLU activation function $\phi(x) = \max(\alpha_1 x,\alpha_2 x)$ with $\alpha_1, \alpha_2 \in \R_{\neq 0}$ it holds that $$\int c_{\mathrm{hom, sup}}(\phi_W)^k \, d\mu(W) < \infty\quad \quad \text{ for } \quad\quad c_{\mathrm{hom, sup}}(\phi_W) = \sup_{x \in X}|c_{\mathrm{hom}}(\phi_W,x)|.$$
\end{lemma}

\begin{proof}
    Note that  $$\min\{ |\alpha_1|,|\alpha_2| \} |Wx| \leq |\phi(Wx)| \leq \max\{ |\alpha_1|,|\alpha_2| \} |Wx|,$$ which follows coordinate-wise from $\phi^2(y_i)\in[\min(\alpha_1^2,\alpha_2^2)\,y_i^2,\;\max(\alpha_1^2,\alpha_2^2)\,y_i^2]$ and by summing over $i$ and taking square roots. Therefore, there is an absolute constant $A = A(\alpha_1, \alpha_2) > 0$ depending only on $\alpha_1$ and $\alpha_2$ such that $$\bigg|\log\frac{|\phi(Wx)|}{|x|}\bigg|^k \leq \left(A + \bigg|\log\frac{|Wx|}{|x|}\bigg|\right)^k.$$
    In the proof of Corollary~\ref{StandardMatrixLLN}, we have shown that $$\bigg|\log\frac{|Wx|}{|x|}\bigg| \leq \log N(W)$$ and so it suffices to deduce that $$\int (A + \log N(W))^k \, d\mu(W) < \infty.$$ The latter follows by applying the inequality $(a + b)^k \leq 2^k(a^k + b^k)$ for all $a,b \in \R_{\geq 0}$ and $k \geq 1$ and therefore $$\int (A + \log N(W))^k \, d\mu(W) \leq 2^k \int A^k + (\log N(W))^k \, d\mu(W) < \infty,$$ concluding the proof by having used the assumption.
\end{proof}

\subsection{Law of Large Numbers: General Result and Deduction of Theorem \ref{ExplicitLeakyReLULLN}}\label{section:ProofExplicitLLN}
We will deduce the Law of Large Numbers \ref{ExplicitLeakyReLULLN} stated in the main part of this paper as a special instance of more general results. Nonetheless, the verification that all the requirements are satisfied will take some further work established in the next subsection.

As before, we denote for $W \in \mathrm{GL}_d(\R)$ by 
\begin{equation}\label{N(W)Definition}
    N(W) =  \max\{ ||W||, ||W^{-1}|| \}.
\end{equation} In the previous section, we explained that our methods also apply to generalized Leaky ReLU activation functions, so we state our general Law of Large Numbers in the following generality.

\begin{theorem}(General Law of Large Numbers for $|X_{\ell}|$)\label{UniqueStatLeakyReLULLN}
     Let $\mu$ be a Borel probability measure on $\mathrm{GL}_d(\R)$. Let $(X_{\ell})_{\ell\geq 0}$ be as in \eqref{ProbNNDef} and let $\phi(x) = \max(\alpha_1 x,\alpha_2 x)$ be a generalized Leaky ReLU activation function with $\alpha_1, \alpha_2 \in \R_{\neq 0}$. Assume that $\mu$ has a unique spherical $\mu$-stationary probability measure\footnote{See Definition~\ref{Def:SS}} and has finite first logarithmic moment, that is $\int \log N(W) \, d\mu(W) < \infty$. 
     
     Then there exists a real number $\lambda_{\mu,\phi} \in \R$ depending on $\mu$ and $\phi$ such that for each fixed $x_0 \in \R^d \backslash \{ 0\}$ as $\ell \to \infty$,
     \begin{equation}\label{NormLLNGeneral}
         \frac{1}{\ell} \log |X_{\ell}(\omega)| \longrightarrow \lambda_{\mu,\phi}
     \end{equation} for $\mathbb{P}_{\mu}$-almost all $\omega \in \Omega$. Moreover, the convergence is in $L^1$ uniformly for $x_0$ of constant modulus, i.e.  $$\lim_{\ell \to \infty}\sup_{x_0 \in \R^d \backslash \{ 0\}} \mathbb{E} \left[ \bigg|\frac{1}{\ell}\left( \log \frac{|X_{\ell}(\omega)|}{|x_0|}\right) - \lambda_{\mu,\phi} \bigg| \right] =0.$$
    Moreover, if $\nu$ is the unique $\mu$-spherical stationary measure on $X$, then 
    \begin{equation}\label{UniqueStatLyapFormula}
        \lambda_{\mu,\phi} = \int_{\mathrm{GL}_d(\R)\times X} c_{\mathrm{hom}}(\phi_W,\overline{x}) \, d\mu(W)d\nu(\overline{x}).
    \end{equation}
\end{theorem}

It is straightforward to deduce Theorem~\ref{UniqueStatLeakyReLULLN} from the previous results.

\begin{proof}(of Theorem~\ref{UniqueStatLeakyReLULLN})
    We consider the pushforward measure $\eta = \phi_{*}\mu$. By Lemma~\ref{lemma:finmom}, the assumptions of Theorem~\ref{PositivelyHomLLN} are satisfied for $\eta$. So the claim follows from Theorem~\ref{PositivelyHomLLN}.
\end{proof}

To conclude the proof of Theorem~\ref{ExplicitLeakyReLULLN}, we therefore need to verify that when $\mu$ satisfies Assumption~\ref{MeasureAssumptions} or Assumption~\ref{SecondMeasureAssumptions} that then there is a unique spherical $\mu$-stationary probability measure and it holds that $\int \log N(W) \, d\mu(W)< \infty$. For convenience of the reader we summarize the completed proof of Theorem~\ref{ExplicitLeakyReLULLN} here using results from the next two subsections.

\begin{proof}(of Theorem~\ref{ExplicitLeakyReLULLN}) We distinguish the two cases:
\begin{enumerate}
    \item[(1)] \textbf{$\mu$ satisfies Assumption~\ref{MeasureAssumptions}:} We note that the measure almost surely takes values in $\mathrm{GL}_d(\R)$ so without loss of generality, we can view $\mu$ to be defined on $\mathrm{GL}_d(\R)$. To verify the two assumptions of Theorem~\ref{UniqueStatLeakyReLULLN}, we note that there is a unique spherical $\mu$-stationary probability measure by Proposition~\ref{UniquenessofStationaryMeasure} and the first logarithmic moment is finite by Proposition~\ref{FinitenessofMoment}.
    \item[(2)] \textbf{$\mu$ satisfies Assumption~\ref{SecondMeasureAssumptions}:} Since the measure is supported on a scalar multiple of the orthogonal group, that is $\eta \cdot \mathrm{O}(d)$ for some $\eta > 0$, it holds that $||W||_{\mathrm{op}} = \eta$ and $||W^{-1}||_{\mathrm{op}} = \eta^{-1}$ for all $W$ in the support of $\mu$. Thus, the first logarithmic moment is finite. Uniqueness of the stationary measure again follows from Proposition~\ref{UniquenessofStationaryMeasure}.
\end{enumerate}
\end{proof}

\subsection{Uniqueness of Spherical Stationary Measure}\label{section:UniqueSphericalMeasure}

Throughout this section, we again denote by $$X = (\R^d\backslash\{ 0\})/\R_{>0}$$ and for a generalized Leaky ReLU activation function $\phi(x) = \max(\alpha_1 x,\alpha_2 x)$ for $\alpha_1, \alpha_2 \in \R_{\neq 0}$ we consider the action given with $\phi_W$ for $W\in \mathrm{GL}_d(\R)$ and $\overline{x} \in X$ by $$\phi_W(\overline{x}) = \overline{\phi(Wx)} \in X.$$ 

To discuss the uniqueness of stationary measures, denote by $P^{\ell}$ the $\ell$-th step transition density of the Markov chain on $X$ given by $\mu$. Indeed, for $x \in X$ and $A \in \mathscr{B}(X)$, where $\mathscr{B}(X)$ is the Borel $\sigma$-algebra on $X$, $$P(x,A)  = \int 1_A(\phi(Wx)) \, d\mu(W) = \mathbb{P}[\phi(Wx) \in A]$$ and we recursively define $P^1(x,A) = P(x,A)$ and for $\ell \geq 1$, $$ P^{\ell+1}(x, A) = \int_{X} P(x,  dy) P^{\ell}(y, A). $$ To introduce further notation, for a signed measure $\lambda$ on $X$ we define $$||\lambda||_{\mathrm{TV}} = \sup_{\substack{f: X \to \mathbb{R} \\ ||f||_{\infty} \leq 1 }} |\lambda(f)|.$$

We state the next result with Theorem~\ref{ExplicitLeakyReLULLN} in mind. We remark that when $\mu$ satisfies Assumption~\ref{MeasureAssumptions}, then a sample from $\mu$ is almost surely in $\mathrm{GL}_d(\R)$ and therefore $\mu$ defines a well-defined Markov chain on $X$. 

\begin{proposition}\label{UniquenessofStationaryMeasure}
    Let $\mu$ be a probability measure on $M_d(\R)$ satisfying either Assumption~\ref{MeasureAssumptions} or Assumption~\ref{SecondMeasureAssumptions}.
    Then there is a unique spherical $\mu$-stationary probability measure $\nu$ on $X$ and there exists $C,\alpha > 0$ such that for all $\ell \geq 1$, $$\sup_{x \in X}||P^{\ell}(x,\cdot) - \nu||_{\mathrm{TV}} \leq Ce^{-\alpha \ell}.$$ 
\end{proposition} 

We prove Proposition~\ref{UniquenessofStationaryMeasure} by applying a standard result from the theory of Markov chains. To state the required result, let $P$ be a Markov transition kernel on a measurable space $(X, \mathscr{B})$. For a state $x \in X$ and a measurable set $A \in \mathscr{B}$, the kernel $P(x, A)$ gives the probability of transitioning from $x$ into $A$ in one step.

The Markov chain (or kernel $P$) is said to satisfy the \textbf{Doeblin condition} if there exists a real number $\delta > 0$ and a probability measure $\theta$ on $(X, \mathscr{B})$ such that for some $\ell \geq 1$
and all measurable sets $A \in \mathscr{B}$,
\begin{equation}\label{DoeblinConditionDefinition}
    \inf_{x \in X}P^{\ell}(x, A) \ge \delta \, \theta(A).
\end{equation}

We will then apply the following result for Markov chains on abstract state spaces with a stationary distribution $\nu$:

\begin{theorem}\citep[Special case of Theorem 16.0.2]{MeynTweedieBook}\label{DoeblinTheorem}
    A Markov chain with transition probability $P$ satisfies the Doeblin condition if and only if the Markov chain has a unique stationary measure $\nu$ and there exist $C,\alpha > 0$ such that for all $\ell \geq 1$, $$\sup_{x \in X}||P^{\ell}(x,\cdot) - \nu||_{\mathrm{TV}} \leq Ce^{-\alpha \ell}.$$
\end{theorem}

In \cite{MeynTweedieBook}, it is actually shown that the equivalent conditions from Theorem~\ref{DoeblinTheorem} are equivalent to numerous further conditions, particularly that the given Markov chain is \textbf{uniformly ergodic}, that is $$\sup_{x\in X} ||P^{\ell}(x, \cdot) - \nu||_{\mathrm{TV}} \to 0 \quad\quad \text{ as }\quad\quad \ell \to \infty.$$

Returning to our proof of Proposition~\ref{UniquenessofStationaryMeasure}, our strategy is simply to apply Theorem~\ref{DoeblinTheorem}. More precisely, we verify the Doeblin condition with $\ell = 1$.

\begin{proof}(of Proposition~\ref{UniquenessofStationaryMeasure})
We treat the two cases separately.

\medskip

\noindent
\textbf{Case 1: Assumption~\ref{MeasureAssumptions}.}
Let $m_{d^2}$ denote the Lebesgue measure on $M_d(\mathbb{R})\cong \mathbb{R}^{d^2}$, and let
\[
    B_{\eps} := \{W\in M_d(\mathbb{R}) : \|W\|_F < \eps \}
\]
be the Frobenius ball of radius $\eps$, where the Frobenius norm is defined as $\|W\|_F = \sqrt{\sum_{i,j} W_{ij}^2}$. Since $\alpha_1,\alpha_2\neq 0$, one has $\phi(v)=0$ if and only if $v=0$. Since $\mu$ is absolutely continuous with respect to $m_{d^2}$, the action $W.x=\phi(Wx)\in X$ is therefore defined for $\mu$-almost every $W$.

Then for every $W\in B_{\eps}$ and every $1\le i,j\le d$ we have
$|W_{ij}|\le \|W\|_F \le \eps.$
It therefore follows from Assumption~\ref{MeasureAssumptions} that for all $W\in B_{\eps}$ and some $\gamma > 0$,
\[
    \frac{d\mu}{dm_{d^2}}(W)
    =
    \prod_{i,j=1}^d p_{ij}(W_{ij})
    \ge \gamma^{d^2}.
\]
Set
\[
    c_0 := \gamma^{d^2}.
\]
Hence for every Borel set $E\subset B_{\eps}$,
\[
    \mu(E)\ge c_0\,m_{d^2}(E).
\]

Now let $m_{\eps}$ denote the normalized Lebesgue probability measure on $B_{\eps}$, namely
\[
    m_{\eps} := \frac{1}{m_{d^2}(B_{\eps})}\,m_{d^2}\big|_{B_{\eps}}.
\]
Fix $x_*\in X$ and define a probability measure $\nu$ on $X$ by
\[
    \nu(A)
    :=
    m_{\eps}\bigl(\{W\in B_{\eps} : W.x_*\in A\}\bigr),
    \qquad
    A\in\mathcal{B}(X).
\]
We claim that $\nu$ does not depend on the choice of $x_*$. Indeed, let $x\in X$. Since $O(d)$ acts transitively on the sphere, there exists $Q\in O(d)$ such that $x=Qx_*$. Since the Frobenius norm is right $O(d)$-invariant,
that is $\|WQ\|_F=\|W\|_F$ for all  $W\in M_d(\mathbb{R}),$
and therefore, the map $W\mapsto WQ$ preserves both $B_{\eps}$ and the normalized measure $m_{\eps}$, it follows that for every Borel set $A\subset X$,
\begin{align*}
    m_{\eps}\bigl(\{W\in B_{\eps} : W.x\in A\}\bigr)
    &=
    m_{\eps}\bigl(\{W\in B_{\eps} : W.(Qx_*)\in A\}\bigr) \\
    &=
    m_{\eps}\bigl(\{WQ\in B_{\eps} : WQ.x_*\in A\}\bigr) \\ &=
    m_{\eps}\bigl(\{V\in B_{\eps} : V.x_*\in A\}\bigr) =
    \nu(A).
\end{align*}
Thus we have established that for all $x \in X$ and all Borel sets $A$ in $X$,
\[
    \nu(A) = m_{\eps}\bigl(\{W\in B_{\eps} : W.x\in A\}\bigr) = \int_{B_{\eps}} \mathbf{1}_A(W.x)\,dm_{\eps}(W).
\]

We now prove the minorization. For every $x\in X$ and every Borel set $A\subset X$,
\begin{align*}
    P(x,A)
    =
    \int_{M_d(\mathbb{R})} \mathbf{1}_A(W.x)\,d\mu(W) &\ge
    \int_{B_{\eps}} \mathbf{1}_A(W.x)\,d\mu(W) \\ &\ge
    c_0 \int_{B_{\eps}} \mathbf{1}_A(W.x)\,dm_{d^2}(W) =
    c_0\,m_{d^2}(B_{\eps})\,\nu(A).
\end{align*}
Therefore, the Doeblin condition is established with $\ell = 1$ and $\delta = c_0\,m_{d^2}(B_{\eps}) >0$.

\medskip

\noindent
\textbf{Case 2: Assumption~\ref{SecondMeasureAssumptions}.}
The proof is very similar to that of Case 1. Let $m_{\eta O(d)}$ be the normalized Haar probability measure on $\eta O(d)$. Fix $x_*\in X$ and define a probability measure $\nu$ on $X$ by
\[
    \nu(A)
    :=
    m_{\eta O(d)}\bigl(\{W\in \eta O(d) : W.x_*\in A\}\bigr),
    \qquad
    A\in\mathcal{B}(X).
\]
As in Case 1 with the same proof, $\nu$ does not depend on the choice of $x_*$. Hence it holds for all $x \in X$ and all Borel sets $A \subset X$ that
\[
    \nu(A) = m_{\eta O(d)}(\{ W \in \eta O(d) \,:\, Wx \in A \}) = \int_{\eta O(d)} \mathbf{1}_A(W.x)\,dm_{\eta O(d)}(W).
\]

Since $\mu\ll m_{\eta O(d)}$ with density $p\ge c$ on $\eta O(d)$, it follows that for every $x\in X$ and every Borel set $A\subset X$,
\begin{align*}
    P(x,A)
    =
    \int_{\eta O(d)} \mathbf{1}_A(W.x)\,p(W)\,dm_{\eta O(d)}(W) \ge
    c \int_{\eta O(d)} \mathbf{1}_A(W.x)\,dm_{\eta O(d)}(W) =
    c\,\nu(A).
\end{align*}
Thus the Doeblin condition holds for $\ell = 1$ and with $\delta:=c$. This completes the proof.
\end{proof}

\subsection{Finiteness of Moments}\label{section:FinitenessofMoments}  

In this subsection, we prove the finiteness of moments for measures satisfying Assumption~\ref{MeasureAssumptions}. As previously, denote for $W \in \mathrm{GL}_d(\R)$ by 
\begin{equation}
    N(W) =  \max\{ ||W||, ||W^{-1}|| \}.
\end{equation}

\begin{proposition}\label{FinitenessofMoment}
Let $\mu$ be a probability measure on the space of $d \times d$ real matrices $M_d(\mathbb{R})$ with independent entries. For each $1 \leq i, j \leq d$, assume the $(i, j)$-coordinate distribution is absolutely continuous with density $p_{ij}$ being bounded and having finite second moment, that is $$\sup_{t \in \R} |p_{ij}(t)| < \infty \quad\quad \text{ and } \quad\quad \int_{-\infty}^{\infty} t^2p_{ij}(t) \, dt < \infty.$$

Then all $k$-th logarithmic moments of $\mu$ are finite, that is for all $k \geq 1$,
$$
\int |\log N(W)|^k \, d\mu(W) < \infty.
$$
\end{proposition}

We first prove the following lemma.

\begin{lemma}[Smallest Singular Value Bound for a Random Matrix]
\label{lemma:singularity}
Let $\mu$ be a probability measure on the space of $d \times d$ real matrices $M_d(\mathbb{R})$ with independent entries. For each $1 \leq i, j \leq d$, assume the $(i, j)$-coordinate distribution is absolutely continuous with density $p_{ij}$ being bounded, that is $\sup_{t \in \R} |p_{ij}(t)| \leq M$ for some $M > 0$. Let $s_{\min}(W)$ be the smallest singular value of $W \in M_d(\R)$.

Then for $W\sim \mu$ and for any $\epsilon > 0$,
$$
\Prob(s_{\min}(W) \le \epsilon) \le 2Md^2 \varepsilon.$$
\end{lemma}

\begin{proof}

Write the columns of $W$ as $C_1,\dots,C_d\in\mathbb R^d$. We first claim that
\begin{equation}\label{eq:union-event}
\{s_{\min}(W)\le \varepsilon\}
\subseteq
\bigcup_{j=1}^d
\left\{\operatorname{dist}(C_j,H_j)\le \sqrt d\,\varepsilon\right\}
\qquad \text{where} \qquad 
H_j:=\operatorname{span}\{C_k:k\neq j\}.
\end{equation}
Indeed, if $s_{\min}(W)\le \varepsilon$, then as $s_{\min}(W) = \inf_{x \in \mathbb{S}^{d-1}} |Wx|$ there exists
$x=(x_1,\dots,x_d)\in \mathbb{S}^{d-1}$ such that
$|Wx|\le \varepsilon.$
Choose $j$ such that $|x_j|=\max_{1\le k\le d}|x_k|$. Since $|x|=1$, we have
$|x_j|\ge d^{-1/2}$. We have
\[
Wx=\sum_{k=1}^d x_k C_k \qquad \text{ and therefore } \qquad
x_j C_j=-\sum_{k\neq j}x_k C_k+Wx.
\]
Hence
\[
C_j=-\sum_{k\neq j}\frac{x_k}{x_j}C_k+\frac{1}{x_j}Wx.
\]
The first term on the right-hand side belongs to $H_j$, and therefore \eqref{eq:union-event} follows since
\[
\operatorname{dist}(C_j,H_j)
\le \left|\frac{1}{x_j}Wx\right|
\le \frac{\varepsilon}{|x_j|}
\le \sqrt d\,\varepsilon.
\]

We now bound each event in the union of \eqref{eq:union-event}. Fix $j\in\{1,\dots,d\}$, and let
\[
\mathcal F_j:=\sigma(C_k:k\neq j),
\]
the $\sigma$-algebra generated by all columns except $C_j$. Then $H_j$ is
$\mathcal F_j$-measurable. Since $H_j$ is spanned by at most $d-1$ vectors, it is
a proper subspace of $\mathbb R^d$. Therefore, one may choose an
$\mathcal F_j$-measurable random unit vector $n_j=(n_{j,1},\dots,n_{j,d})\in S^{d-1}$
such that $n_j\perp H_j.$
For instance, one may take $n_j$ to be the first vector in a Gram--Schmidt
orthonormal basis construction applied measurably to a fixed basis of
$\mathbb R^d$ modulo $H_j$.

If $\operatorname{dist}(C_j,H_j)\le \sqrt d\,\varepsilon$, then, since $n_j$ is a
unit vector orthogonal to $H_j$,
\[
|\langle n_j,C_j\rangle|
\le \operatorname{dist}(C_j,H_j)
\le \sqrt d\,\varepsilon.
\]
Hence
\begin{equation}\label{eq:distance-to-inner-product}
\{\operatorname{dist}(C_j,H_j)\le \sqrt d\,\varepsilon\}
\subseteq
\{|\langle n_j,C_j\rangle|\le \sqrt d\,\varepsilon\}.
\end{equation}

Next, choose an index $m_j\in\{1,\dots,d\}$ such that
$|n_{j,m_j}|=\max_{1\le i\le d}|n_{j,i}|.$
Since $|n_j|=1$, we again have
$|n_{j,m_j}|\ge d^{-1/2}.$ As $n_j$ is $\mathcal F_j$-measurable, so is $m_j$. Now, let
\[
\mathcal G_j:=\sigma\bigl(\mathcal F_j,\; W_{ij}: i\neq m_j\bigr).
\]
Conditionally on $\mathcal G_j$, the quantity $\langle n_j,C_j\rangle$ is an
affine function of the single remaining random variable $W_{m_j j}$:
\[
\langle n_j,C_j\rangle
=
n_{j,m_j}W_{m_j j}+b_j
\]
for some $\mathcal G_j$-measurable real random variable $b_j$. We remark that since $m_j$ takes only finitely many values and is $F_j$-measurable, this conditioning argument can be justified by restricting to the events $\{m_j=m\}, m=1,\dots,d$. Therefore,
conditionally on $\mathcal G_j$, the event
\[
\{|\langle n_j,C_j\rangle|\le \sqrt d\,\varepsilon\}
\]
forces $W_{m_j j}$ to lie in an interval of length
\[
\frac{2\sqrt d\,\varepsilon}{|n_{j,m_j}|}
\le 2d\,\varepsilon.
\]
Since $W_{m_j j}$ is independent of $\mathcal G_j$ and has density bounded by $M$,
we obtain
\[
\mathbb P\bigl(|\langle n_j,C_j\rangle|\le \sqrt d\,\varepsilon \,\big|\, \mathcal G_j\bigr)
\le 2Md\,\varepsilon
\qquad\text{a.s.}
\]
Using \eqref{eq:distance-to-inner-product} and taking expectations to use that $\mathbb{E}[Y] = \mathbb{E}[\mathbb{E}[Y|\mathcal{G}_j]]$ for any random variable $Y$, we conclude
\[
\mathbb P\bigl(\operatorname{dist}(C_j,H_j)\le \sqrt d\,\varepsilon\bigr)
\le
\mathbb P\bigl(|\langle n_j,C_j\rangle|\le \sqrt d\,\varepsilon\bigr)
\le 2Md\,\varepsilon.
\]

Finally, by \eqref{eq:union-event} and the union bound,
\[
\mathbb P\bigl(s_{\min}(W)\le \varepsilon\bigr)
\le
\sum_{j=1}^d
\mathbb P\bigl(\operatorname{dist}(C_j,H_j)\le \sqrt d\,\varepsilon\bigr)
\le 2Md^2\,\varepsilon.
\]
This proves the lemma.
\end{proof}

Using the lemma, we can deduce Proposition~\ref{FinitenessofMoment}.

\begin{proof}(of Proposition~\ref{FinitenessofMoment})
Denote  $Y = \log N(W)$. By the properties of operator norms, $1 = ||I|| = ||W W^{-1}|| \le ||W|| \cdot ||W^{-1}||$. This implies that $N(W)  \ge 1$. Therefore, $\log N(W) \ge \log(1) = 0$.

Since $Y$ is therefore a non-negative random variable, $\E[|Y|^k] = \E[Y^k]$, and we can compute the $k$-th moment using the tail probability formula, that is
$$
\E[Y^k] = \int_0^\infty k t^{k-1} \Prob(Y > t) \, dt
$$
This integral is finite if the tail probability $\Prob(Y > t)$ decays sufficiently fast. We will show that it decays exponentially.

Note that $\Prob(Y > t) =  \Prob(N(W) > e^t)$ and so we can split the probability
\begin{align*}
\Prob(Y > t) &= \Prob(\max(||W||, ||W^{-1}||) > e^t) \\
&\le \Prob(||W|| > e^t) + \Prob(||W^{-1}|| > e^t)
\end{align*}
We will bound the integral of these two terms separately.

First, we bound $\Prob(||W|| > e^t)$. We use the Frobenius norm $||W||_F^2 = \sum_{i,j=1}^d W_{ij}^2$. Since all norms on $\R^{d \times d}$ are equivalent, bounding the Frobenius norm is sufficient.
Since the $W_{ij}$ have a finite second moment,  $\E[W_{ij}^2] \leq \sigma^2 < \infty$ for some $\sigma > 0$ and all $1 \leq i,j \leq d$.
Then, by linearity of expectation:
$$ \E[||W||_F^2] = \E\left[\sum_{i,j} W_{ij}^2\right] = \sum_{i,j} \E[W_{ij}^2] \leq d^2 \sigma^2 $$
Using Markov's inequality for some $T > 0$ and a constant $C_1 > 0$,
$$ \Prob(||W|| > T) \le \Prob(||W||_F > C_1T) = \Prob(||W||_F^2 > C_1^2T^2) \le \frac{\E[||W||_F^2]}{C_1^2 T^2} = \frac{d^2 \sigma^2}{C_1^2T^2} $$
Thus, there exists a constant $C_1' = C_1^{-2}d^2 \sigma^2$ such that $\Prob(||W|| > T) \le C_1' T^{-2}$. Thus we conclude that $\Prob(||W|| > e^t) \le C_1' e^{-2t}$ and therefore $$\int_0^\infty k t^{k-1} \Prob(||W|| > e^t) \, dt \leq C_1' \int_0^\infty k t^{k-1} e^{-2t} \, dt < \infty.$$

Finally, we bound $\Prob(||W^{-1}|| > e^t)$ for which we use Lemma~\ref{lemma:singularity}. Let $s_{\min}(W)$ be the smallest singular value of $W$. Then $||W^{-1}|| = 1/s_{\min}(W)$ (for the operator norm). Thus, by Lemma~\ref{lemma:singularity},
$$ \Prob(||W^{-1}|| > e^{t}) = \Prob(s_{\min}(W) < e^{-t}) \leq C_2 e^{-t} $$ for some constant $C_2 > 0$. So we conclude that $$\int_0^\infty k t^{k-1} \Prob(||W^{-1}|| > e^t) \, dt \leq C_2 \int_0^\infty k t^{k-1} e^{-t} \, dt < \infty,$$ concluding the proof.
\end{proof}

\subsection{Central Limit Theorem}\label{section:ProofExplicitCLT}

We want to establish the Central Limit Theorem \ref{ExplicitLeakyReLUCLT} by using the general result stated in Theorem~\ref{CLTCocycleGeneral}. However, Theorem~\ref{CLTCocycleGeneral} requires some further assumptions on the cocycle, which will require some argument to establish.

\begin{theorem}(General Central Limit Theorem for $|X_{\ell}|$)\label{GeneralCLT}
    Let $\mu$ be a Borel probability measure on $\mathrm{GL}_d(\R)$.  Let $(X_{\ell})_{\ell\geq 0}$ be as in \eqref{ProbNNDef} and let $\phi(x) = \max(\alpha_1 x,\alpha_2 x)$ be a generalized Leaky ReLU activation function with $\alpha_1, \alpha_2 \in \R_{\neq 0}$. Assume that the resulting Markov chain on $X$ satisfies the Doeblin condition from \eqref{DoeblinConditionDefinition}. Moreover, assume that $\mu$ has finite second logarithmic moment, that is $\int \log N(W)^2 \, d\mu(W) < \infty$. 

     Then there exists $\gamma_{\mu,\phi} \geq 0$ such that for any $x_0 \in \R^d\backslash \{ 0 \}$ it holds that as $\ell \to \infty$, $$\frac{\log |X_{\ell}| - \ell \lambda_{\mu, \phi}}{\sqrt{\ell}} \longrightarrow \mathcal{N}(0,\gamma_{\mu,\phi}^2),$$ where the convergence holds in distribution and $\lambda_{\mu, \phi} \in \R$ is from Theorem~\ref{UniqueStatLeakyReLULLN}.
\end{theorem}

Before proving Theorem~\ref{GeneralCLT}, we show how Theorem~\ref{GeneralCLT} implies Theorem~\ref{ExplicitLeakyReLUCLT}.

\begin{proof}(of Theorem~\ref{ExplicitLeakyReLUCLT})
    We just need to check that the assumptions of Theorem~\ref{GeneralCLT} are satisfied. When $\mu$ satisfies Assumption~\ref{MeasureAssumptions} or Assumption~\ref{SecondMeasureAssumptions}, the Doeblin condition \eqref{DoeblinConditionDefinition} holds as was proved in Proposition~\ref{UniquenessofStationaryMeasure} and therefore by Theorem~\ref{DoeblinTheorem}, $\mu$ has a unique spherical $\mu$-stationary probability measure. By Proposition~\ref{FinitenessofMoment}, if $\mu$ satisfies Assumption~\ref{MeasureAssumptions}, it has finite first and second logarithmic moment. If $\mu$ satisfies Assumption~\ref{SecondMeasureAssumptions}, then it clearly has finite first and second moments since, for $W \sim \mu$, we have $||W|| = \eta$ and $||W^{-1}|| = \eta^{-1}$. Therefore, the Law of Large Numbers follows by Theorem~\ref{UniqueStatLeakyReLULLN} and the result follows by Theorem~\ref{GeneralCLT}.
\end{proof}

Consider the previously studied cocycle $c_{\mathrm{hom}}: G \times X \to \mathbb{R}$, defined as $c_{\mathrm{hom}}(f, \overline{x}) = \log \frac{|f(x)|}{|x|}$ for $f \in \mathscr{C}_{\mathrm{hom}}^{\times}(\R^d)$ and $\overline{x} \in X$. Before proving Theorem~\ref{GeneralCLT}, we first show the following lemma that uses the Doeblin condition and for which we use the following notation: For $f:X \to \R$ continuous we denote by $P_{\mu}$ the operator given for $\overline{x} \in X$ by $$P_\mu f(\overline{x}) = \int_{G} f(\phi_W(\overline{x})) \, d\mu(W).$$

\begin{lemma}\label{LemmaGeneralCLT}
    Let $\mu$ and $X_{\ell}$ be as in Theorem~\ref{GeneralCLT} and let $\lambda_{\mu,\phi}$ be from Theorem~\ref{UniqueStatLeakyReLULLN}. Then there is a continuous function $\psi : X \to \R$ such that for all $\overline{x} \in X$
\begin{equation}\label{CohomEq}
    \psi(\overline{x}) - P_\mu \psi(\overline{x}) = \overline{c}_{\mathrm{hom}}(\overline{x}) - \lambda_{\mu,\phi}, \quad\quad \text{where} \quad\quad   \overline{c}_{\mathrm{hom}}(\overline{x}) = \int_{G} c_{\mathrm{hom}}(\phi_W, \overline{x}) \, d\mu(W).\\
\end{equation}
\end{lemma}

\begin{proof}(of Lemma~\ref{LemmaGeneralCLT}) Note that by Cauchy-Schwarz, $\mu$ has finite first logarithmic moment too and so Theorem~\ref{UniqueStatLeakyReLULLN} can be applied. Let $\lambda_{\mu,\phi}$ be the Lyapunov exponent from Theorem~\ref{UniqueStatLeakyReLULLN}. For convenience, write for $\overline{x} \in X$
\begin{equation*}
    \varphi_0(\overline{x}) := \overline{c}_{\mathrm{hom}}(\overline{x}) - \lambda_{\mu,\phi}.
\end{equation*}
Note that $\overline{c}_{\mathrm{hom}}$ is continuous by dominated convergence using Lemma~\ref{lemma:finmom} and therefore so is $\varphi_0$. Moreover, recall that by equation~\ref{UniqueStatLyapFormula} it holds that $\lambda_{\mu,\phi} = \int_X \overline{c}_{\mathrm{hom}}(\overline{x}) \, d\nu(\overline{x})$, and therefore the function $\varphi_0$ has zero mean with respect to the unique stationary measure $\nu$, that is
    \begin{equation*}
        \int_X \varphi_0(\overline{x}) \, d\nu(\overline{x}) = \int_X (\overline{c}_{\mathrm{hom}}(\overline{x}) - \lambda_{\mu,\phi}) \, d\nu(\overline{x}) = 0.
    \end{equation*}

We now exploit the assumption that $\mu$ satisfies the Doeblin condition and apply Theorem~\ref{DoeblinTheorem}. Using the notation $P^{\ell}$ from Theorem~\ref{DoeblinTheorem}, we note that $P_{\mu}^{\ell}f(x) = P^\ell(x,f)$ for all $\ell \geq 1$, $f: X \to \R$ and $x \in X$. Thus, it follows from Theorem~\ref{DoeblinTheorem} using that $\nu(\varphi_0) = \int \varphi_0 \, \nu = 0$ that 
\begin{equation}
    \| P_\mu^\ell \varphi_0 \|_\infty \le \|\varphi_0 \|_\infty\sup_{x \in X}||P^{\ell}(x,\cdot) - \nu||_{\mathrm{TV}} \leq C||\varphi_0 ||_{\infty} e^{-\alpha \ell}.
\end{equation}

We can now invert the operator $(I - P_\mu)$ to define $\psi$ for $\overline{x} \in X$ as
\begin{equation}
    \psi(\overline{x}) = \sum_{k=0}^{\infty} P_\mu^k \varphi_0(\overline{x}).
\end{equation}
Due to the estimate $\| P_\mu^k \varphi_0 \|_\infty \le C||\varphi_0 ||_{\infty} e^{-\alpha k}$, this series converges uniformly on $X$. Therefore, $\psi$ is a well-defined continuous function.

We finally verify that this constructed $\psi$ solves \eqref{CohomEq} by applying the operator $(I - P_\mu)$:
\begin{align*}
    (I - P_\mu)\psi(x) &= \psi(x) - P_\mu \psi(x) \\
    &= \sum_{k=0}^{\infty} P_\mu^k \varphi_0(x) - \sum_{k=0}^{\infty} P_\mu^{k+1} \varphi_0(x) \\
    &= \varphi_0(x) + \sum_{k=1}^{\infty} P_\mu^k \varphi_0(x) - \sum_{j=1}^{\infty} P_\mu^j \varphi_0(x) \quad (\text{Telescoping sum}) \\
    &= \varphi_0(x) = \overline{c}_{\mathrm{hom}}(x) - \lambda_{\mu,\phi}.
\end{align*}
This concludes the proof of the lemma.
\end{proof}

\begin{proof}(of Theorem~\ref{GeneralCLT})
The moment condition of Theorem~\ref{CLTCocycleGeneral} follows from Lemma~\ref{lemma:finmom}. The Doeblin condition implies uniqueness of stationary measures. So by Theorem~\ref{CLTCocycleGeneral} applied to the measure $\eta = (W\mapsto \phi_W)_{*}\mu$, the proof reduces to showing that $c_{\mathrm{hom}}$ is centerable for which we use Lemma~\ref{LemmaGeneralCLT}. According to the definition of centerability, we must show that there is a continuous function $\psi: X \to \R$ and a cocycle $c_0$ with constant drift and average $\lambda_{\mu,\phi}$ such that for all $f\in \mathscr{C}_{\mathrm{hom}}^{\times}(\mathbb{R}^d)$ and $\overline{x} \in X$,
\begin{equation}\label{CenterableGoal}
    c_{\mathrm{hom}}(f, \overline{x}) = c_0(f, \overline{x}) + \psi(\overline{x}) - \psi(f(\overline{x})).
\end{equation}

Let $\psi$ be the function from Lemma~\ref{LemmaGeneralCLT} satisfying \eqref{CohomEq}. Then we set for $f\in \mathscr{C}_{\mathrm{hom}}^{\times}(\mathbb{R}^d)$ and $\overline{x} \in X$, $$c_0(f, \overline{x}) = c_{\mathrm{hom}}(f, \overline{x}) - \psi(\overline{x}) + \psi(f(\overline{x})).$$ It is clear that $c_0$ is continuous and it is a cocycle since for any $f,g\in \mathscr{C}_{\mathrm{hom}}^{\times}(\mathbb{R}^d)$ and $\overline{x} \in X$ we have that
\begin{align*}
    c_0(f\circ g,\overline{x}) &= c_{\mathrm{hom}}(f\circ g, \overline{x}) - \psi(\overline{x})  + \psi((f\circ g)(\overline{x}))\\
    &= c_{\mathrm{hom}}(f\circ g, \overline{x}) - \psi(\overline{x}) + \psi(g(\overline{x})) - \psi(g(\overline{x})) + \psi((f\circ g)(\overline{x})) \\
    &=  c_{\mathrm{hom}}(f, g(\overline{x})) + c_{\mathrm{hom}}(g,\overline{x}) - \psi(\overline{x}) + \psi(g(\overline{x}))- \psi(g(\overline{x})) + \psi((f\circ g)(\overline{x}))\\
    &= (c_{\mathrm{hom}}(f, g(\overline{x})) - \psi(g(\overline{x})) + \psi((f\circ g)(\overline{x}))) + (c_{\mathrm{hom}}(g,\overline{x}) - \psi(\overline{x}) + \psi(g(\overline{x})) \\
    &= c_0(f, g(\overline{x})) + c_0(g,\overline{x}),
\end{align*} having added $0 = \psi(g(\overline{x})) - \psi(g(\overline{x}))$ in the second line and used the cocycle property in the third.

Finally, $c_0$ has constant drift with average $\lambda_{\mu,\phi}$ as for any $\overline{x} \in X$ we have by using \eqref{CohomEq} that
\begin{align*}
    \int c_0(\phi_W,\overline{x}) \, d\mu(W) &= \int c_{\mathrm{hom}}(\phi_W, \overline{x}) - \psi(\overline{x}) + \psi(\phi(W\overline{x})) \, d\mu(W) \\
    &= \overline{c}_{\mathrm{hom}}(\overline{x}) - \psi(\overline{x}) + P_{\mu}\psi(\overline{x}) = \lambda_{\mu,\phi}.
\end{align*}
This concludes the proof that $c_0$ as defined above is an additive cocycle that has constant drift with average $\lambda_{\mu,\phi}$. Thus, $c_{\mathrm{hom}}$ is centerable and the proof is concluded.
\end{proof}

\subsection{Counterexamples to Law of Large Numbers}\label{section:CounterexamplesLLN}

We end this section by establishing the following lemma stating that the Law of Large Numbers is a subtle result and fails under certain conditions.

\begin{lemma}(Counterexample to Law of Large Numbers for $|X_{\ell}|$)\label{Lemma:LLNCounterexample}
    Let $\mu$ be a probability measure on $M_d(\R)$. Let $(\Omega, \mathscr{F}, \mathbb{P}_{\mu})$ be as in \eqref{ProbSpace}, $(X_{\ell})_{\ell\geq 0}$ be as in \eqref{ProbNNDef}. Then the conclusion of Theorem~\ref{ExplicitLeakyReLULLN} is false under either of the following assumptions:
    \begin{enumerate}
        \item[(1)] (ReLU activation function) When $\phi(x) = \max(x,0)$ and the entries of $\mu$ are independent and all distributed as $\mathcal{N}(0,\sigma^2)$ for $\sigma > 0$.
        \item[(2)] (Measure supported on matrices with positive entries) When $\phi(x) = \max( x,\alpha x)$ for $\alpha \in (0,1)$ and the entries of $\mu$ are independent and all distributed as $\mathrm{Unif}[0,a]$ for some $a > 0$.
    \end{enumerate}
\end{lemma}

We briefly give an intuitive explanation why these results fail. In the ReLU case, with a positive probability $X_{\ell}$ is zero and so the Lyapunov exponent could only be $-\infty$, which contradicts the statement of Theorem~\ref{ExplicitLeakyReLULLN} that $\lambda_{\mu,\phi}$ is a finite real number. In the Leaky ReLU case and when $\mu$ consists of matrices with positive entries, the positive cone $\R^d_{>0}$ and negative cone $\R^d_{<0}$ are preserved, leading to different Lyapunov exponents depending on the starting point. Moreover, as we show in the proof of Lemma~\ref{Lemma:LLNCounterexample}, even for certain fixed starting points, the Lyapunov exponent does not converge almost surely to the same number.

\begin{proof}
    For (1), we simply note that for every $\ell$, with a positive probability $X_{\ell} = 0 \in \R^d$ and so $\log |X_{\ell}| = - \infty$ with positive probability. So the Lyapunov exponent is not $\in \R$, yet it holds that almost surely $\log |X_{\ell}| = - \infty$ for sufficiently large $\ell$.

    For (2), we first note that if we write $Y_{\ell} = W_{\ell} \dots W_1 x_0$ for the standard linear matrix product, then, analogously to Proposition~\ref{UniquenessofStationaryMeasure}, there is a unique $\mu$-stationary on $\mathbb{PR}^d$, it follows by Proposition~\ref{StandardLLNUniqueStat} that there is $\lambda_{\mathrm{mat}} \in \R$ such that almost surely $$\frac{1}{\ell} \log |Y_{\ell}| \overset{\ell\to \infty}{\longrightarrow} \lambda_{\mathrm{mat}}.$$ 
    
    Moreover, as the matrices in $\mu$ have almost surely positive entries, they map $\mathbb{R}^d_{> 0}$ to itself as well as $\mathbb{R}^d_{< 0}$ to itself. So if $x_0 \in \mathbb{R}^d_{> 0}$, then $X_{\ell} = Y_{\ell} \in \mathbb{R}^d_{> 0}$ for all $\ell \geq 1$  and hence almost surely 
    \begin{equation}\label{Limit1}
        \frac{1}{\ell}\log |X_\ell| \overset{\ell\to \infty}{\longrightarrow} \lambda_{\mathrm{mat}}.
    \end{equation}

    On the other hand, if $x_{0} \in \R^d_{< 0}$, then with probability one $\mu$ maps $\R^d_{<0}$ to $\R^d_{<0}$ and so it follows that $X_{\ell} \in R^d_{< 0}$ for all $\ell \geq 1$ and $$X_{\ell} = \alpha^{\ell}(W_{\ell}W_{\ell-1}\cdots W_1)x_0 = \alpha^{\ell}Y_{\ell}.$$ Thus it follows that \begin{equation}\label{Limit2}
        \frac{1}{\ell}\log |X_\ell| \overset{\ell\to \infty}{\longrightarrow} \lambda_{\mathrm{mat}} + \log \alpha.
    \end{equation}

    As $\alpha \in (0,1)$ the limits \eqref{Limit1} and \eqref{Limit2} are distinct. This contradicts the conclusion of Theorem~\ref{ExplicitLeakyReLULLN}, which requires the limit $\lambda_{\mu, \phi}$ to be independent of the starting point $x_0 \in \R^d_{\neq 0}$.

    Moreover, the Law of Large Numbers fails more strongly for points not in $\R^d_{> 0}$ or $\R^d_{< 0}$. Indeed, if for example $x_0 = (-1,1,1,\ldots) \in \R^d$, then with positive probability $X_{1}$ is in $\R^{d}_{>0}$ and also with positive probability it is in $\R^{d}_{<0}$. Therefore, it follows that the events  $$A_{1} = \left\{ \omega \in \Omega \,:\, \lim_{\ell \to \infty} \frac{1}{\ell}\log |X_{\ell}(\omega)| = \lambda_{\mathrm{mat}}  \right \}$$
    
    $$A_{2} = \left\{ \omega \in \Omega \,:\, \lim_{\ell \to \infty} \frac{1}{\ell}\log |X_{\ell}(\omega)| =   \lambda_{\mathrm{mat}}  + \log \alpha \right \}$$ have positive probability, i.e. $\mathbb{P}_{\mu}(A_{i}) > 0$ for $i = 1,2$. Thus, for such an $x_0$ the sequence $\frac{1}{\ell}\log |X_{\ell}(\omega)|$ does not converge to the same limit almost surely.
\end{proof}

\section{Lyapunov Exponent Calculations: Proof of Theorem~\ref{LyapunovGaussianFormula}, Theorem~\ref{LyapunovOrthogonalFormula} and Theorem~\ref{AsymptoticExpansion}}\label{section:AppendixFormulas}

In Section~\ref{section:LyapunovFormulas}, Theorem~\ref{LyapunovGaussianFormula} and Theorem~\ref{LyapunovOrthogonalFormula} are proved, while in Section~\ref{section:AsymptoticExpansion} we prove Theorem~\ref{AsymptoticExpansion}. We provide extensive lookup tables for the studied parameters in Section~\ref{section:lookuptables}.

\subsection{Integral Formulas for Lyapunov Exponents: Proof of Theorem~\ref{LyapunovGaussianFormula} and Theorem~\ref{LyapunovOrthogonalFormula}}\label{section:LyapunovFormulas}

In this section, we prove Theorem~\ref{LyapunovGaussianFormula} and Theorem~\ref{LyapunovOrthogonalFormula}, which we prove for generalized Leaky ReLU activation functions, as stated next.  

\begin{theorem}\label{GeneralLyapunovGaussianFormula}
    Let $\mu$ be the probability measure on $M_d(\R)$  with all coefficients being independent and distributed as $\mathcal{N}(0,\sigma^2)$ for some $\sigma > 0$. Let $\phi = \max(\alpha_1 x,\alpha_2 x)$ be a generalized Leaky ReLU activation function with $\alpha_1,\alpha_2 \in \R_{\neq 0}$. Then the Lyapunov exponent $\lambda_{\mu,\phi}$ from Theorem~\ref{UniqueStatLeakyReLULLN} satisfies $$\lambda_{\mu,\phi} = \log(\sigma) + I(d,\alpha_1,\alpha_2),$$
    where $I(d,\alpha_1,\alpha_2)$ is the integral given by
    \begin{equation}\label{GeneralInteralalphad}
        I(d,\alpha_1,\alpha_2) = \int_0^\infty \frac{e^{-t} - \frac{1}{2^d} \left(\frac{1}{\sqrt{1+2\alpha_1^2t}} + \frac{1}{\sqrt{1+2\alpha_2^2 t}}  \right)^d }{2t}  dt.
    \end{equation}
\end{theorem}

\begin{theorem}\label{GeneralLyapunovOrthogonalFormula}
    For $\eta > 0$ let $m_{\eta\cdot \mathrm{O}(d)}$ be the volume measure on $\eta\cdot \mathrm{O}(d)$ as defined in \eqref{HaaretaOdef}. Let $\phi = \max(\alpha_1 x,\alpha_2 x)$ be a generalized Leaky ReLU activation function with $\alpha_1,\alpha_2 \in \R_{\neq 0}$. Then the Lyapunov exponent $\lambda_{m_{\eta \cdot \mathrm{O}(d)},\phi}$ from Theorem~\ref{UniqueStatLeakyReLULLN} satisfies for $I(d,\alpha_1,\alpha_2)$ from \eqref{GeneralInteralalphad} that
    \begin{align}
        \lambda_{m_{\eta\cdot\mathrm{O}(d)},\phi} &= \log(\eta) + I(d,\alpha_1,\alpha_2)  - I(d,1,1).
    \end{align}
\end{theorem}

Theorem~\ref{LyapunovGaussianFormula} and Theorem~\ref{LyapunovOrthogonalFormula} directly follow from Theorem~\ref{GeneralLyapunovGaussianFormula} and Theorem~\ref{GeneralLyapunovOrthogonalFormula} by specializing to $\alpha_1 = 1$ and $\alpha_2 = \alpha$.

As in the previous section, denote by $$X = (\R^d\backslash\{0\})/\R_{+}$$ and recall that $X$ is diffeomorphic to the sphere $\mathbb{S}^{d-1} = \{x \in \R^d \,:\, |x| = 1 \}$. There is the natural volume probability measure on the sphere, and we denote by $m_{\mathbb{S}^d}$ the resulting probability measure on $X$.

The main reason that the Lyapunov exponent can be computed in these cases is Lemma~\ref{SphericalStatLemma}, showing that the distributions in question are invariant under orthogonal transformations. This then leads to the following conclusion, which is a central ingredient in our proof of Theorem~\ref{GeneralLyapunovGaussianFormula} and Theorem~\ref{GeneralLyapunovOrthogonalFormula}.

\begin{proposition}\label{SphericalStatMeasure}
    Let $\phi = \max(\alpha_1 x, \alpha_2 x)$ for $\alpha_1, \alpha_2 \in \R_{\neq 0}$ and let $\mu$ be a measure on $M_d(\R)$ of one of the following two types:
    \begin{enumerate}
        \item[(1)] The coordinate distributions of $\mu$ are independent and all distributed as $\mathcal{N}(0,\sigma^2)$ for some $\sigma > 0$.
        \item[(2)] The Haar probability measure $ m_{\eta \cdot \mathrm{O}(d)}$ on $\eta \cdot \mathrm{O}(d)$ for some $\eta > 0$.
    \end{enumerate}
    Then for any $x\in \R^d$ with $|x| = 1$ and with $W\sim \mu$ it holds that $$\lambda_{\mu,\phi} = \mathbb{E}[\log |\phi(Wx)|].$$
\end{proposition}

We first prove the following key preliminary lemma.

\begin{lemma}\label{SphericalStatLemma}
    Let $\mu$ be a measure as in Proposition~\ref{SphericalStatMeasure}. Then if $W \sim \mu$ and $Q \in \mathrm{O}(d)$ then it holds that $WQ$ and $W$ have the same distribution.
\end{lemma}

\begin{proof}
    For (2) , the conclusion holds since the Haar probability measure on $m_{\eta \cdot \mathrm{O}(d)}$ is also right invariant as shown in \citep[\S 30]{LoomisBook}.
    So assume that (1) holds. Then the joint probability density function of $W$ is the product of the independent marginal densities of its entries and so, for $||W||_F$ the Frobenius norm, it holds that
\begin{align*}
p(W) &= \prod_{i,j=1}^d \frac{1}{\sqrt{2\pi\sigma^2}} \exp\left( -\frac{W_{ij}^2}{2\sigma^2} \right) \\ &= \frac{1}{(2\pi\sigma^2)^{d^2/2}}\exp\left( -\frac{1}{2\sigma^2} \sum_{i,j} W_{ij}^2 \right) 
= \frac{1}{(2\pi\sigma^2)^{d^2/2}}\exp\left( -\frac{1}{2\sigma^2} \|W\|_F^2 \right).
\end{align*}
Consider the linear transformation $W \mapsto WQ$. The Frobenius norm is invariant under right-multiplication by an orthogonal matrix, as shown by the cyclic property of the trace:
\[
\|WQ\|_F^2 = \mathrm{tr}((WQ)^T WQ) =  \mathrm{tr}(Q^T W^T W Q) =  \mathrm{tr}(W^T W Q Q^T) = \|W\|_F^2.
\]
Furthermore, the Jacobian determinant of this transformation is 1, since $|\det(I \otimes Q)| = |\det(Q)|^d = 1$.
Combining the invariance of the norm with this Jacobian property, we conclude that the density is preserved under orthogonal transformations.
\end{proof}

\begin{proof}(of Proposition~\ref{SphericalStatMeasure})
    By Proposition~\ref{UniquenessofStationaryMeasure}, there exists a unique spherical $\mu$-stationary probability measure $\nu$ on $X$. So we can apply Theorem~\ref{UniqueStatLeakyReLULLN} and therefore by equation \eqref{UniqueStatLyapFormula} it holds that
\[
\lambda_{\mu,\phi}
=
\int_X \left( \int \log |\phi(Wx)|\, d\mu(W) \right) d\nu(x).
\]
Define
\[
F(x):=\int \log |\phi(Wx)|\, d\mu(W), \qquad \text{ for } \qquad x\in \mathbb{S}^{d-1}.
\]
We claim that it follows by Lemma~\ref{SphericalStatLemma} that $F$ is constant on $\mathbb{S}^{d-1}$. Indeed, let $x,y\in \mathbb{S}^{d-1}$ and choose $Q\in O(d)$ such that $Qx=y$. Then, by Lemma~\ref{SphericalStatLemma}, $WQ \stackrel{d}{=} W$ and hence
\begin{align*}
F(y)
=
\int \log |\phi(Wy)|\, d\mu(W)
&=
\int \log |\phi(WQx)|\, d\mu(W)
\\ &=
\int \log |\phi((WQ)x)|\, d\mu(W)
=
\int \log |\phi(Wx)|\, d\mu(W)
=
F(x).
\end{align*}
Thus, $F$ is constant on $\mathbb{S}^{d-1}$ and therefore,
\[
\lambda_{\mu,\phi}
=
\int_X F(x)\, d\nu(x)
=
F(x)
=
\int \log |\phi(Wx)|\, d\mu(W)
= \mathbb{E}[\log |\phi(Wx)|]  \] 
for every $x\in \mathbb{S}^{d-1}$, as claimed.
\end{proof}

We first deal with the Gaussian case in the following lemma.

\begin{lemma}\label{lem:lyapunov-exponent-analytic} Let $\mu$ be the measure on $M_d(\R)$ such that the coordinate distributions of $\mu$ are independent and all distributed as $\mathcal{N}(0,\sigma^2)$ for some $\sigma > 0$. Let $\phi = \max(\alpha_1 x, \alpha_2 x)$ for $\alpha_1, \alpha_2 \in \R_{\neq 0}$. Then for the standard basis vector $e_1 = (1, 0,0,\ldots ,0) \in \R^d$ and $W\sim \mu$ it holds that
    $$ \E[\log|{\phi(We_1)}|] =\log(\sigma) + \frac{1}{2}  \int_0^\infty \frac{e^{-t} - \frac{1}{2^d} \left(\frac{1}{\sqrt{1+2\alpha_1^2 t}} + \frac{1}{\sqrt{1+2\alpha_2^2 t}}  \right)^d }{t}  dt.$$
\end{lemma}
Before we prove Lemma~\ref{lem:lyapunov-exponent-analytic}, we derive the following moment generating function.
\begin{lemma} \label{lem:mgf}
For $Z\sim \mathcal{N}(0,1)\in \R$, the moment generating function of $\phi^2(Z)$ is for $t\in \R$ with $t< \min\{\frac{1}{2\alpha_1^2}, \frac{1}{2\alpha_2^2 }\}$ equal to
$$ M_{\phi^2(Z)}(t):= \E[\exp(t\phi^2(Z))] = \frac{1}{2} \left( \frac{1}{\sqrt{1-2\alpha_1^2 t}} + \frac{1}{\sqrt{1-2\alpha_2^2 t}} \right).$$
\end{lemma}
\begin{proof}
   Since both \(\phi(x)=\max(\alpha_1x,\alpha_2x)\) and the claimed formula are symmetric in \(\alpha_1,\alpha_2\), we may assume without loss of generality that \(\alpha_1\le \alpha_2\). We note 
   \begin{align*}
       \E[\mathrm{exp}(t \phi^2 \left(Z \right))] 
        &= \E[\exp(t\alpha_1^2 Z^2)\mathbf{1}_{Z<0}] + \E[\exp(t\alpha_2^2 Z^2) \mathbf{1}_{Z \geq 0}] \\
        &=  \E[\exp(t \alpha_1^2  Z^2)] \E[\mathbf{1}_{Z<0}]
        +\E[\exp(t \alpha_2^2 Z^2)] \E[\mathbf{1}_{Z\geq 0}] \\
        &= \frac{1}{2}(\E[\exp(t \alpha_1^2  Z^2)] + \E[\exp(t \alpha_2^2 Z^2)] ),
   \end{align*}
   where in the second line we have used that the distribution of $Z^2$ conditional on $Z<0$ is identical to the distribution of $Z^2$ conditional on $Z\geq 0$ and they are both distributed as $Z^2$ without conditioning. The distribution of $Z^2$ is called the $\chi_1^2$-distribution and it is well-known that the moment generating function of $\chi_1^2$ satisfies $$M_{\chi_1^2}(u) = \mathbb{E}[\exp(u\chi_1^2)] = \frac{1}{\sqrt{1 - 2u}} \quad\quad \text{ for } \quad\quad u \in \R \,\text{ with }\, u < \frac{1}{2}.$$ Thus it holds that 
   \begin{align*}
       \E[\mathrm{exp}(t \phi^2 \left(Z \right))]  &= \frac{1}{2}(M_{\chi_1^2}(t\alpha_1^2) + M_{\chi_1^2}(t\alpha_2^2) ) \\
       &= \frac{1}{2} \left( \frac{1}{\sqrt{1-2\alpha_1^2 t}} + \frac{1}{\sqrt{1-2\alpha_2^2 t}} \right)
   \end{align*} for $t< \min\{\frac{1}{2\alpha_1^2}, \frac{1}{2\alpha_2^2 }\}$.
\end{proof}

We also prove the following identity. 

\begin{lemma}(Frullani's identity)\label{Frullani}
    For $x > 0$, it holds that $$\log(x) = \int_0^\infty \frac{e^{-t} - e^{-xt}}{t} dt.$$
\end{lemma}

\begin{proof}
For convenience, write $I(x) = \int_0^\infty \frac{e^{-t} - e^{-xt}}{t} dt$. Differentiating with respect to $x$ (upon applying dominated convergence) we obtain
\begin{align*}
    I'(x) &= \int_0^\infty \frac{d}{d x} \left(\frac{e^{-t} - e^{-xt}}{t} \right) dt = \int_0^\infty e^{-xt} \, dt = \frac{1}{x}.
\end{align*}
Integrating with respect to $x$, we have $I(x) = \log(x) + C$. To determine the constant $C$, we observe that $I(1) = 0$, which implies $ 0 = \log(1) + C$ and so $C = 0$, concluding the proof.
\end{proof}

With these lemmas at hand, we are now in a position to derive the integral expression for the Lyapunov exponent.

\begin{proof}[Proof of Lemma \ref{lem:lyapunov-exponent-analytic}]
    We first note that the lemma reduces to proving 
    $$ \E\left[\log\left(\left(\sum_{1\leq i \leq d}\phi^2(Y_i)\right)^{1/2}\right)\right]$$
    for $Y_i \sim \mathcal{N}(0,\sigma^2)$ independent. We further observe that $\phi(Y_i) = \sigma \phi(Z_i)$ for the $Z_i \sim \mathcal{N}(0,1)$ independent standard normal random variables, so that we obtain 
    \begin{equation*}
        \E[\log(|\phi(W e_1)|)] = \E\left[\log\left(\sigma\left(\sum_{1\leq i \leq d}\phi^2(Z_i)\right)^{1/2}\right)\right]= \log(\sigma) + \frac{1}{2} \E\left[\log\left(\sum_{1\leq i \leq d} \phi^2(Z_i) \right)\right]
    \end{equation*}
    Since we have by Frullani's identity $\log(x) = \int_0^\infty \frac{\exp(-t) - \exp(-xt)}{t} dt$ for $x > 0$ as shown in Lemma~\ref{Frullani} and as $\sum_{1\leq i \leq d} \phi^2(Z_i) > 0$ almost surely as it is a sum of squares, we obtain 
    \begin{align}
        \E\left[\log\left(\sum_{1\leq i \leq d} \phi^2(Z_i) \right)\right]
        &=\E\left[\int_0^\infty \frac{\exp(-t) - \exp(-t \sum_{i \in [d]} \phi^2(Z_i))}{t}  dt \right] \tag{Frullani} \\
        &= \int_0^\infty \frac{\exp(-t) - \E\left[\exp(-t \sum_{i \in [d]} \phi^2(Z_i)) \right]}{t}  dt \tag{Fubini}\\
        &= \int_0^\infty \frac{\exp(-t) - \E\left[\prod_{i = 1}^d\exp(-t \phi^2(Z_i)) \right]}{t}  dt \\
        &= \int_0^\infty \frac{\exp(-t) - M_{\phi^2(Z_1)}(-t)^d }{t}  dt \\
        &= \int_0^\infty \frac{\exp(-t) -2^{-d} \left((1+2\alpha_1^2 t)^{-1/2}+ (1+2\alpha_2^2 t)^{-1/2} \right)^d }{t}  dt,
    \end{align}
    where we used independence of the $Z_i$ in the penultimate step and have applied Lemma~\ref{lem:mgf} in the last step noting that the requirement on $t$ is satisfied. Hence the claim follows.
\end{proof}

We finally can deduce Theorem~\ref{GeneralLyapunovGaussianFormula} and Theorem~\ref{GeneralLyapunovOrthogonalFormula}, where we deduce Theorem~\ref{GeneralLyapunovOrthogonalFormula} by a little trick from Theorem~\ref{GeneralLyapunovGaussianFormula}.

\begin{proof}(Theorem~\ref{GeneralLyapunovGaussianFormula})
    This follows directly from Proposition~\ref{SphericalStatMeasure} and Lemma~\ref{lem:lyapunov-exponent-analytic}
\end{proof}

\begin{proof}(Theorem~\ref{GeneralLyapunovOrthogonalFormula})
    Denote by $\mu_1$ the measure on $M_d(\R)$ such that the coordinate distributions of $\mu_1$ are independent and all distributed as standard Gaussians $\mathcal{N}(0,1)$. Then, as follows by Lemma~\ref{SphericalStatLemma}, the term $\frac{We_1}{|We_1|}$ is distributed as the volume probability measure on $\mathbb{S}^{d-1}$. Thus, it follows by  Proposition~\ref{SphericalStatMeasure} and Lemma~\ref{lem:lyapunov-exponent-analytic} that for $W\sim \mu_1$,
    \begin{align*}
        \lambda_{m_{\eta\cdot\mathrm{O}(d)},\phi} &= \E[\log |\phi(\tfrac{\eta We_1}{|We_1|})|] \\
        &= \E[\log |\tfrac{\eta}{|We_1|}\phi(We_1)|] \\
        &= \log(\eta) + \E[\log |\phi(We_1)|] - \E[\log |We_1|] \\
        &= \log(\eta) + I(d,\alpha_1,\alpha_2) - I(d,1,1),
    \end{align*} having used Proposition~\ref{SphericalStatMeasure} in the first line, that $\phi$ is positively homogeneous in the second line and Lemma~\ref{lem:lyapunov-exponent-analytic} in the fourth once for $\phi(x) = \max(\alpha_1x,\alpha_2x)$ and once for $\phi(x) = x$.
\end{proof}

\subsection{Asymptotic Expansion for $I(d,\alpha)$: Proof of Theorem~\ref{AsymptoticExpansion}}\label{section:AsymptoticExpansion}

As in the main part of the paper, we write for $d\in \mathbb{N}_{\neq 0}$ and $\alpha \in \R_{\neq 0}$, $$I(d,\alpha) = \int_0^\infty \frac{e^{-t} - \frac{1}{2^d} \left(\frac{1}{\sqrt{1+2 t}} + \frac{1}{\sqrt{1+2\alpha^2 t}}  \right)^d }{2t}  dt.$$ We prove the following asymptotic expansion for $I(d,\alpha)$.

\begin{proposition}\label{MainIdalphbound}
For $d\in \mathbb{N}_{\neq 0}$ and $\alpha \in \R_{\neq 0}$ it holds as $d \to \infty$ that
\[
I(d,\alpha) =  \frac{1}{2}\log\left(d\frac{1+\alpha^2}{2}\right) - \frac{C_\alpha}{4d} + O_{\alpha}(d^{-2})
\]
where $C_\alpha = \frac{5 - 2\alpha^2 + 5\alpha^4}{(1+\alpha^2)^2}$ and the implied constant depends on $\alpha$.
\end{proposition}

We can easily conclude the proof of Theorem~\ref{AsymptoticExpansion} by Proposition~\ref{MainIdalphbound}.

\begin{proof}(of Theorem~\ref{AsymptoticExpansion})
    Equation~\ref{AsymtoticExpansionIdalpha} is exactly Proposition~\ref{MainIdalphbound} and equation~\ref{AsymtoticExpansionGaussian} follows by Proposition~\ref{MainIdalphbound} and Theorem~\ref{LyapunovGaussianFormula}. For equation~\ref{AsymtoticExpansionOrthogonal} we calculate by using that $C_1 = 2$ that
    \begin{align*}
        \lambda_{m_{\eta\cdot\mathrm{O}(d)},\phi} &= \log(\eta) + I(d,\alpha) - I(d,1) \\
        &= \log(\eta) + \left( \frac{1}{2}\log\left(d\frac{1+\alpha^2}{2}\right) - \frac{C_\alpha}{4d} + O_{\alpha}(d^{-2}) \right) - \left(  \frac{1}{2}\log(d) - \frac{1}{2d} + O_{\alpha}(d^{-2})\right) \\
        &= \frac{1}{2}\log\left(\eta^2\frac{1+\alpha^2}{2}\right) + \frac{2-C_{\alpha}}{4d} + O_{\alpha}(d^{-2}),
    \end{align*} concluding the proof by having used Theorem~\ref{LyapunovOrthogonalFormula} in the first line and Proposition~\ref{MainIdalphbound} for $I(d,\alpha)$ and $I(d,1)$ in the second.
\end{proof}

Let $W$ be a Gaussian matrix with i.i.d $\mathcal{N}(0,1)$ entries. Then for $\phi(x) = \max(x,\alpha x)$ and $e_1 = (1,0,\ldots , 0)^T \in \R^d$ we have by Proposition~\ref{SphericalStatMeasure} and Theorem~\ref{LyapunovGaussianFormula} that $$I(d,\alpha) = \E[\log |\phi(We_1)|].$$ Let $V = (V_1, \ldots , V_d) = W e_1$. Since the $W_{ij}$ are i.i.d.\ $\mathcal{N}(0, 1)$, the vector $V$ follows a multivariate Gaussian distribution $V \sim \mathcal{N}(0,  I_d)$. Write $X_i = \phi(V_i)^2$ and denote
\begin{equation}
S_d  = \sum_{i=1}^d X_i = \sum_{i=1}^d \phi(V_i)^2
\end{equation} and note that we seek to compute $$\mathbb{E}\left[ \log |\phi(We_1)| \right] = \frac{1}{2}\E[\log S_d].$$ 

Note that, similarly to Lemma~\ref{lem:mgf}, it holds that the mean of $X_i$ is
\begin{align*}
    \E[X_i] &= \E[X_i|V_i \geq 0]\mathbb{P}[V_i \geq 0] + \E[X_i|V_i < 0]\mathbb{P}[V_i < 0] \\
    &= \frac{1}{2}(\E[V_i^2] + \E[\alpha^2 V_i^2]) = \frac{1}{2}(1 + \alpha^2)
\end{align*} and, by a similar argument, the variance 
\begin{align*}
    \text{Var}(X_i) &= \E[X_i^2] - \E[X_i]^2 \\
    &= \frac{1}{2}(1 + \alpha^4)\E[V_i^4] - \frac{1}{4}(1 + \alpha^2)^2 \\
    &= \frac{1}{2}(3 + 3\alpha^4) - \frac{1}{4}(1 + \alpha^2)^2 \\ &= \frac{1}{4}(6 + 6\alpha^4 - (1 + 2\alpha^2 +\alpha^4)) = \frac{1}{4}(5 - 2\alpha^2 + 5\alpha^4)
\end{align*}
For convenience, we denote
\begin{equation}
\mu = \E[X_i] = \frac{1}{2}(1+\alpha^2), \quad\text{ and } \quad \tau^2 = \text{Var}(X_i) = \frac{1}{4}(5 - 2\alpha^2 + 5\alpha^4).
\end{equation}
and define the normalized deviation as $$\Delta = \frac{S_d - d\mu}{d\mu}.$$ We can then write
\begin{equation}\label{logSdExpression}
\E[\log S_d] = \E[\log(d\mu (1 + \Delta))] = \log(d\mu) + \E[\log(1 + \Delta)] = \log\left(\frac{d}{2}(1 + \alpha^2) \right) + \E[\log(1 + \Delta)].
\end{equation}

In order to do a Taylor expansion of $\log(1 + \Delta)$, we establish the following results on the moments of $\Delta$.

\begin{lemma}\label{AsymptoticExpansionFirstLemma}
    In the above setting, fix $\delta > 0$ and denote $E_{\delta} = \{ |\Delta| \leq \delta \}$. Then there are constants $C_1, C_2 > 0$ depending on $\delta$ and $\alpha$ such that
    $$\Prob(E_\delta^c) = \Prob(|S_d - d\mu| > \delta d\mu) \le C_1 e^{-C_2 d}.$$
\end{lemma}

\begin{proof}
We make extensive use of material from \cite[Section 2]{VershyninBook}. Denote by $||\cdot||_{\psi_2}$ the subgaussian norm as defined in \cite[Section 2.6.1]{VershyninBook} and by $||\cdot||_{\psi_1}$ the subexponential norm as in \cite[Section 2.8.2]{VershyninBook}. Recall that by \cite[Lemma 2.8.6]{VershyninBook} it holds that $||X^2||_{\psi_1} \leq ||X||_{\psi_2}^2$ for any random variable $X$. Thus, since \(V_i\sim N(0,1)\) is Gaussian and therefore $||X||_{\psi_2} < \infty$, \(V_i^2\) is subexponential, that is $||V_i^2||_{\psi_1} < \infty$. As
$0\le X_i=\phi(V_i)^2\le \max\{1,\alpha^2\}V_i^2,$
the centered variables \(Y_i=X_i-\mu\) are independent, mean-zero, and satisfy
\(\|Y_i\|_{\psi_1}\le K_\alpha\) for some constant \(K_\alpha<\infty\) depending
only on \(\alpha\). Bernstein's inequality for averages of independent centered
sub-exponential random variables \cite[Corollary 2.9.1]{VershyninBook}
therefore gives
\begin{align*}
\mathbb P(E_\delta^c)
=
\mathbb P\left(
\left|\frac1d\sum_{i=1}^d Y_i\right|>\delta\mu
\right)  \le
2\exp\left[
-cd\min\left\{
\frac{\delta^2\mu^2}{K_\alpha^2},
\frac{\delta\mu}{K_\alpha}
\right\}
\right].
\end{align*}
Thus \(\mathbb P(E_\delta^c)\le C_1e^{-C_2d}\), where
$C_1=2$
and $C_2
=
c\min\left\{
\frac{\delta^2\mu^2}{K_\alpha^2},
\frac{\delta\mu}{K_\alpha}
\right\}>0.$
\end{proof}


\begin{lemma}\label{AsymptoticExpansionSecondLemma}
    In the above setting, it holds that: \begin{enumerate}
        \item[(1)] $\E[\Delta] = 0$.
        \item[(2)] $\E[\Delta^2] = \frac{\tau^2}{d\mu^2}$.
        \item[(3)] $\E[\Delta^3] = O_{\alpha}(d^{-2})$.
        \item[(4)] $\E[\Delta^4] = O_{\alpha}(d^{-2})$.
    \end{enumerate}
\end{lemma}

\begin{proof}
    (1) is clear and for (2) we calculate $\E[\Delta^2] = \frac{1}{d^2\mu^2} \E[(S_d - d\mu)^2] = \frac{d\cdot\mathrm{Var}(X_i)}{d^2\mu^2} = \frac{\tau^2}{d\mu^2}.$ For (3) write for convenience $Y_i = X_i - \mu$ so that $\E[\Delta^3] = \frac{1}{d^3\mu^3}\sum_{i,j,k = 1}^d \E[Y_iY_jY_k]$. By independence of the $Y_i$ and as $\E[Y_i] = 0$ it follows that the latter sum is only non-zero if $i=j=k$. Also, since $Y_i$ is subexponential, it follows that $E[Y_i^3] < \infty$ and hence as the $Y_i$ are identically distributed, $\E[\Delta^3] = \frac{1}{d^3\mu^3}\sum_{i = 1}^d \E[Y_i^3] = \frac{d\E[Y_1^3]}{d^3\mu^3} = O_{\alpha}(d^{-2})$. Finally for (4) a similar argument applies. Indeed, $\E[\Delta^4] = \frac{1}{d^4\mu^4}\sum_{i,j,k,h = 1}^d \E[Y_iY_jY_kY_h]$ is only non-zero if every index is equal to at least one other index. A brief combinatorial calculation then leads to $\E[\Delta^4] = O_{\alpha}(d^{-2})$.
\end{proof}

\begin{proof}(of Proposition~\ref{MainIdalphbound})
    The notation introduced above is used in the proof. We use the event $E_{\delta}$ for some fixed $\delta \in (0,1/2)$ to split the expectation $\mathbb E[\log(1+\Delta)]$ into
\[
\mathbb E[\log(1+\Delta)]
=
\mathbb E[\log(1+\Delta)\mathbf 1_{E_\delta}]
+
\mathbb E[\log(1+\Delta)\mathbf 1_{E_\delta^c}].
\]

On the event $E_\delta$, we have $\Delta\in[-\delta,\delta]$, so Taylor's theorem gives
\[
\log(1+\Delta)
=
\Delta-\frac{\Delta^2}{2}+\frac{\Delta^3}{3}+R_4(\Delta),
\]
for some function $R_4$. On the interval $\Delta \in [-\delta,\delta]$ we have since $\delta \in (0,1/2)$ that
$|R_4(\Delta)|\le C_\delta |\Delta|^4$
for some constant $C_\delta>0$. Therefore,
\[
\mathbb E[\log(1+\Delta)\mathbf 1_{E_\delta}]
=
\mathbb E[\Delta\mathbf 1_{E_\delta}]
-\frac12\mathbb E[\Delta^2\mathbf 1_{E_\delta}]
+\frac13\mathbb E[\Delta^3\mathbf 1_{E_\delta}]
+O_\alpha\!\bigl(\mathbb E[\Delta^4\mathbf 1_{E_\delta}]\bigr).
\]
Using $\mathbf 1_{E_\delta}\le 1$ and Lemma \ref{AsymptoticExpansionSecondLemma} (4) it follows $\mathbb E[\Delta^4\mathbf 1_{E_\delta}] = O_{\alpha}(d^{-2})$. To deal with $E[\Delta^3\mathbf 1_{E_\delta}]$ we use Hölder's inequality and  Lemma \ref{AsymptoticExpansionFirstLemma} as well as Lemma \ref{AsymptoticExpansionSecondLemma} (4) to deduce $|\E[\Delta^3\mathbf 1_{E_\delta^c}]| \leq \mathbb{E}[\Delta^4]^{3/4} \mathbb{P}[E_{\delta}^c]^{1/4} = O_{\alpha}(d^{-3/2}e^{-cd}) = O_{\alpha}(d^{-2})$ and therefore by \ref{AsymptoticExpansionSecondLemma} (3) it follows $\E[\Delta^3\mathbf 1_{E_\delta}] = \E[\Delta^3] + O_{\delta}(d^{-2}) = O_{\alpha}(d^{-2})$.  So we conclude
\begin{equation}\label{StepCalculation}
\mathbb E[\log(1+\Delta)\mathbf 1_{E_\delta}]
=
\mathbb E[\Delta\mathbf 1_{E_\delta}]
-\frac12\mathbb E[\Delta^2\mathbf 1_{E_\delta}]
+O_\alpha(d^{-2}).
\end{equation}

We now compare truncated moments with full moments. By Cauchy--Schwarz, Lemma \ref{AsymptoticExpansionFirstLemma} and Lemma \ref{AsymptoticExpansionSecondLemma} (2) and (4),
$\bigl|\mathbb E[\Delta\mathbf 1_{E_\delta^c}]\bigr|
\le
\bigl(\mathbb E[\Delta^2]\bigr)^{1/2}\mathbb P(E_\delta^c)^{1/2}
=
O_\alpha(d^{-1/2}e^{-c d})
=
O_\alpha(d^{-2})$
and similarly
$\mathbb E[\Delta^2\mathbf 1_{E_\delta^c}]
\le
\bigl(\mathbb E[\Delta^4]\bigr)^{1/2}\mathbb P(E_\delta^c)^{1/2}
=
O_\alpha(d^{-1}e^{-c d})
=
O_\alpha(d^{-2}).$
Hence
\[
\mathbb E[\Delta\mathbf 1_{E_\delta}]
=
\mathbb E[\Delta]+O_\alpha(d^{-2})
=
O_\alpha(d^{-2}) \qquad \text{ and } \qquad
\mathbb E[\Delta^2\mathbf 1_{E_\delta}]
=
\mathbb E[\Delta^2]+O_\alpha(d^{-2})
=
\frac{\tau^2}{d\mu^2}+O_\alpha(d^{-2}).
\]
Combining these two estimates with equation \eqref{StepCalculation} and Lemma~\ref{AsymptoticExpansionSecondLemma} (1) and (2) we conclude
\begin{equation}\label{EdeltaCalc}
\mathbb E[\log(1+\Delta)\mathbf 1_{E_\delta}]
=
-\frac{\tau^2}{2d\mu^2}+O_\alpha(d^{-2}).
\end{equation}

It remains to bound the complement term $\mathbb E[\log(1+\Delta)\mathbf 1_{E_\delta^c}]$. Since $1+\Delta=S_d/(d\mu)>0$ and we have for all $x>0$ that
$|\log x|\le x+x^{-1},$
it follows that
\[
|\log(1+\Delta)|
\le
\frac{S_d}{d\mu}+\frac{d\mu}{S_d}.
\]
Also, because $\alpha\neq 0$,
\[
X_i=\phi(V_i)^2\ge \beta V_i^2,\qquad \beta:=\min\{1,\alpha^2\}>0,
\]
so
\[
S_d\ge \beta \sum_{i=1}^d V_i^2.
\]
Since $\sum_{i=1}^d V_i^2\sim \chi_d^2$, its negative second moment is finite for all sufficiently large $d$, and thus
\[
\mathbb E\!\left[\left(\frac{d\mu}{S_d}\right)^2\right]=O_\alpha(1).
\]
Similarly, since $X_i$ has finite fourth moment by Lemma~\ref{AsymptoticExpansionSecondLemma} (4),
\[
\mathbb E\!\left[\left(\frac{S_d}{d\mu}\right)^2\right]=O_\alpha(1).
\]
Hence, by Cauchy--Schwarz and Lemma D.9,
\[
\bigl|\mathbb E[\log(1+\Delta)\mathbf 1_{E_\delta^c}]\bigr|
\le
\left(
\mathbb E\!\left[\left(\frac{S_d}{d\mu}+\frac{d\mu}{S_d}\right)^2\right]
\right)^{1/2}
\mathbb P(E_\delta^c)^{1/2}
=
O_\alpha(e^{-c d})
=
O_\alpha(d^{-2}).
\]

Combining the estimates on $E_\delta$ and $E_\delta^c$, we obtain
\[
\mathbb E[\log(1+\Delta)]
=
-\frac{\tau^2}{2d\mu^2}+O_\alpha(d^{-2}).
\]
Therefore,
\[
\mathbb E[\log S_d]
=
\log(d\mu)+\mathbb E[\log(1+\Delta)]
=
\log(d\mu)-\frac{\tau^2}{2d\mu^2}+O_\alpha(d^{-2}),
\]
and so
\[
I(d,\alpha)
=
\frac12\,\mathbb E[\log S_d]
=
\frac12\log(d\mu)-\frac{\tau^2}{4d\mu^2}+O_\alpha(d^{-2}).
\]
Since
$\mu=\frac{1+\alpha^2}{2}$ and $
\tau^2=\frac{5-2\alpha^2+5\alpha^4}{4},$
we conclude that
\[
I(d,\alpha)
=
\frac12\log\!\left(\frac d2(1+\alpha^2)\right)
-\frac{5-2\alpha^2+5\alpha^4}{4d(1+\alpha^2)^2}
+O_\alpha(d^{-2})=
\frac12\log\!\left(\frac d2(1+\alpha^2)\right)
-\frac{C_\alpha}{4d}
+O_\alpha(d^{-2})\]
for $C_\alpha:=\frac{5-2\alpha^2+5\alpha^4}{(1+\alpha^2)^2}
$. This concludes the proof.
\end{proof}


\newpage

\subsection{Some more computational results}\label{section:lookuptables}

In this section, we give extensive lookup tables for the various parameters discussed in this paper for $\alpha = 0.1$, $\alpha = 0.01$ and $\alpha = 0.001$. The value $I(d,\alpha)$ is defined in \eqref{GaussianFormula}. When the Leaky ReLU activation function has slope $\alpha$, we set $\sigma_{\mathrm{He}} = \sqrt{\frac{2}{d(1 + \alpha^2)}}$ as in \eqref{DefSigmaHe}. We then set $\lambda_{\mathrm{He}}$ as in \eqref{DefLambdaHe} so that $\lambda_{\mathrm{He}}$ is the Lyapunov exponent of He initialization. $\lambda_{\mathrm{orth}}$ defined in \eqref{DefLambdaHe} is the Lyapunov exponent of standard scaled orthogonal initialization. Finally, $\sigma_{\mathrm{crit}} = \exp(-I(d,\alpha))$ as in \eqref{sigmacrit} is the standard deviation such that the Gaussian initialization has Lyapunov exponent $0$. Similarly, $\eta_{\mathrm{crit}} = \exp(I(d,1) - I(d,\alpha))$  as given in \eqref{sigmacrit} is the scaling factor of the orthogonal group such that the Lyapunov exponent is zero.

\vfill

\begin{table}[h!]
\centering
\caption{Lookup table for $\alpha = 0.1$}
\label{tab:lookup0.1}
\begin{tabular}{cccccccc}
\midrule
 $d$ & $I(d,\alpha)$ & $I(d,1)$ & $\lambda_{\mathrm{He}}$ & $\lambda_{\mathrm{orth}}$ & $\sigma_{\mathrm{He}}$ &$\sigma_{\mathrm{crit}}$ & $\eta_{\mathrm{crit}}$ \\
\midrule
1          & -1.786474 & -0.6351814 & -1.4448755 & -0.8096941 & 1.4071951 & 5.9683707 & 3.1622777 \\
2          & -0.816599 & 0.0579658 & -0.8215742 & -0.5329663 & 0.9950372 & 2.262791 & 2.3978315 \\
3          & -0.3529343 & 0.3648186 & -0.560642 & -0.3761544 & 0.8124445 & 1.4232376 & 2.0498217 \\
4          & -0.0636424 & 0.5579658 & -0.4151912 & -0.2800097 & 0.7035975 & 1.0657112 & 1.8619199 \\
5          & 0.1390395 & 0.6981519 & -0.324081 & -0.217514 & 0.6293168 & 0.8701936 & 1.7491193 \\
6          & 0.2914356 & 0.8079658 & -0.2628457 & -0.1749317 & 0.574485 & 0.7471901 & 1.6762014 \\
7          & 0.4118081 & 0.8981519 & -0.2195485 & -0.1447454 & 0.5318698 & 0.6624514 & 1.626359 \\
8          & 0.5104289 & 0.9746324 & -0.1876934 & -0.1226051 & 0.4975186 & 0.6002381 & 1.5907467 \\
9          & 0.593532 & 1.0410091 & -0.1634819 & -0.1058786 & 0.469065 & 0.5523729 & 1.5643604 \\
10         & 0.6651223 & 1.0996324 & -0.1445718 & -0.0929117 & 0.4449942 & 0.5142106 & 1.5442064 \\
16         & 0.9600277 & 1.3543943 & -0.0846682 & -0.0527682 & 0.3517988 & 0.3828823 & 1.4834443 \\
20         & 1.0900216 & 1.4724499 & -0.0662461 & -0.0408299 & 0.3146584 & 0.3362092 & 1.4658398 \\
30         & 1.31609 & 1.6837469 & -0.0429103 & -0.0260585 & 0.2569175 & 0.2681819 & 1.4443465 \\
32         & 1.3511818 & 1.7170803 & -0.0400877 & -0.0243 & 0.2487593 & 0.2589341 & 1.4418088 \\
40         & 1.471102 & 1.8318356 & -0.0317393 & -0.0191352 & 0.2224971 & 0.2296722 & 1.4343813 \\
50         & 1.5892275 & 1.9459448 & -0.0251856 & -0.0151189 & 0.1990074 & 0.2040832 & 1.428632 \\
60         & 1.6846978 & 2.0387927 & -0.020876 & -0.0124964 & 0.1816681 & 0.1855005 & 1.4248903 \\
64         & 1.7183043 & 2.0715884 & -0.0195388 & -0.0116856 & 0.1758994 & 0.17937 & 1.4237355 \\
70         & 1.764823 & 2.1170707 & -0.0178262 & -0.0106493 & 0.168192 & 0.1712171 & 1.4222608 \\
80         & 1.8338609 & 2.1847373 & -0.015554 & -0.009278 & 0.1573292 & 0.1597954 & 1.4203117 \\
90         & 1.8945107 & 2.2443287 & -0.0137957 & -0.0082195 & 0.1483314 & 0.1503919 & 1.4188092 \\
100        & 1.9485921 & 2.2975684 & -0.0123946 & -0.0073779 & 0.1407195 & 0.1424745 & 1.4176156 \\
128        & 2.0747663 & 2.4220987 & -0.0096504 & -0.005734 & 0.1243796 & 0.1255858 & 1.4152871 \\
200        & 2.3014107 & 2.6466545 & -0.0061495 & -0.0036454 & 0.0995037 & 0.1001175 & 1.4123342 \\
256        & 2.426194 & 2.770633 & -0.0047963 & -0.0028406 & 0.0879497 & 0.0883725 & 1.4111981 \\
300        & 2.5062035 & 2.8502227 & -0.0040893 & -0.0024207 & 0.0812444 & 0.0815774 & 1.4106057 \\
400        & 2.6510708 & 2.9944812 & -0.0030631 & -0.001812 & 0.0703598 & 0.0705756 & 1.4097472 \\
500        & 2.763257 & 3.1063034 & -0.0024486 & -0.0014479 & 0.0629317 & 0.063086 & 1.409234 \\
512        & 2.7751729 & 3.1181851 & -0.002391 & -0.0014138 & 0.0621898 & 0.0623387 & 1.409186 \\
600        & 2.8548269 & 3.197631 & -0.0020395 & -0.0012056 & 0.0574485 & 0.0575658 & 1.4088927 \\
700        & 2.9321943 & 3.2748255 & -0.0017475 & -0.0010328 & 0.053187 & 0.05328 & 1.4086492 \\
800        & 2.9991788 & 3.3416806 & -0.0015286 & -0.0009033 & 0.0497519 & 0.049828 & 1.4084668 \\
900        & 3.0582404 & 3.4006416 & -0.0013585 & -0.0008027 & 0.0469065 & 0.0469703 & 1.4083251 \\
1000       & 3.1110568 & 3.4533774 & -0.0012225 & -0.0007222 & 0.0444994 & 0.0445538 & 1.4082118 \\
1024       & 3.1229437 & 3.4652474 & -0.0011938 & -0.0007053 & 0.0439748 & 0.0440274 & 1.4081879 \\
\bottomrule
\end{tabular}
\end{table}

\vfill

\newpage

\vfill

\begin{table}[h!]
\centering
\caption{Lookup table for $\alpha = 0.01$}
\label{tab:lookup0.01}
\begin{tabular}{cccccccc}
\midrule
 $d$ & $I(d,\alpha)$ & $I(d,1)$ & $\lambda_{\mathrm{He}}$ & $\lambda_{\mathrm{orth}}$ & $\sigma_{\mathrm{He}}$ &$\sigma_{\mathrm{crit}}$ & $\eta_{\mathrm{crit}}$ \\
\midrule
1          & -2.9377665 & -0.6351814 & -2.5912429 & -1.9560615 & 1.4141429 & 18.8736452 & 10.0 \\
2          & -1.4349252 & 0.0579658 & -1.4349752 & -1.1463674 & 0.99995 & 4.1993309 & 4.4499416 \\
3          & -0.6949542 & 0.3648186 & -0.8977368 & -0.7132492 & 0.8164558 & 2.0036174 & 2.8857154 \\
4          & -0.2587831 & 0.5579658 & -0.6054067 & -0.4702253 & 0.7070714 & 1.2953528 & 2.2631301 \\
5          & 0.0236727 & 0.6981519 & -0.4345227 & -0.3279556 & 0.6324239 & 0.9766053 & 1.9630104 \\
6          & 0.2203017 & 0.8079658 & -0.3290544 & -0.2411405 & 0.5773214 & 0.8022767 & 1.7997793 \\
7          & 0.365729 & 0.8981519 & -0.2607025 & -0.1858994 & 0.5344958 & 0.6936908 & 1.7030537 \\
8          & 0.4788695 & 0.9746324 & -0.2143277 & -0.1492393 & 0.499975 & 0.6194833 & 1.6417503 \\
9          & 0.5706 & 1.0410091 & -0.1814887 & -0.1238855 & 0.471381 & 0.5651862 & 1.6006488 \\
10         & 0.6474594 & 1.0996324 & -0.1573095 & -0.1056494 & 0.4471912 & 0.5233738 & 1.5717238 \\
16         & 0.9515602 & 1.3543943 & -0.0882106 & -0.0563106 & 0.3535357 & 0.3861381 & 1.4960588 \\
20         & 1.0827713 & 1.4724499 & -0.0685712 & -0.043155 & 0.316212 & 0.3386557 & 1.4765061 \\
30         & 1.3098905 & 1.6837469 & -0.0441845 & -0.0273328 & 0.258186 & 0.2698496 & 1.4533284 \\
32         & 1.3450864 & 1.7170803 & -0.0412579 & -0.0254702 & 0.2499875 & 0.2605172 & 1.450624 \\
40         & 1.4652935 & 1.8318356 & -0.0326226 & -0.0200184 & 0.2235956 & 0.2310102 & 1.4427371 \\
50         & 1.5836255 & 1.9459448 & -0.0258624 & -0.0157957 & 0.19999 & 0.2052297 & 1.4366576 \\
60         & 1.6792239 & 2.0387927 & -0.0214248 & -0.0130452 & 0.1825651 & 0.1865187 & 1.4327114 \\
64         & 1.7128689 & 2.0715884 & -0.020049 & -0.0121958 & 0.1767679 & 0.1803476 & 1.4314951 \\
70         & 1.7594363 & 2.1170707 & -0.0182877 & -0.0111108 & 0.1690224 & 0.1721419 & 1.4299428 \\
80         & 1.8285374 & 2.1847373 & -0.0159523 & -0.0096763 & 0.158106 & 0.1606484 & 1.4278929 \\
90         & 1.8892353 & 2.2443287 & -0.014146 & -0.0085698 & 0.1490637 & 0.1511874 & 1.4263139 \\
100        & 1.9433543 & 2.2975684 & -0.0127072 & -0.0076905 & 0.1414143 & 0.1432227 & 1.4250603 \\
128        & 2.0696008 & 2.4220987 & -0.0098907 & -0.0059743 & 0.1249938 & 0.1262362 & 1.4226167 \\
200        & 2.2963348 & 2.6466545 & -0.0063003 & -0.0037961 & 0.099995 & 0.100627 & 1.4195213 \\
256        & 2.421152 & 2.770633 & -0.0049131 & -0.0029575 & 0.0883839 & 0.0888192 & 1.4183313 \\
300        & 2.5011791 & 2.8502227 & -0.0041886 & -0.00252 & 0.0816456 & 0.0819883 & 1.417711 \\
400        & 2.6460716 & 2.9944812 & -0.0031371 & -0.001886 & 0.0707071 & 0.0709293 & 1.4168125 \\
500        & 2.7582728 & 3.1063034 & -0.0025076 & -0.0015069 & 0.0632424 & 0.0634012 & 1.4162755 \\
512        & 2.7701901 & 3.1181851 & -0.0024487 & -0.0014714 & 0.0624969 & 0.0626501 & 1.4162252 \\
600        & 2.8498527 & 3.197631 & -0.0020885 & -0.0012547 & 0.0577321 & 0.0578528 & 1.4159183 \\
700        & 2.9272271 & 3.2748255 & -0.0017895 & -0.0010748 & 0.0534496 & 0.0535453 & 1.4156636 \\
800        & 2.9942169 & 3.3416806 & -0.0015653 & -0.00094 & 0.0499975 & 0.0500758 & 1.4154728 \\
900        & 3.0532827 & 3.4006416 & -0.0013911 & -0.0008353 & 0.0471381 & 0.0472037 & 1.4153246 \\
1000       & 3.1061023 & 3.4533774 & -0.0012518 & -0.0007516 & 0.0447191 & 0.0447751 & 1.4152061 \\
1024       & 3.1179899 & 3.4652474 & -0.0012224 & -0.0007339 & 0.044192 & 0.044246 & 1.4151811 \\
\bottomrule
\end{tabular}
\end{table}

\vfill

\newpage

\vfill

\begin{table}[h!]
\centering
\caption{Lookup table for $\alpha = 0.001$}
\label{tab:lookup0.001}
\begin{tabular}{cccccccc}
\midrule
 $d$ & $I(d,\alpha)$ & $I(d,1)$ & $\lambda_{\mathrm{He}}$ & $\lambda_{\mathrm{orth}}$ & $\sigma_{\mathrm{He}}$ &$\sigma_{\mathrm{crit}}$ & $\eta_{\mathrm{crit}}$ \\
\midrule
1          & -4.0890591 & -0.6351814 & -3.742486 & -3.1073045 & 1.4142129 & 59.6837066 & 31.6227766 \\
2          & -2.0150469 & 0.0579658 & -2.0150474 & -1.7264396 & 0.9999995 & 7.5010792 & 7.9487339 \\
3          & -0.9881306 & 0.3648186 & -1.1908637 & -1.0063761 & 0.8164962 & 2.6862083 & 3.8688187 \\
4          & -0.4073402 & 0.5579658 & -0.7539143 & -0.6187329 & 0.7071064 & 1.5028153 & 2.6255908 \\
5          & -0.0518063 & 0.6981519 & -0.5099522 & -0.4033852 & 0.6324552 & 1.0531718 & 2.1169116 \\
6          & 0.181837 & 0.8079658 & -0.3674697 & -0.2795557 & 0.57735 & 0.8337372 & 1.870356 \\
7          & 0.3460547 & 0.8981519 & -0.2803272 & -0.2055241 & 0.5345222 & 0.7074738 & 1.7368918 \\
8          & 0.4687565 & 0.9746324 & -0.2243912 & -0.1593028 & 0.4999998 & 0.6257799 & 1.6584375 \\
9          & 0.5653643 & 1.0410091 & -0.1866749 & -0.1290717 & 0.4714043 & 0.5681532 & 1.6090514 \\
10         & 0.6447186 & 1.0996324 & -0.1600009 & -0.1083408 & 0.4472134 & 0.5248102 & 1.5760376 \\
16         & 0.951425 & 1.3543943 & -0.0882963 & -0.0563963 & 0.3535532 & 0.3861903 & 1.4962611 \\
20         & 1.0826938 & 1.4724499 & -0.0685992 & -0.043183 & 0.3162276 & 0.3386819 & 1.4766205 \\
30         & 1.3098279 & 1.6837469 & -0.0441977 & -0.0273459 & 0.2581988 & 0.2698665 & 1.4534195 \\
32         & 1.3450249 & 1.7170803 & -0.04127 & -0.0254823 & 0.2499999 & 0.2605332 & 1.4507133 \\
40         & 1.465235 & 1.8318356 & -0.0326317 & -0.0200275 & 0.2236067 & 0.2310237 & 1.4428216 \\
50         & 1.5835691 & 1.9459448 & -0.0258693 & -0.0158027 & 0.1999999 & 0.2052413 & 1.4367387 \\
60         & 1.6791688 & 2.0387927 & -0.0214304 & -0.0130508 & 0.1825741 & 0.186529 & 1.4327904 \\
64         & 1.7128142 & 2.0715884 & -0.0200542 & -0.0122011 & 0.1767766 & 0.1803575 & 1.4315734 \\
70         & 1.7593821 & 2.1170707 & -0.0182924 & -0.0111156 & 0.1690308 & 0.1721512 & 1.4300203 \\
80         & 1.8284839 & 2.1847373 & -0.0159564 & -0.0096803 & 0.1581138 & 0.160657 & 1.4279694 \\
90         & 1.8891822 & 2.2443287 & -0.0141496 & -0.0085734 & 0.1490711 & 0.1511954 & 1.4263896 \\
100        & 1.9433016 & 2.2975684 & -0.0127104 & -0.0076937 & 0.1414213 & 0.1432303 & 1.4251354 \\
128        & 2.0695489 & 2.4220987 & -0.0098932 & -0.0059768 & 0.1249999 & 0.1262427 & 1.4226906 \\
200        & 2.2962838 & 2.6466545 & -0.0063018 & -0.0037976 & 0.1 & 0.1006321 & 1.4195937 \\
256        & 2.4211013 & 2.770633 & -0.0049143 & -0.0029586 & 0.0883883 & 0.0888237 & 1.4184032 \\
300        & 2.5011286 & 2.8502227 & -0.0041896 & -0.002521 & 0.0816496 & 0.0819924 & 1.4177826 \\
400        & 2.6460213 & 2.9944812 & -0.0031379 & -0.0018868 & 0.0707106 & 0.0709329 & 1.4168837 \\
500        & 2.7582227 & 3.1063034 & -0.0025082 & -0.0015075 & 0.0632455 & 0.0634044 & 1.4163464 \\
512        & 2.77014 & 3.1181851 & -0.0024492 & -0.001472 & 0.0625 & 0.0626532 & 1.4162961 \\
600        & 2.8498027 & 3.197631 & -0.002089 & -0.0012552 & 0.057735 & 0.0578557 & 1.4159891 \\
700        & 2.9271772 & 3.2748255 & -0.0017899 & -0.0010753 & 0.0534522 & 0.053548 & 1.4157343 \\
800        & 2.9941671 & 3.3416806 & -0.0015657 & -0.0009404 & 0.05 & 0.0500783 & 1.4155434 \\
900        & 3.0532329 & 3.4006416 & -0.0013914 & -0.0008356 & 0.0471404 & 0.0472061 & 1.4153951 \\
1000       & 3.1060525 & 3.4533774 & -0.0012521 & -0.0007519 & 0.0447213 & 0.0447774 & 1.4152765 \\
1024       & 3.1179401 & 3.4652474 & -0.0012227 & -0.0007342 & 0.0441942 & 0.0442482 & 1.4152515 \\
\bottomrule
\end{tabular}
\end{table}

\vfill 

\newpage

\section{Failure of the Central Limit Theorem for $tanh$ activations}\label{section:FailureCLT}

In this section we present a plot that shows that the Central Limit Theorem fails for $\tanh$-networks.

\begin{figure*}[!ht]

   \begin{subfigure}{0.33\linewidth}
    \centering
    \includegraphics[width=\linewidth]{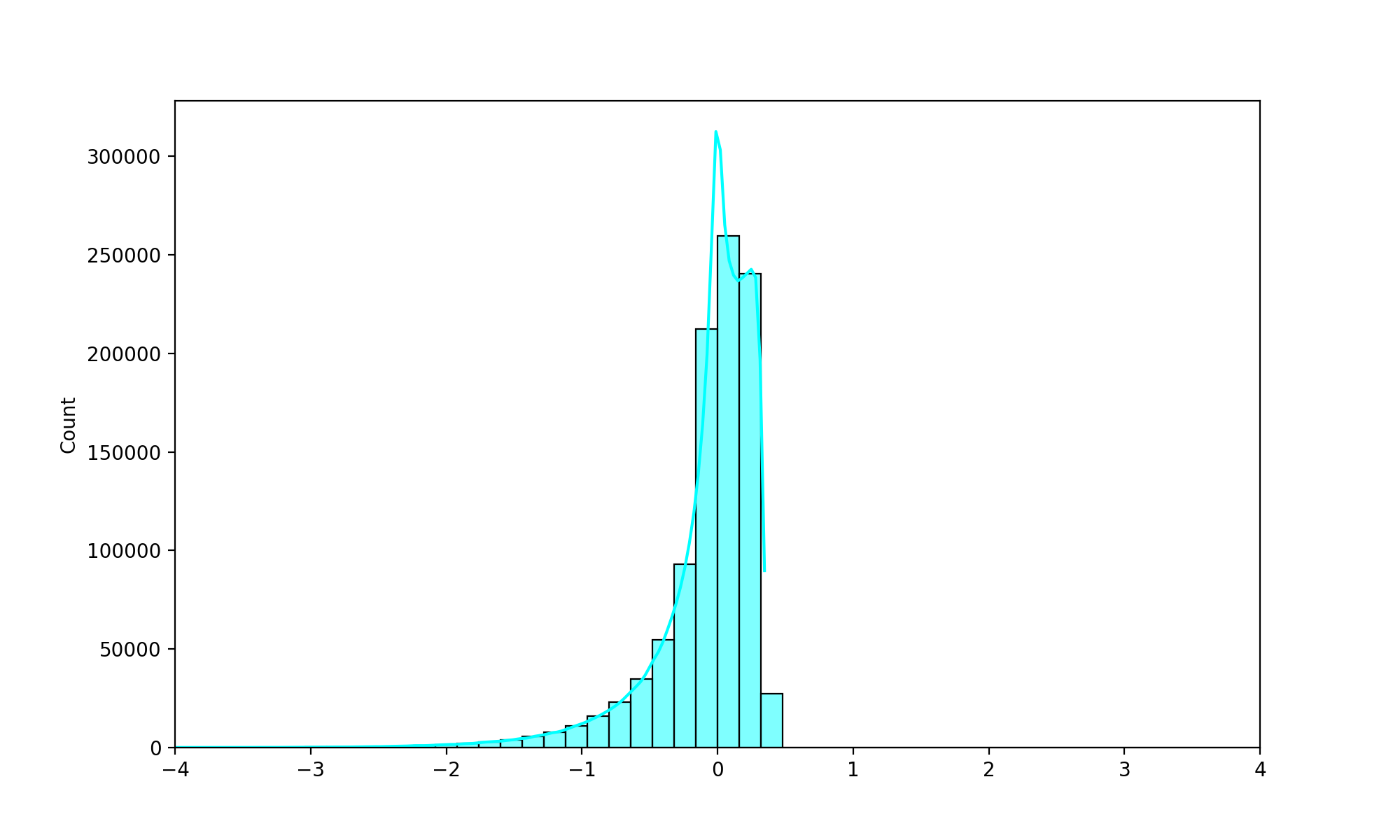}
    \caption{1 Layer}
  \end{subfigure}
  \begin{subfigure}{0.33\linewidth}
    \centering
    \includegraphics[width=\linewidth]{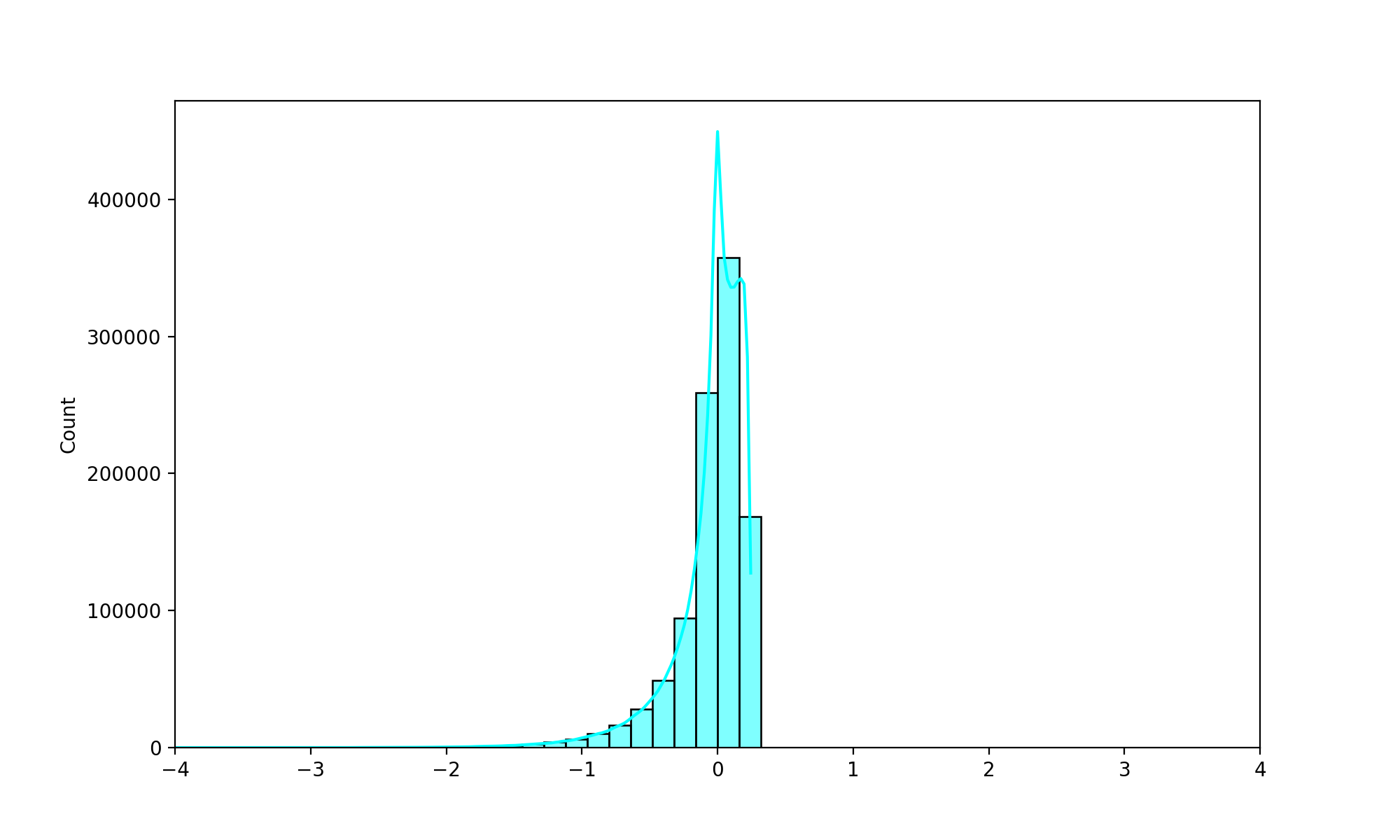}
    \caption{2 Layers}
  \end{subfigure}
  \begin{subfigure}{0.33\linewidth}
    \centering
    \includegraphics[width=\linewidth]{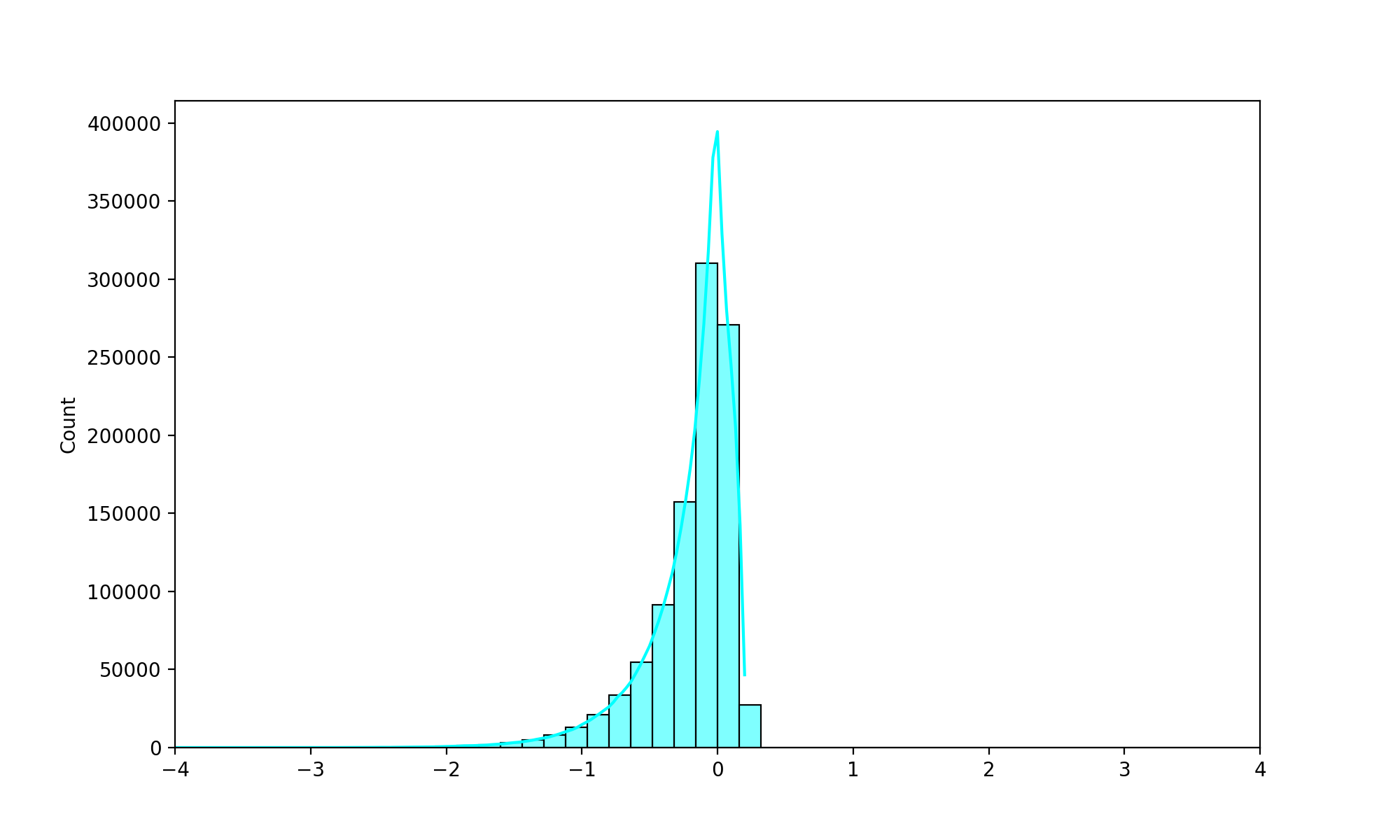}
    \caption{3 Layers}
\end{subfigure}

  \begin{subfigure}{0.33\linewidth}
    \centering
    \includegraphics[width=\linewidth]{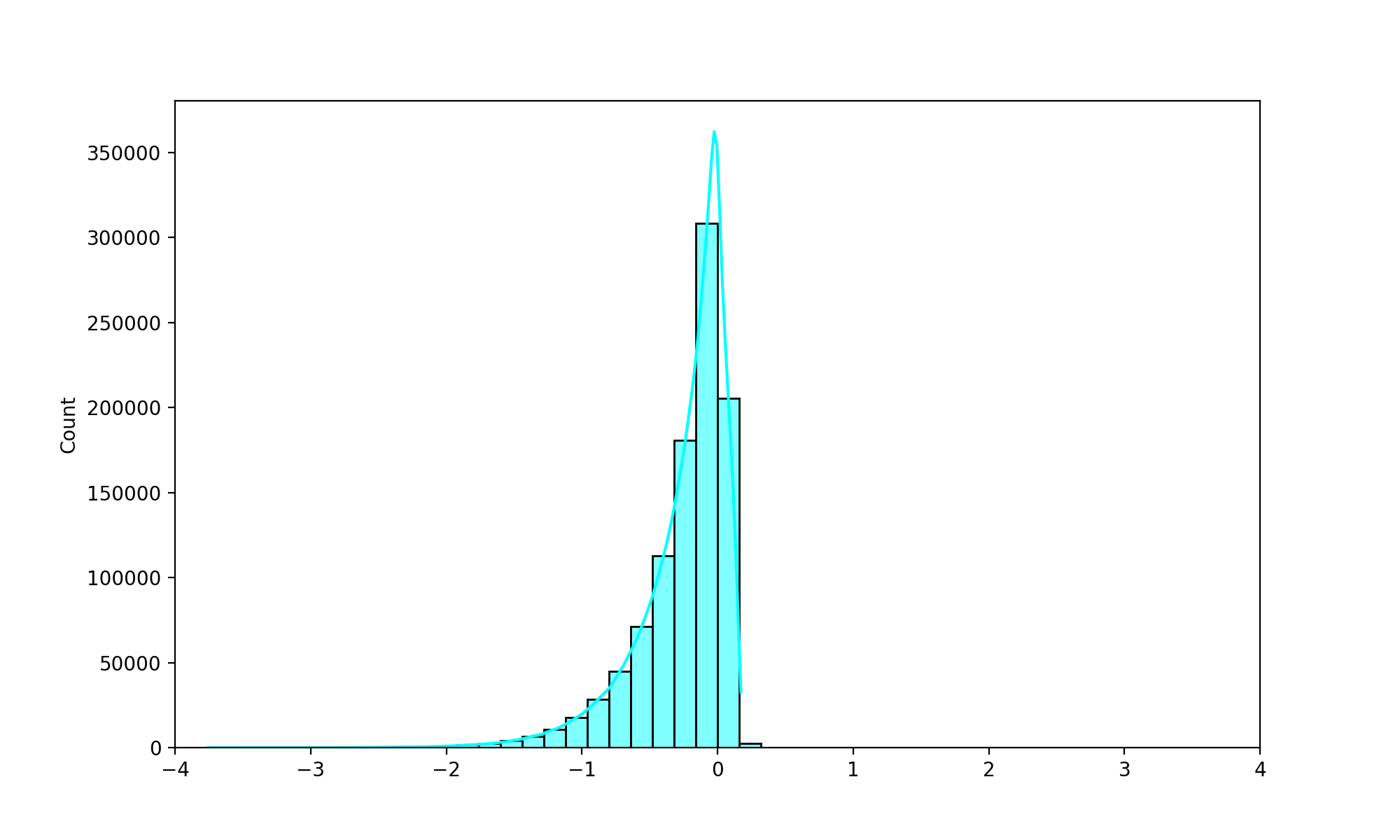}
    \caption{4 Layers}
  \end{subfigure}
  \begin{subfigure}{0.33\linewidth}
    \centering
    \includegraphics[width=\linewidth]{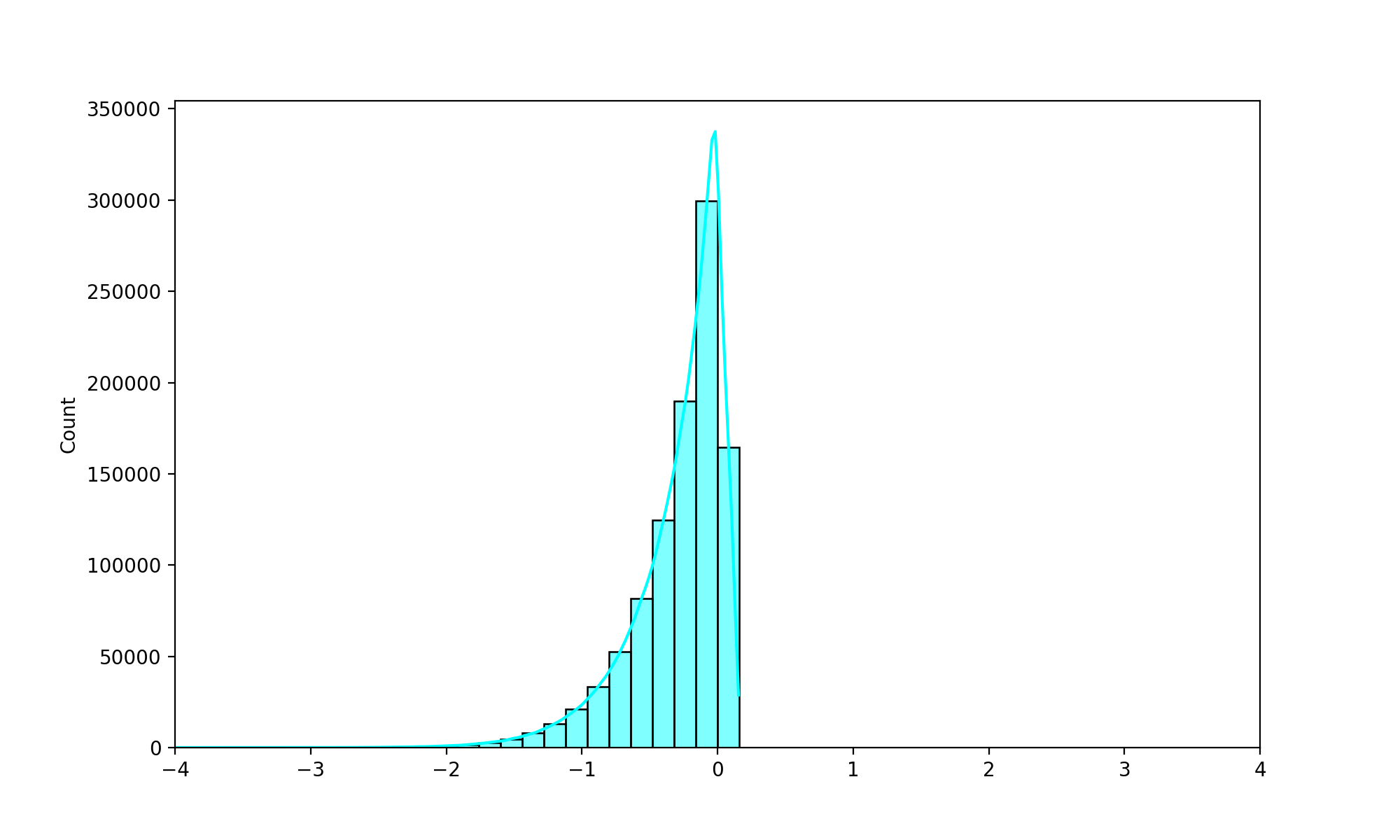}
    \caption{5 Layers}
  \end{subfigure}
  \begin{subfigure}{0.33\linewidth}
    \centering
    \includegraphics[width=\linewidth]{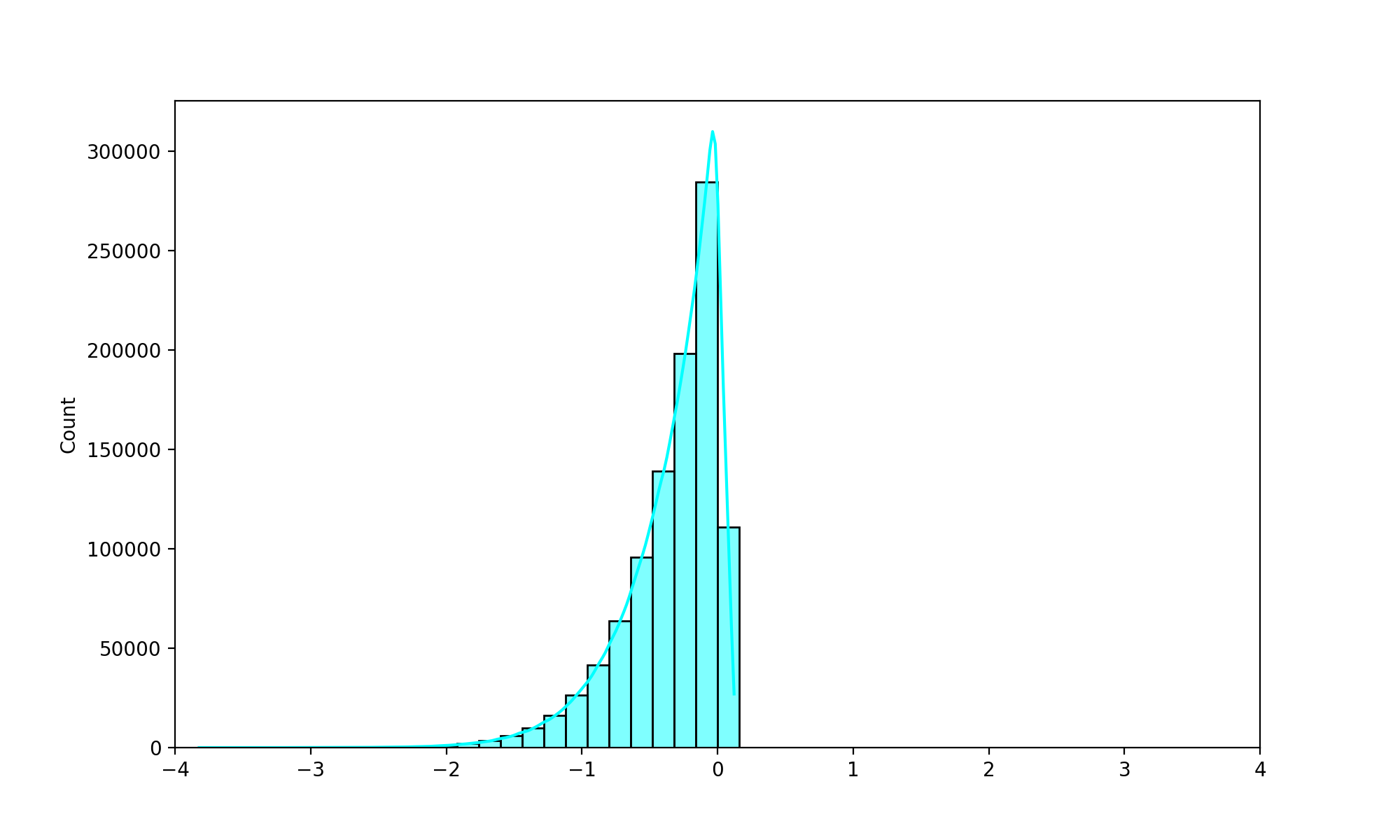}
    \caption{8 Layers}
  \end{subfigure}

  \caption{Visualization that Theorem~\ref{ExplicitLeakyReLUCLT} fails for $\tanh$ networks: Plot of $\frac{\log |X_{\ell}|}{\sqrt{\ell}}$ for $d = 2$ and $\ell = 1,2,3,4,5$ and $8$ for $X_{\ell}$ a $\tanh$ network initialized with Gaussian weights of mean zero and unit variance. We sampled 1'000'000 points and showed the histograms. The distribution never looks Gaussian.}
  \label{fig:CLT_failure_for_tanh}
\end{figure*}

\section{Calculations for He and orthogonal initialization}\label{sec:Heorthcalc}

In this section, we explain why for He and scaled orthogonal initialization, it holds that $\E[|X_{\ell}|^2] = \E[|X_{\ell-1}^2|]$, yet neither the mean nor the variance of $|X_{\ell}|$ are preserved. This gives an informal reason why He initialization and scaled orthogonal initialization lead to a zero Lyapunov exponent as $d\to \infty$. 

The main reason this holds is the following: If $Z$ is a symmetric one dimensional continuous random variable (that is $Z$ and $-Z$ have the same distribution), then we have for $\phi(x) = \max(x,\alpha x)$ with $\alpha \in \R_{\neq 0}$ that $$\mathbb{E}[\phi(Z)^2] = \frac{1}{2}\mathbb{E}[Z^2|Z\geq 0] + \frac{1}{2}\mathbb{E}[\alpha^2 Z^2|Z\leq 0] = \frac{1 + \alpha^2}{2}\mathbb{E}[Z^2].$$

Denote as in the rest of the paper by $X_{\ell}$ a neural network as in \eqref{ProbNNDef}.

\paragraph{He initialization.}  Assume now that the matrix coefficients $W_{\ell}^{ij}$ are all independent and identically distributed as a fixed symmetric continuous distribution of mean zero and variance $\sigma^2$ for some $\sigma > 0$. We fix for simplicity $x_0 \in \R^d \backslash \{0\}$ with $X_0 = x_0$ and write $$s_{\ell} = \mathbb{E}[|X_{\ell}^j|^2].$$ Then it follows that $\sum_j W_{\ell}^{ij}X^j_{\ell - 1}$ is symmetric because the $W_{\ell}^{ij}$ are symmetric and therefore we conclude by using linearity of expectations and independence of the $W_{ij}$ that $$s_{\ell} = \frac{1 + \alpha^2}{2}\E\left[\sum_{j = 1}^d(W_{\ell}^{ij}X^j_{\ell - 1})^2\right] = \frac{1 + \alpha^2}{2}d\sigma^2 s_{\ell -1}.$$ Thus we conclude  that $$s_{\ell} = \left( \sigma^2 \frac{d(1 + \alpha^2)}{2} \right)^{\ell}.$$

For He initialization, we aim to have $s_{\ell} = 1$ for all $\ell$, which is equivalent to $\sigma^2 \frac{d(1 + \alpha^2)}{2} = 1$ or $\sigma = \sqrt{\frac{2}{d(1 + \alpha^2)}},$ leading to \eqref{DefSigmaHe} from He initialization as discussed previously.

\paragraph{Orthogonal Initialization.} An analogous argument also applies in the orthogonal case. Indeed, assume now that $W^{ij}$ is distributed as $m_{\eta\cdot\mathrm{O}(d)}$. If $O\in \mathrm{O}(d)$ is distributed as $m_{\mathrm{O}(d)}$, then the distribution of each coordinate of $Ox_0$ is symmetric. Therefore it holds that $$\E[|X_{\ell}|^2] = \E[|\phi(\eta O X_{\ell - 1})|^2] = \eta^2 \frac{1 + \alpha^2}{2} \E[|OX_{\ell - 1}|^2] = \eta^2 \frac{1 + \alpha^2}{2} \E[|X_{\ell - 1}|^2]. $$ Thus, for $\E[|X_{\ell}|^2]$ to be preserved we need to have 
$\eta =\sqrt{\frac{2}{1 + \alpha^2}}$.






\section{Experimental Details: MNIST} \label{sec:app:exp-details:MNIST}                                      
                  
  \paragraph{Dataset.}
  We train on the MNIST handwritten digit dataset, consisting of 60,000 training images of size $28 \times 28$ pixels. We flatten each image to a
  784-dimensional vector and normalize pixel values to $[0,1]$. The standard 10,000-sample test set is partitioned into a validation set of 5,000 samples   
  and a held-out test set of 5,000 samples using a fixed random permutation (seed 42); this split is used consistently across all methods.
                                                              
  \paragraph{Architecture and Input/Output Layer.}
  We use a deep network of depth 100 and hidden dimension $d = 10$, with Leaky ReLU activations with slope $\alpha = 0.1$ and all biases initialized to
  zero. Since the hidden dimension equals the number of classes, no separate output head is required; the final layer maps directly from $\mathbb{R}^{10}$  
  to $\mathbb{R}^{10}$ logits. The input layer converts from 784 dimensions to $d = 10$ dimensions. The initialization of the first weights has vanishing significance compared to the initialization of the $10 \times 10 = 100$ weights in each of the 98 hidden layers. We initialize the first layer with standard He initialization in the Gaussian cases and scaled orthogonal initialization in the orthogonal cases. 

  \paragraph{Learning Rate Decay.}
  We employ cosine learning rate decay from an initial learning rate $\mathrm{lr}_\mathrm{init}$ to a fixed final learning rate $\mathrm{lr}_\mathrm{final}
  = 10^{-5}$. At step $i$ out of $N$ total training steps, the learning rate is                                                                             
  $$
  \mathrm{lr}(i) = \mathrm{lr}_\mathrm{final} + \bigl(\mathrm{lr}_\mathrm{init} - \mathrm{lr}_\mathrm{final}\bigr) \cdot \frac{1 + \cos\!\left(\pi\, i /    
  N\right)}{2}.                                                                                                                                             
  $$
  We use the AdamW optimizer with weight decay $\lambda$, batch size 256, and train for 200 epochs.            
  
\paragraph{Hyperparameter Tuning.}
  We tune over initial learning rates $\mathrm{lr}_\mathrm{init} \in \{3 \times 10^{-4},\, 10^{-3},\, 3 \times 10^{-3}\}$ and weight decay coefficients     
  $\lambda \in \{10^{-4},\, 10^{-3}\}$, yielding six combinations per method. Each combination is evaluated over 100 random seeds. The best combination is   
  selected by mean final validation accuracy. 
  
  The training time for a single seed and hyperparameter on an A100 GPU was around 10 minutes. 

  \begin{table}[h]
  \centering                                             \caption{Best hyperparameters for each initialization method on the MNIST classification experiment, selected by mean final validation accuracy over 100   
  seeds. Batch size 256 and final learning rate $\mathrm{lr}_\mathrm{final} = 10^{-5}$ are fixed across all methods.}                                       
  \label{tab:best-hyperparams:mnist}                                                                                            
  \begin{tabular}{lcccc}
  \toprule                                                                                                                                                  
  Method & Scale & $\mathrm{lr}_\mathrm{init}$ & $\lambda$ \\
  \midrule
  Glorot--Bengio          & $\sigma = 0.316$  & $10^{-3}$          & $10^{-3}$ \\                                                                           
  He                      & $\sigma = 0.445$  & $3 \times 10^{-4}$ & $10^{-3}$ \\
  Basic Orthogonal        & $\eta = 1.408$    & $3 \times 10^{-4}$ & $10^{-3}$ \\                                                                           
  Lyapunov Gaussian       & $\sigma_{\mathrm{crit}} = 0.514$ & $3 \times 10^{-4}$ & $10^{-4}$ \\                                                                          
  Lyapunov Orthogonal     & $\eta_{\mathrm{crit}} = 1.5442$ & $3 \times 10^{-4}$ & $10^{-3}$ \\                                                                           
  Sampled Lyapunov Gaussian & $\sigma_{\mathrm{crit}} = 0.514$ & $3 \times 10^{-4}$ & $10^{-3}$ \\                                                                        
  Sampled Lyapunov Orthogonal & $\eta_{\mathrm{crit}} = 1.5442$ & $3 \times 10^{-4}$ & $10^{-4}$ \\                                                                       
  \bottomrule                                                                                                                                               
  \end{tabular}                                                                                                                                                                                                                                                                 
  \end{table}      

 \paragraph{Model Selection.}
  We select the best hyperparameter choice based on the mean final validation accuracy across seed values. The mean test accuracy over 100 seeds is shown in Figure~\ref{fig:exp:mnist100:losses} and we show in Figure~\ref{MNISTHists} the histograms over all 100 seeds giving precise statistics on how our methods outperform the previous methods. In fact, our results are robust under the choice of hyperparameters.

\begin{figure}[htbp]

\centering

\includegraphics[width=0.22\textwidth]{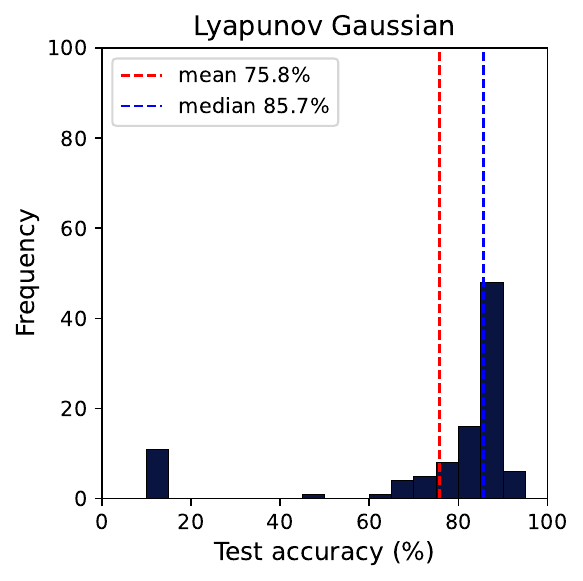}\hfill
\includegraphics[width=0.22\textwidth]{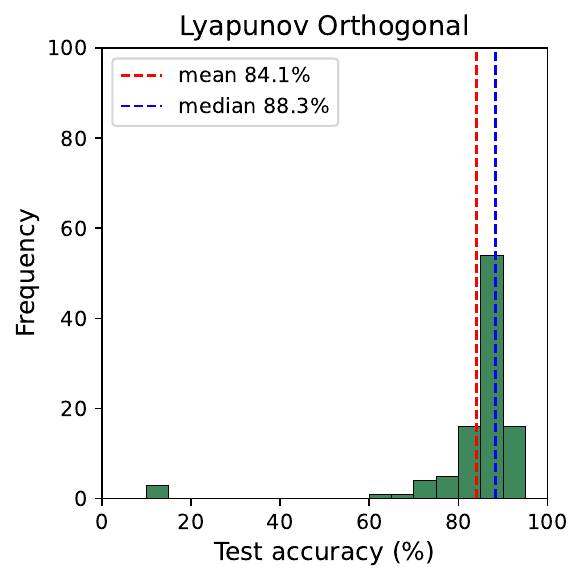}\hfill
\includegraphics[width=0.22\textwidth]{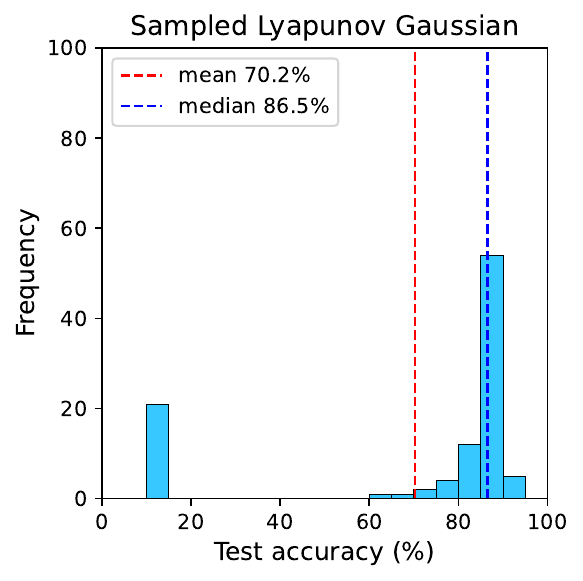}\hfill
\includegraphics[width=0.22\textwidth]{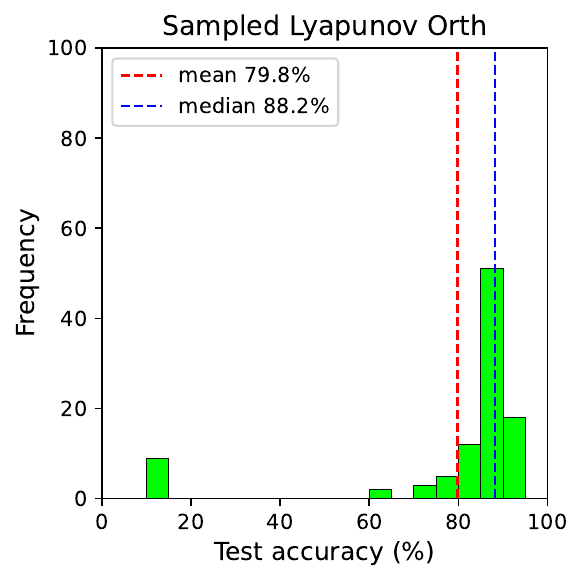}
\vspace{0.5cm}

\makebox[\textwidth][c]{%

  \includegraphics[width=0.22\textwidth]{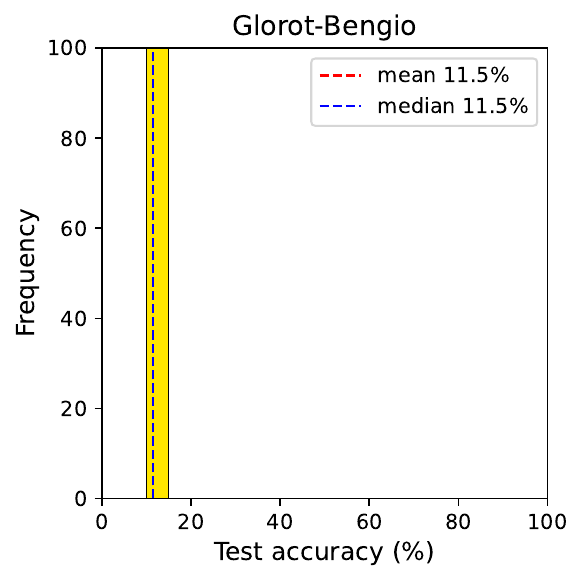}\hspace{0.04\textwidth}%

  \includegraphics[width=0.22\textwidth]{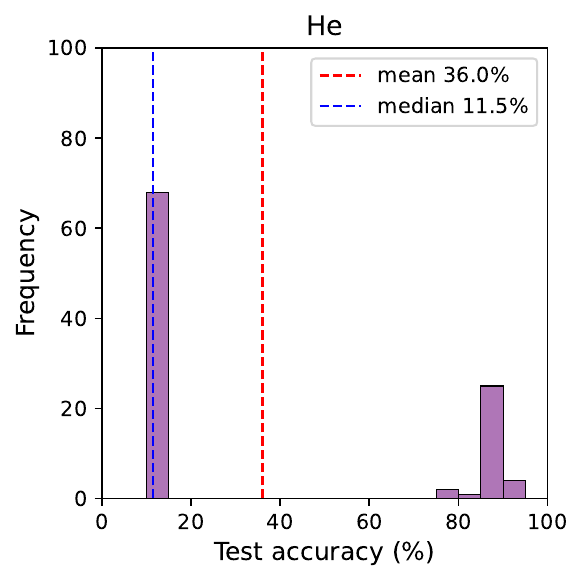}\hspace{0.04\textwidth}%

  \includegraphics[width=0.22\textwidth]{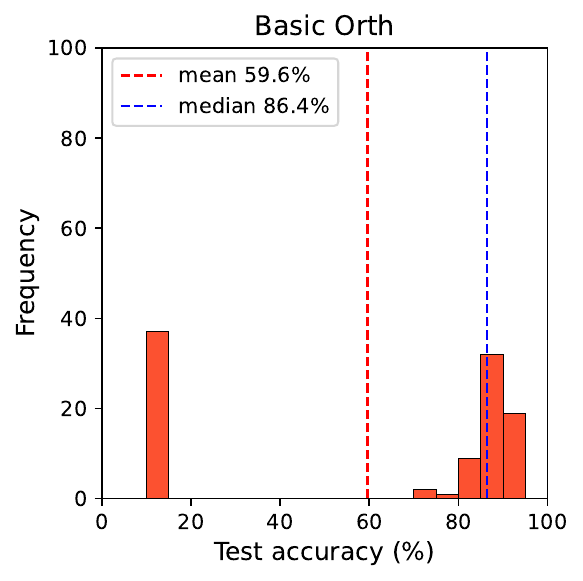}%

}

\caption{Histograms for the MNIST experiments at epoch 200. For each method, we chose the best hyperparameter and show 100 seeds.}
\label{MNISTHists}
\end{figure}

\section{Experimental Details: Polynomial} \label{sec:app:exp-details:poly}

\begin{figure}[b]

\centering

\includegraphics[width=0.22\textwidth]{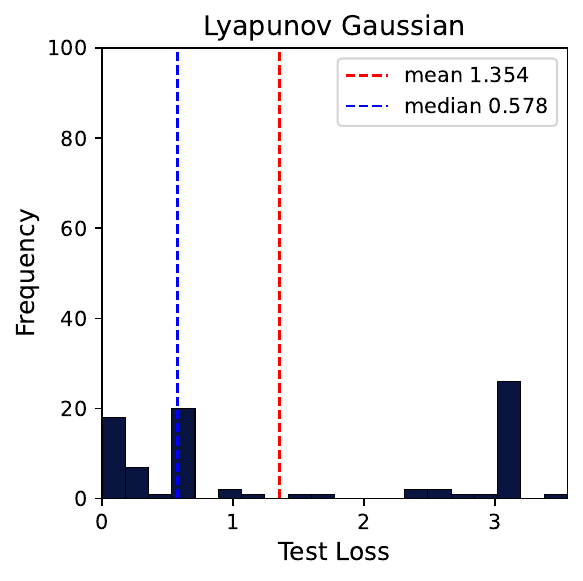}\hfill
\includegraphics[width=0.22\textwidth]{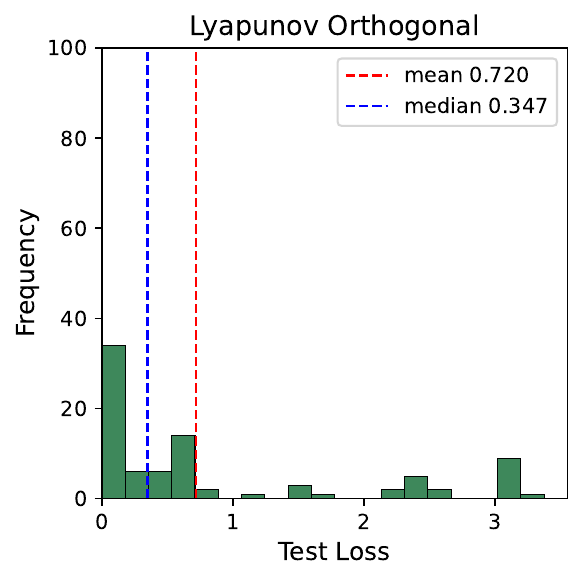}\hfill
\includegraphics[width=0.22\textwidth]{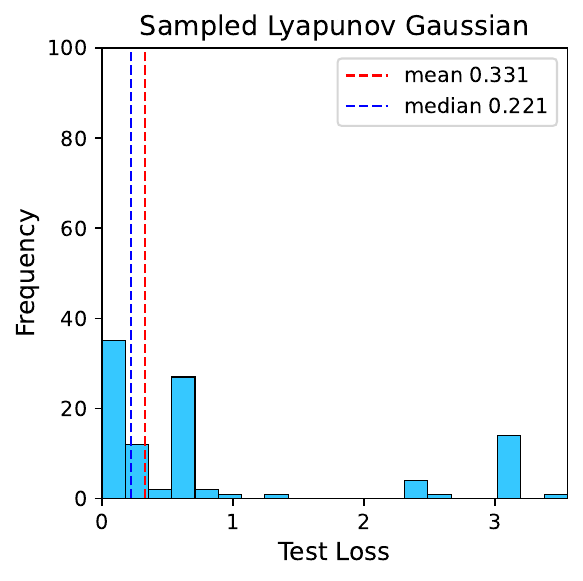}\hfill
\includegraphics[width=0.22\textwidth]{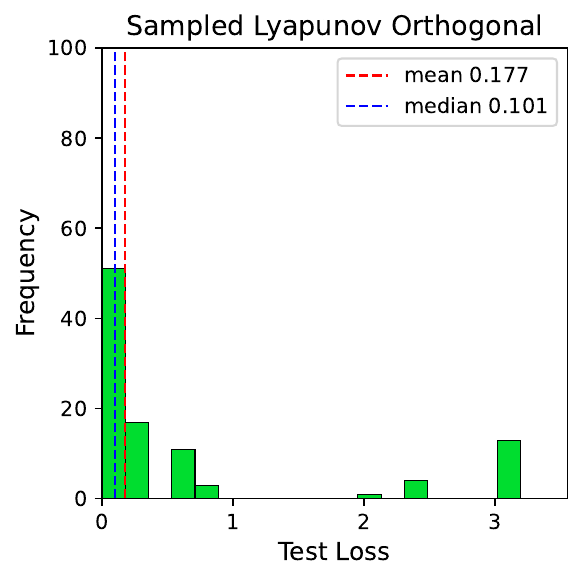}
\vspace{0.5cm}

\makebox[\textwidth][c]{%

  \includegraphics[width=0.22\textwidth]{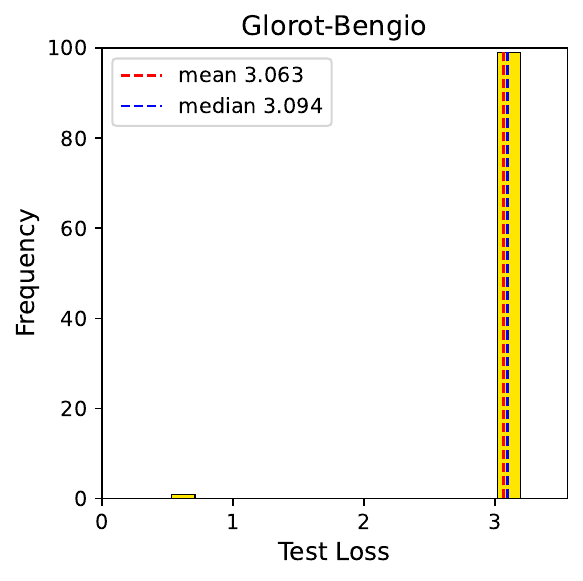}\hspace{0.04\textwidth}%

  \includegraphics[width=0.22\textwidth]{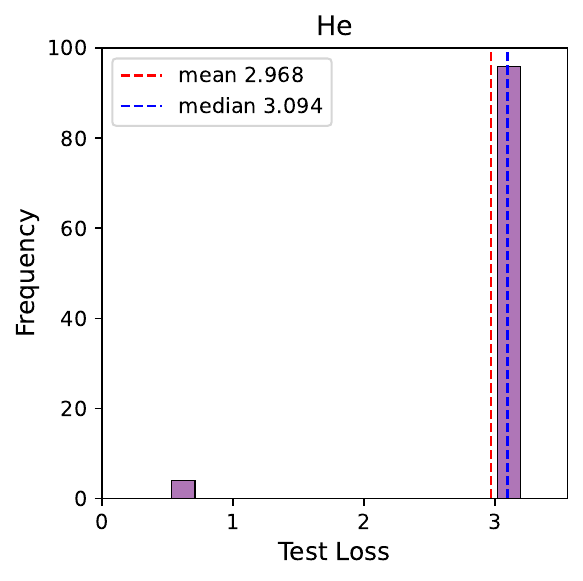}\hspace{0.04\textwidth}%

  \includegraphics[width=0.22\textwidth]{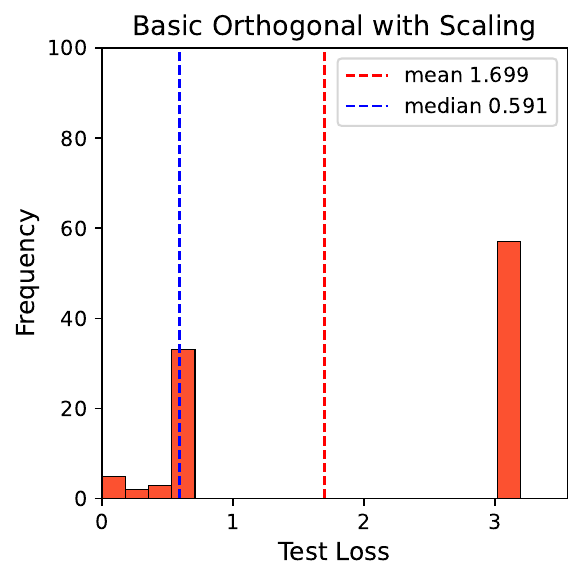}%

}

\caption{Histograms for the polynomial experiments at training step 10'000. For each method, we chose the best hyperparameter according to the performance for the top 80 seeds. In this plot we show the histograms for all 100 seeds.}
\label{PolynomialHists}
\end{figure}

\paragraph{Dataset.}
  We aim to learn the polynomial $p(x) = x^5 + x^2 - x$ on the interval $[-1.5, 1.5]$. At each training step, we randomly sample batch size many datapoints $x_i \in [-1.5,1.5]$ and use $(x_i,p(x_i))$ to calculate the loss. For the validation and test set, we have sampled 10'000 datapoints in $[-1.5,1.5]$ for each.

\paragraph{Input/Output layer.} We add, in addition to the deep neural network, one input layer and one output layer. We do this for the following reason:
For polynomial learning, we use a deep network of depth $60$ and with $d=2$. Since the polynomial is a $1d$ function, we add an input and output layer converting from $1d$ to $2d$ and vice versa. The initialization of these four weights has vanishing significance compared to the initialization of the $240$ weights in the $60$ layers in between, but we initialize them with the Lyapunov Gaussian / He Gaussian  / Glorot uniform distribution for consistency.

\begin{figure}
 \centering
    \begin{subfigure}[t]{0.48\linewidth}
        \centering
        \includegraphics[width=\linewidth]{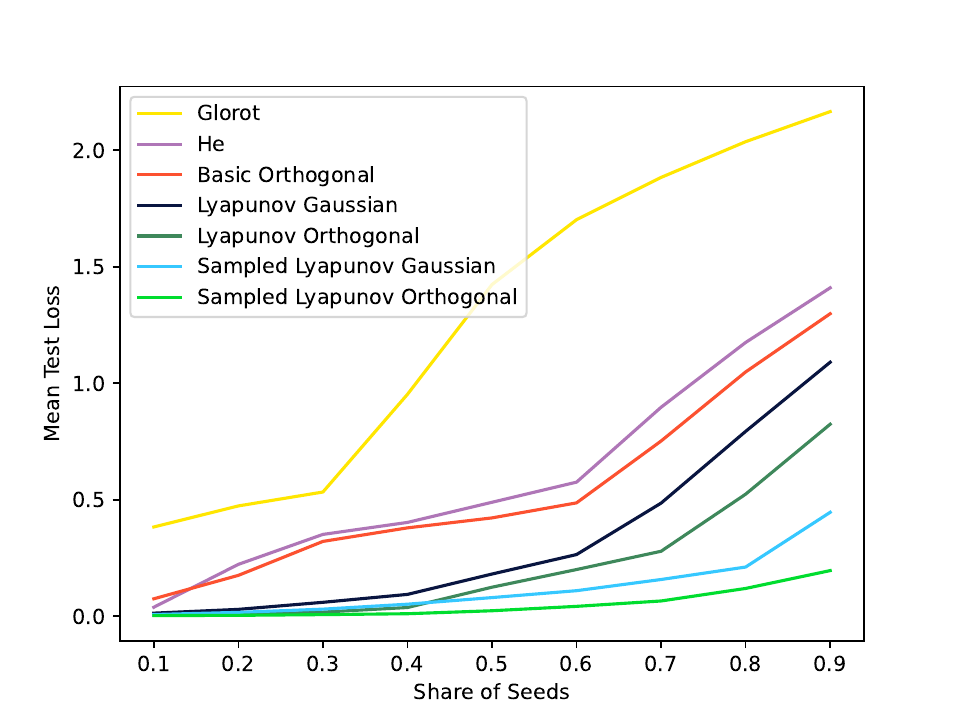}
        \caption{Polynomial learning}
        \label{fig:app:avg-test-loss-per-share-to-keep:poly}
    \end{subfigure}
    \hfill
    \begin{subfigure}[t]{0.48\linewidth}
        \centering
        \includegraphics[width=\linewidth]{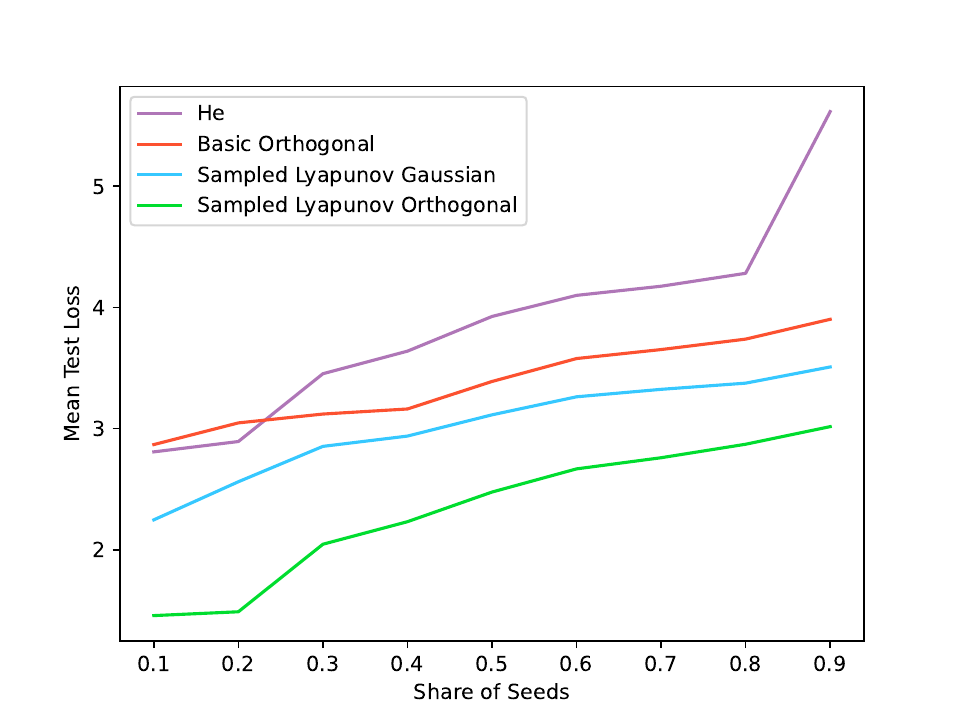}
        \caption{Score learning}
        \label{fig:app:avg-test-loss-per-share-to-keep:score}
    \end{subfigure}
    \caption{Robustness to choice of share of seed values to consider when learning the polynomial and the score: For robustness to outliers, we consider for every initialization method only the best $0 \ll p<1$ samples.}
    \label{fig:app:avg-test-loss-per-share-to-keep}
\end{figure}

\paragraph{Learning Rate Decay.}
We employ learning rate decay from an initial learning rate to a final learning rate, both of which are hyperparameters for which we consider different values during hyperparameter tuning. At step $i$, with $N$ total learning steps, the learning rate is
$$ \mathrm{lr}_\mathrm{init}- \left(\mathrm{lr}_\mathrm{init} - \mathrm{lr}_\mathrm{final} \right)\left(\frac{i}{N}\right)^2.$$

\paragraph{Hyperparameter tuning.}
 We tune over initial learning rates $\mathrm{lr}_\mathrm{init} \in \{3 \times 10^{-4},\, 10^{-3}\}$, final learning rates $\mathrm{lr}_\mathrm{final} \in \{5 \times 10^{-5},\, 10^{-4}, 10^{-3} \}$ and batch sizes    
  $\{ 50,100,500,1000 \}$, yielding 24 combinations per method.

The training time for a single seed and hyperparameter on an A100 GPU was around 1 minute. 

\paragraph{Moving average.}

When reporting loss curves, we apply a moving average (median) for stability. We use window size 100 for polynomial learning. 

\begin{table}[t]
\centering
\caption{Training hyperparameters and average test loss across methods for polynomial learning.}
\label{tab:best-hyperparams:poly}
\begin{tabular}{lcccc}
\toprule
Method & $\mathrm{lr}_{\mathrm{init}}$ & $\mathrm{lr}_{\mathrm{final}}$ & Batch Size & Avg Test Loss \\ \midrule
Glorot & 1.00e-04 & 1.00e-04 & 1000 & 2.05 \\
He & 1.00e-04 & 1.00e-04 & 500 & 1.20 \\
Basic Orthogonal & 1.00e-04 & 1.00e-04 & 1000 & 1.07 \\
Lyapunov Gaussian & 1.00e-04 & 1.00e-04 & 1000 & 0.83 \\
Lyapunov Orthogonal & 1.00e-03 & 1.00e-03 & 500 & 0.55 \\
Sampled Lyapunov Gaussian & 1.00e-03 & 1.00e-04 & 1000 & 0.22 \\
Sampled Lyapunov Orthogonal & 1.00e-03 & 1.00e-03 & 1000 & 0.12 \\
\bottomrule
\end{tabular}
\end{table}

\paragraph{Model selection.}

We select the best hyperparameter choice based on the mean of the best performing $p=80\%$ of samples measured in terms of test loss. We choose $p=80\%$ to be robust to outliers with large loss caused by the Central Limit Theorem. However, our results are robust to the choice of $p$, as we depict in \Cref{fig:app:avg-test-loss-per-share-to-keep} for both the polynomial and the score learning case and as long as we don't use $p \approx 1$.

We show in Figure~\ref{fig:polynomiallosses} the median losses for the best hyperparameter and in Figure~\ref{PolynomialHists} the histograms of the losses for all seeds at the final training step, showing that our new methods strongly outperform the previous methods. As for MNIST, our experimental results are rather robust under the choice of hyperparameters.

\section{Experiment: Score}

As a third experiment, we learn the score of a distribution. This is well-motivated as learning the score of a probability distribution is \emph{the} learning task in diffusion models  \citep{sohl-dicksteinDeepUnsupervisedLearning2015, songScoreBasedGenerativeModeling2021, hoDenoisingDiffusionProbabilistic2020}. The quantity of interest is the score $\nabla \ln p$ where $p$ is typically a mixture of Gaussians. 

We use exactly the same experimental setup as for polynomial learning, so we will be brief in our explanations. 

In order to emphasize the early advantage of our sampled methods, we show the first 1,000 training steps as well as the 130,000 training steps. In addition, we only show He, basic orthogonal and the sampled Lyapunov methods as they are the best methods from our previous two experiments.

\begin{table}[t]
\centering
\caption{Training hyperparameters and average test loss across methods for score learning.}
\label{tab:best-hyperparams:score}
\begin{tabular}{lcccc}
\toprule
Method & $\mathrm{lr}_{\mathrm{init}}$ & $\mathrm{lr}_{\mathrm{final}}$ & $\mathrm{Batch Size}$ & Avg Test Loss \\ \midrule
He & 1.00e-03 & 1.00e-04 & 1600 & 4.88 \\
Scaled Orthogonal & 1.00e-03 & 1.00e-04 & 1600 & 3.82 \\
Sampled Lyapunov Gaussian & 1.00e-02 & 1.00e-04 & 400 & 3.42 \\
Sampled Lyapunov Orthogonal & 1.00e-02 & 1.00e-04 & 1600 & 2.96 \\
\bottomrule
\end{tabular}
\end{table} 

\paragraph{Model Setup.} 

In our experiment we learn the score of a mixture of three Gaussians on the 2d interval $[-8,+8]^2$. We consider a mixture of three Gaussians in $2d$ with means 
\[
\mu_1 = \begin{pmatrix} -3 \\ 3 \end{pmatrix},\quad
\mu_2 = \begin{pmatrix} 3 \\ -3 \end{pmatrix},\quad
\mu_3 = \begin{pmatrix} 0 \\ 0 \end{pmatrix},
\]
and covariances
\[
\Sigma_1 = \begin{pmatrix}
1 & 0 \\
0 & 1
\end{pmatrix},\quad
\Sigma_2 = \begin{pmatrix}
2 & 1 \\
1 & 2
\end{pmatrix},\quad
\Sigma_3 = \begin{pmatrix}
0.5 & 0 \\
0 & 0.5
\end{pmatrix}.
\]
The mixture's density at $x \in \R^2$ is given by 
$$ p_{mixture}(x) = 0.4 p_1(x)+0.4 p_2(x) + 0.2 p_3(x),$$
where $p_i(x)$ denotes the density of the Gaussian with mean $\mu_i$ and covariance $\Sigma_i$.

We learn the score with a network of width $d=2$, and depth $30$, which is a challenging task and therefore many learning steps are necessary.

\paragraph{Hyperparameter tuning.} We use the same learning rate decay as for the polynomial experiment. We tune over initial learning rates $\mathrm{lr}_\mathrm{init} \in \{ 10^{-2},\, 10^{-3}\}$, final learning rates $\mathrm{lr}_\mathrm{final} \in \{10^{-3},\, 10^{-4}\}$ and batch sizes    $\{ 1600,400 \}$, yielding 8 combinations per method. The best hyperparameters with respect to the final training step are shown in Table~\ref{tab:best-hyperparams:score}. 

The training time for a single seed and hyperparameter on an A100 GPU was around 10 minutes. 

\paragraph{Moving average.}

When reporting loss curves, we apply a moving average of 10 for the score test losses until training step 1,000 in Figure~\ref{fig:exp:score:losses} and 1000 for the score test losses up to training step 130,000.

\paragraph{Score Results: Short Training.}  We depict the short time scale ($1,000$ learning steps) in \Cref{fig:exp:score:losses}, and observe that the losses of both sampled Lyapunov methods decrease faster than with existing initialization methods. Consistently with the MNIST and polynomial experiments, Sampled Lyapunov Orthogonal is the fastest learning initialization.

\paragraph{Score Results: Long training.}
\label{sec:app:score-long}

We depict in \Cref{fig:loss-score-long} the behavior of the loss for more training steps. While we stress that learning a very deep network is an inherently hard task and further improvements could be made in other areas, such as optimization algorithms, we observe that Sampled Lyapunov Orthogonal initialization behaves best for all scales. Interestingly, Orthogonal initialization overtakes after a few thousand training steps even Sampled Lyapunov Gaussian initialization, yet after longer scales Sampled Lyapunov Gaussian initialization outperforms Orthogonal initialization. This is somewhat consistent with the established results that orthogonal initialization performs better than Gaussian initialization for deep networks \cite{SaxeOrthogonal,Huorthogonal}.

\begin{figure}
    \centering
    \includegraphics[width=0.6\linewidth]{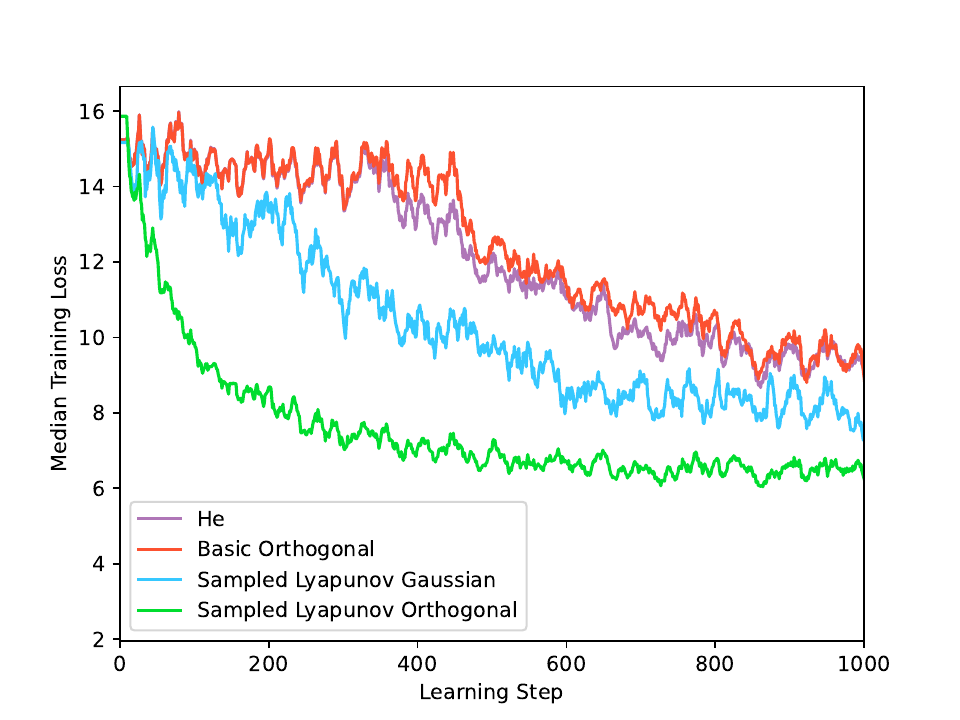}
    \caption{Improved initial phase for learning a score: Both Sampled Lyapunov Gaussian (blue) and Sampled Lyapunov Orthogonal (green) show lower losses in the first $1,000$ steps while learning the score of a Gaussian mixture.}
    \label{fig:exp:score:losses}
\end{figure}

\begin{figure}[h!]
\centering
\begin{minipage}{0.69\textwidth}
    \centering    
    \includegraphics[width=\linewidth]{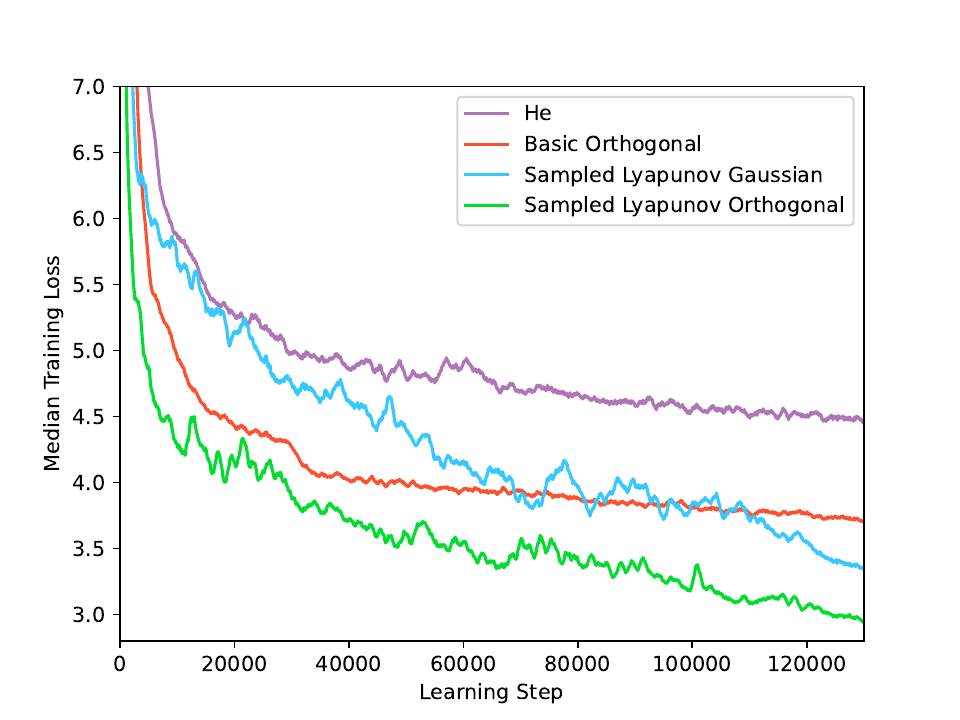}
    \end{minipage}\\
    \begin{minipage}{0.69\textwidth}
    \centering
    \scriptsize
    \begin{tabular}{lcccccc}
\toprule
Method \quad\quad\quad\quad\quad\quad\quad\quad Epoch: & 100 & 1000 & 20000 & 60000 & 100000 & 120000 \\ 
\midrule
He & 12.14 & 12.14 & 5.27 & 4.90 & 4.56 & 4.51 \\
Orthogonal & 12.45 & 12.44 & 4.43 & 3.95 & 3.81 & 3.77 \\
\textbf{Sampled Lyapunov Gaussian} & 10.35 & 10.34 & 5.14 & 4.16 & 3.84 & 3.55 \\
\textbf{Sampled Lyapunov Orthogonal} & 7.56 & 7.55 & 4.15 & 3.55 & 3.24 & 3.06 \\
\bottomrule
\end{tabular}
\end{minipage}
\caption{Long training of the depth-30 width-2 neural network for learning the score: We observe that Sampled Lyapunov Orthogonal behaves best.}
\label{fig:loss-score-long}
\end{figure}

\newpage

\clearpage

\end{document}